\documentclass[12pt]{article}

\usepackage{newtxtext,newtxmath}

\usepackage{graphicx}

\usepackage[letterpaper,margin=1in]{geometry}

\renewenvironment{abstract}
	{\quotation}
	{\endquotation}

\date{}

\makeatletter
\renewcommand{\fnum@figure}{\textbf{Figure \thefigure}}
\renewcommand{\fnum@table}{\textbf{Table \thetable}}
\makeatother

\usepackage{scicite}

\usepackage{url}

\usepackage{amsthm}
\theoremstyle{plain}
\newtheorem{Theorem}{Theorem}

\newtheorem{Proposition}{Proposition}
\newtheorem{Definition}{Definition}

\usepackage{multirow}

\makeatletter
\long\def\@makecaption#1#2{%
  \vskip\abovecaptionskip
  \small                          
  \setlength{\baselineskip}{12pt} 
  \sbox\@tempboxa{#1: #2}
  \ifdim\wd\@tempboxa >\hsize
    #1: #2\par                   
  \else
    \global\@minipagefalse
    \hb@xt@\hsize{\hfil\box\@tempboxa\hfil}\par
  \fi
  \vskip\belowcaptionskip}
\makeatother

\def\scititle{
    Learning to Move in Rhythm: Task-Conditioned Motion Policies with Orbital Stability Guarantees
}
\title{\bfseries \boldmath \scititle}

\author{
    Maximilian~Stölzle$^{1,2\ast}$,
    T.~Konstantin~Rusch$^{1,3,4\dagger}$,
    Zach~J.~Patterson$^{1,5\dagger}$, \and
    Rodrigo~Pérez-Dattari$^{6}$,
    Francesco~Stella$^{1,7,8}$,\and
    Josie Hughes$^{7}$,
    Cosimo~Della~Santina$^{2}$,
    Daniela~Rus$^{1}$
	\and
	\small$^{1}$CSAIL, Massachusetts Institute of Technology, Cambridge, MA 02139, USA.
        \and
	\small$^{2}$Cognitive Robotics, Delft University of Technology, Delft, 2628 CD, Netherlands.
        \and
        \small$^{3}$ELLIS Institute Tübingen, 72076, Germany.
        \and
        \small$^{4}$Max Planck Institute for Intelligent Systems, Tübingen, 72076, Germany.
        \and
        \small$^{5}$Mechanical and Aerospace Engineering, Case Western Reserve University, Cleveland, OH 44106, USA.
        \and
        \small$^{6}$ Division of Robotics, Perception and Learning, KTH, Lindstedtsvägen 24, SE-100 44 Stockholm, Sweden.
        \and
        \small$^{7}$CREATE Lab, EPFL, Lausanne, 1015, Switzerland. \and
        \small$^{8}$ Embodied AI AG, Lausanne, Switzerland. \and
	\small$^\ast$Corresponding author. Email: M.W.Stolzle@tudelft.nl\and
	\small$^\dagger$These authors contributed equally to this work.
}

\begin{document} 

\maketitle


\begin{abstract} \bfseries \boldmath

Learning from demonstration provides a sample-efficient approach to acquiring complex behaviors, enabling robots to move robustly, compliantly, and with fluidity. In this context, Dynamic Motion Primitives offer built-in stability and robustness to disturbances but often struggle to capture complex periodic behaviors. Moreover, they are limited in their ability to interpolate between different tasks. These shortcomings substantially narrow their applicability, excluding a wide class of practically meaningful tasks such as locomotion and rhythmic tool use.
In this work, we introduce Orbitally Stable Motion Primitives (OSMPs)—a framework that combines a learned diffeomorphic encoder with a supercritical Hopf bifurcation in latent space, enabling the accurate acquisition of periodic motions from demonstrations while ensuring formal guarantees of orbital stability and transverse contraction. Furthermore, by conditioning the bijective encoder on the task, we enable a single learned policy to represent multiple motion objectives, yielding consistent zero-shot generalization to unseen motion objectives within the training distribution.
We validate the proposed approach through extensive simulation and real-world experiments across a diverse range of robotic platforms—from collaborative arms and soft manipulators to a bio-inspired rigid–soft turtle robot—demonstrating its versatility and effectiveness in consistently outperforming state-of-the-art baselines such as diffusion policies, among others.
\end{abstract}

\noindent

\noindent \textbf{Summary.}
We introduce a novel training procedure for motion policies from periodic demonstrations backed by global convergence guarantees.

\section{Introduction}
Imitation Learning \cite{schaal1999imitation, zare2024survey} has regained substantial traction in recent years due to its superior sample and iteration efficiency in acquiring complex tasks compared to Reinforcement Learning. Recent work has focused on improving the robustness, expressiveness, and generalization of motion policies learned from demonstration by leveraging modern ML architectures such as diffusion models and flow matching~\cite{chi2023diffusion, black2024pi0}; scaling up the number of demonstrations to increase robustness~\cite{o2024open, black2024pi0, gemini2025robotics}; training across multiple robot embodiments (e.g., different manipulators) to promote generalization~\cite{o2024open, black2024pi0}; and conditioning policies on semantic task instructions and environment context via embeddings from large vision–language models (VLMs)~\cite{black2024pi0, gemini2025robotics}.
Among these, Dynamic Motion Primitives (DMPs)~\cite{ijspeert2002learning, ijspeert2013dynamical, saveriano2023dynamic, hu2024fusion} parametrize a motion policy through dynamical systems that predict the desired velocity or acceleration based on the system’s current state. By grounding the formulation in dynamical systems, researchers can leverage established tools from nonlinear system theory~\cite{khalil2002nonlinear} to analyze and ensure-by-design convergence properties in motion primitive—such as global asymptotic stability~\cite{kober2009learning, ijspeert2013dynamical, rana2020euclideanizing, urain2020imitationflow, zhang2022learning, perez2023stable, perez2024puma} or orbital stability~\cite{ijspeert2002learning, kober2009learning, wensing2017sparse, urain2020imitationflow, khadivar2021learning, abu2021periodic, abu2024learning, zhi2024teaching, nah2025combining}. This is not typically the case for other ML-based motion policies like RNNs or Diffusion Policies (DPs)~\cite{chi2023diffusion, o2024open, black2024pi0, gemini2025robotics}.
Such approaches—often referred to as Stable Motion Primitives (SMPs)—are robust to perturbations, disturbances, and model mismatches, as the motion policy continuously steers the system back to the desired reference. This also enhances data efficiency, a trait that is increasingly important as robots take on a broader range of tasks.


An important subclass of DMP strategies addresses tasks that require continuous, non-resting motion—those for which rest-to-rest trajectories are neither representative nor sufficient. Canonical examples include wiping a surface, swimming, or walking, where motion generation must produce sustained activity across cycles. These so-called rhythmic or periodic DMPs have spurred extensive research, both within the traditional dynamical systems formulation~\cite{ijspeert2002learning,kober2009learning,ijspeert2013dynamical,wensing2017sparse,kramberger2018passivity,saveriano2023dynamic,abu2024learning,hu2024fusion,nah2025combining} and in more recent methods combining simple latent-space limit cycles with learned diffeomorphic mappings~\cite{urain2020imitationflow,khadivar2021learning,zhi2024teaching}. Still, despite these advances, existing approaches struggle to reproduce non-trivial trajectories—especially those with sharp transitions, high curvature, or discontinuous velocity profiles, which are common in real-world rhythmic tasks. Overcoming this typically requires many demonstrations, thus strongly limiting their applicability in practical settings. 

Such limitations are exacerbated by the incapability of classical deterministic DMPs to generalize across tasks~\cite{jaquier2025transfer}: a fresh or fine-tuned model must be trained for every new motion or task \cite{saveriano2023dynamic}. Although several studies have introduced task-conditioned variants—such as conditioning on encoded visual observations~\cite{bahl2020neural, mohammadi2024extended} or adopting probabilistic DMP formulations~\cite{seker2019conditional, saveriano2023dynamic, pekmezci2024coupled}—these methods often yield incoherent trajectories when presented with tasks they did not explicitly encounter during training~\cite{jaquier2025transfer}, even when those tasks lie within the original training distribution. 

So, despite their promise of being an alternative to data-intensive learning strategies, DMPs ultimately require a substantial amount of data and a complex training process when tasks are varied and trajectories are not straightforward. Instead, the ability to generate purposeful motions in zero-shot settings for unseen tasks will be essential on the path towards truly generalist autonomous robots in the future.

In this paper, we introduce Orbitally Stable Motion Primitives (OSMPs), a framework, visualized in Fig.~\ref{fig:concept_overview}, designed to address the limitations of existing rhythmic motion primitives by learning an expressive, orbitally stable limit cycle capable of capturing elaborate periodic behaviors. Our approach imposes a dynamic inductive bias by shaping the latent space according to a supercritical Hopf bifurcation oscillator; a well-studied system in nonlinear dynamics~\cite{strogatz2018nonlinear, khadivar2021learning, nah2025combining} that has remained unexplored in the context of machine learning. This core dynamical prior is complemented by a novel bijective Euclideanizing-flow encoder, extending the Real NVP architecture~\cite{rana2020euclideanizing, dinh2017density}.

Under mild architectural assumptions, we prove OSMPs are almost-globally transverse contracting, so every trajectory converges exponentially—not merely asymptotically—to the learned limit cycle. A tailored loss suite binds the cycle’s shape and speed to the demonstration, eliminating the long-standing mismatch between a stable latent orbit and a highly curved, nonlinear sample. Thus, a single demonstration already yields an effective policy, drastically outperforming the data efficiency of DPs~\cite{chi2023diffusion}. 
A novel conditioning–interpolation loss then drives smooth, zero-shot transitions between related tasks: for example, the model continuously morphs between reverse and forward turtle-swimming gaits after seeing only those two exemplars. This ability sharply reduces the need to densely sample the motion task space and further increases data efficiency. 
We also introduce solutions for synchronizing multiple primitives in their phase and, without retraining, affinely scale, translate, or otherwise modulate the learned velocity field, turning rhythmic DMPs into a practical, data-efficient controller for complex robots.

We rigorously validate our approach in both simulation and on hardware. Quantitative and qualitative benchmarks compare it with classical neural models (MLPs, RNNs, NODEs), state-of-the-art motion policies like Diffusion Policy (DP)~\cite{chi2023diffusion}, and diffeomorphic-encoder methods with stable latent dynamics, including Imitation Flow (iFlow)~\cite{urain2020imitationflow} and SPDT~\cite{zhi2024teaching}. Experiments show our OSMPs accurately track periodic trajectories on diverse robotic platforms—UR5 and Kuka arms, the Helix soft robot, and a hybrid underwater “crush turtle” robot. Expressed as autonomous dynamical systems, they deliver more compliant, natural behavior than time-indexed, error-based controllers. We further demonstrate in both simulations and real-world experiments in-phase synchronization of multiple OSMPs and, via encoder conditioning, smooth interpolation among several distinct motion tasks with a single motion policy.

\subsection{Methodology in a Nutshell}
Below, we provide a brief overview of the OSMP methodology and architecture, as depicted in Fig.\ref{fig:methodology_overview}. Further details are provided in Sec.\ref{sec:materials_and_methods} and the Supplementary Material~\cite{methods}. In the DMP framework, an OSMP outputs the desired system velocity as $\dot{x}=f(x; z)$, with $x\in\mathbb{R}^n$ the configuration and $z$ the motion-task conditioning. The computation proceeds in three steps: (i) map $x$ into a latent coordinate $y\in\mathbb{R}^n$ using a $z$-conditioned bijective encoder (a learned diffeomorphism); (ii) evaluate the designed latent dynamics to obtain $\dot{y}$; and (iii) project $\dot{y}$ back to the original space via the encoder’s inverse Jacobian. While the architecture can be trained under various regimes (e.g., reinforcement learning), this paper focuses on imitation learning—specifically, behaviour cloning—where both the latent representation y and the predicted velocity $\dot{x}$ are supervised.

\subsection{Related Work}
While there is a long history of research on both discrete and rhythmic/periodic DMPs~\cite{ijspeert2002learning, kober2009learning, ijspeert2013dynamical, wensing2017sparse, kramberger2018passivity, saveriano2023dynamic, abu2024learning, hu2024fusion, nah2025combining}, the expressive power of classical DMPs~\cite{ijspeert2002learning, kober2009learning, ijspeert2013dynamical, wensing2017sparse} is limited, preventing them from learning highly complex and intricate trajectories.

Recently, however, an exciting research direction has emerged that combines diffeomorphisms into a latent space—learned using ML techniques—with relatively simple, analytically tractable (e.g., linear) latent space dynamics to enhance expressiveness while preserving stability and convergence guarantees~\cite{rana2020euclideanizing, urain2020imitationflow, zhang2022learning, perez2023stable}. Most of these works focus on point-to-point motions and aim to ensure global asymptotic stability~\cite{rana2020euclideanizing, zhang2022learning, perez2023stable, perez2024puma}, although there have also been several works combining diffeomorphic encoders with rhythmic latent dynamics for learning periodic motions from demonstration~\cite{urain2020imitationflow, khadivar2021learning, zhi2024teaching}.
However, in the existing methods, either the chosen architecture for the bijective encoder lacks expressiveness~\cite{urain2020imitationflow, khadivar2021learning}, the method training is very sensitive to the initial neural network parameter~\cite{urain2020imitationflow}, the method does not learn the demonstrated velocities but only the general direction of motion~\cite{zhi2024teaching}, or cannot accurately learn very complex oracle shapes~\cite{zhi2024teaching}.

Although the proposed model architecture is similar to that of Zhi \textit{et al.}~\cite{zhi2024teaching}, our training pipeline differs substantially: we incorporate an imitation loss that teaches the model the demonstrated velocities, replace the Hausdorff-distance objective with a limit-cycle matching loss better suited to complex or discontinuous paths, optionally guide latent polar angles with a time-alignment term to capture highly curved, possibly concave, contours, and regularize workspace velocities outside the demonstration to improve numerical stability during inference. 
Finally, we allow a parametrization of the polar angular velocity with a neural network, allowing the learning of complex velocity profiles along the limit cycle without compromising the strong contraction guarantees. 

Furthermore, the above-mentioned methods do not offer solutions for many practical issues, such as synchronizing multiple systems—a common requirement in locomotion or bimanual manipulation~\cite{gams2015accelerating}—or to shape the learned velocity online.

We provide a comparison with relevant existing methods in Table~\ref{tab:osmp_characteristics_vs_baselines} (Supplementary Text).

\section{Results}\label{sec:results}





\subsection{OSMPs are Asymptotically Orbitally Stable and Transverse Contracting}
In the general setting, we show that OSMPs possess Asymptotic Orbital Stability (AOS). This property was noted—but not fully proved—in earlier work~\cite{urain2020imitationflow,zhi2024teaching}. In this paper, we formally prove AOS in Theorem~\ref{theorem:asymptotic_orbital_stability} (Supplementary Text).


To our knowledge, earlier studies have not tackled exponential (orbital) stability or contraction \cite{lohmiller1998contraction}. When a velocity scaling is applied in the original coordinates—as in Euclideanizing flows \cite{rana2020euclideanizing}—such guarantees cannot be established unless the scaling factor is explicitly bounded. By contrast, if no velocity scaling is used in the x-coordinates (i.e., $f_\mathrm{s}=1$), we can prove contraction in the directions orthogonal to the limit cycle, a property known as transverse contraction \cite{manchester2014transverse}. Transverse contraction implies Exponential Orbital Stability (EOS), ensuring trajectories converge to the limit cycle at an exponential rate.

\begin{Theorem}\label{theorem:transverse_contraction}
    Let $\alpha, \beta > 0$, $R > 0$, and $z \in \mathbb{R}$ be constants.
    Also, choose $f_\mathrm{s}(x) = 1$ and $\omega(\Bar{y}_{1:2}): \mathbb{R}^2 \to \mathbb{R}_{>0}$.
    Then, the OSMP dynamics $\dot{x} = f(x;z)$ defined in \eqref{eq:dynamics} are transverse contracting 
    in the region $\mathcal{X} = \left \{ x \in \mathbb{R}^n | \sqrt{\{ \Psi(x,z) \}_1^2 + \{ \Psi(x,z) \}_2^2} > 0 \right \}$.
\end{Theorem}
\begin{proof}
    The proof can be outlined as follows:
    \begin{enumerate}
        \item In Proposition~\ref{prop:polar_latent_dynamics_transverse_contraction} (Supplementary Text), we prove that the polar latent dynamics are transverse contracting in the region $\mathcal{Y}_\mathrm{pol} = \left \{ y_\mathrm{pol} \in \mathbb{R}^n | r > 0, \varphi \in [-\pi, \pi) \right \}$.
        \item Proposition~\ref{prop:cartesian_latent_dynamics_transverse_contraction} (Supplementary Text) shows that the same transverse contraction properties hold in the Cartesian latent coordinates with dynamics $\dot{y} = f_\mathrm{y}(y)$.
        \item Existing work~\cite{manchester2017control, mohammadi2024neural, jaffe2024learning} demonstrates that such contraction properties also hold after a change of coordinates $x = \Psi^{-1}(y;z)$ that is defined by a smooth diffeomorphism, which is the case for our encoder based on Euclideanizing flows~\cite{rana2020euclideanizing}. 
    \end{enumerate}
\end{proof}

Formal definitions of transverse contraction and exponential orbital stability (EOS) are provided in Def.~\ref{def:transverse_contraction} and Def.~\ref{def:exponential_orbital_stability} (Supplementary Text). Intuitively, Theorem~\ref{theorem:transverse_contraction} tells us that two trajectories starting from any initial conditions outside the exact center of the limit cycle will converge exponentially to the same periodic orbit~\cite{manchester2014transverse}, demonstrating almost-global contraction. This, in turn, implies almost-global exponential orbital stability: no matter the initial conditions (as long as they are outside the center of the limit cycle with $r = 0$), the trajectories will reach the stable limit cycle specified by the OSMP in exponential time.

\subsection{OSMPs Exhibit a High Imitation Fidelity and Ensure Global Convergence to the Oracle}\label{sub:osmp_benchmarking}

We conduct both quantitative and qualitative evaluations of OSMPs against several baselines. 
Specifically, we evaluate the transverse contracting/exponentially stable variant of OSMP with $f_\mathrm{s}(x) = 1$.
The baselines include classical neural motion policies—MLPs, RNNs, and LSTMs—that directly predict the next system position, plus NODEs \cite{chen2018neural}, which instead predict the desired velocity. We also compare with state-of-the-art robotic imitation-learning methods such as Diffusion Policies (DPs) that predict system trajectory over a horizon and existing SMPs designed for periodic motion, namely Imitation Flows (iFlow) \cite{urain2020imitationflow} and Stable Periodic Diagrammatic Teaching (SPDT) \cite{zhi2024teaching}, predicting system velocities.

Table \ref{tab:benchmarking_quantitative_results} summarizes the quantitative benchmarking of OSMPs versus the baselines, assessing both imitation fidelity and convergence characteristics. To gauge imitation quality, we compute trajectory and velocity RMSE, together with the Dynamic Time Warping (DTW) distance, following prior work \cite{urain2020imitationflow, perez2023stable, nawaz2024learning}.
Convergence is examined at two levels. For local convergence, the system is initialized near a demonstration; we roll out each policy for one estimated period and compare the resulting shapes using directed Hausdorff distance and Mean Euclidean Distance (MED) after aligning the sequences with Iterative Closest Point (ICP). This scenario reflects small deviations from the desired limit cycle caused by low-level control errors or external disturbances. For global convergence, the system starts farther from the demonstrations. We roll out the motion policy for two full periods, then compute the same shape metrics on the second half of the rollout, allowing each policy sufficient time to settle into its limit cycle before measuring how closely that cycle matches the target demonstration.

Our benchmarks span several dataset categories. We include datasets used in earlier studies—such as the IROS letter drawings \cite{urain2020imitationflow} and other 2-D shapes \cite{nawaz2024learning}—as well as particularly challenging 2-D image contours (e.g., Star, MIT CSAIL and TU Delft flame logos, Dolphin, Bat), whose tight curves and discontinuous velocity profiles test the methods limits. In addition, we employ turtle-swimming datasets generated from biologically inspired oracles; unlike previous work, these sequences require reproducing not only the positional trajectory but also its complex, nonlinear velocity profile. 

Please note that we train a separate motion policy on each dataset contained in the dataset category and report the mean of all datasets and demonstrations contained in a dataset category. In order to give statistical relevance to the results, we conduct three training runs on each model+dataset combination, where we initialized the neural network weights in each run with a different random seed. Subsequently, we report the mean and standard deviation across the three random seeds.

The results in Tab.~\ref{tab:benchmarking_quantitative_results} indicate that, across most dataset categories and evaluation metrics, OSMPs outperform the baselines. They not only converge more reliably than neural policies without formal guarantees, but also imitate the demonstrations and align their limit cycles to intricate periodic shapes more accurately than other orbitally stable methods, including iFlow~\cite{urain2020imitationflow} and SPDT~\cite{zhi2024teaching}.
This conclusion is also supported by the qualitative benchmarking in Fig.~\ref{fig:qualitative_benchmarking_results}, that shows while some of the baseline methods, such as MLPs, Neural ODEs, or SPDT~\cite{zhi2024teaching}, might be sufficient for simpler oracles, such as the RShape~\cite{urain2020imitationflow}, the planar drawing~\cite{nawaz2024learning}, or the Star oracle, but fail to track the periodic demonstration for more complex and highly curved oracles, such as the TUD-Flame or the Dolphin image contour.
In an ablation study with results reported in Tab.~\ref{tab:ablation_study_loss_functions} and Fig.~\ref{fig:ablation_study_loss_functions} (Supplementary Text), we furthermore demonstrate that the training loss terms proposed in this paper all contribute towards improving the reported performance metrics, although the most suitable combination of loss terms depends on the complexity of the oracle shape.

More details about the implementation of the baseline methods, the evaluation metrics, and the datasets can be found in the supplementary materials~\cite{methods}.

\subsection{Stable and Accurate Tracking of Oracles Across Robot Embodiments}

While the previous section focused on evaluating and benchmarking the learning of the OSMP, we now aim to demonstrate that the proposed OSMP can effectively control robot motion in real-world scenarios. To achieve this, we apply the method to a diverse range of robot embodiments, including robot manipulators (UR5), cobots (KUKA), continuum soft robots (Helix Soft Robot)~\cite{guan2023trimmed}, and prototypes of hybrid soft-rigid underwater robots (Crush turtle robot).
Figure~\ref{fig:robot_embodiments_results} illustrates the effectiveness of OSMPs across all tested robot embodiments.

First, we deploy the OSMPs trained on image contours on both the UR5 arm and the Helix soft robot, achieving accurate and stable contour tracking. The deviations and oscillations seen on the Helix stem not from the OSMP itself but from the low-level controller—particularly inverse-kinematics errors—as demonstrated in Fig.~\ref{fig:robot_embodiments_extended_results} and Movie~S1, where we benchmark against a classical trajectory-tracking controller.
Also, a quantitative evaluation of the imitation metrics and shape similarity between the actual system trajectory and the desired oracle shape is contained in Tab.~\ref{tab:robot_embodiments_quantitative_evaluation} (Supplementary Text).

We then target swimming behavior on the Crush Turtle robot using biologically inspired oracles collected by marine biologists. Our goal is for the OSMP to drive the two front flippers, the main propulsion surfaces. We use both a three-dimensional joint-space oracle~\cite{van2023soft} and a four-dimensional task-space oracle comprising flipper-tip position and twist~\cite{van2022new}, each derived from video recordings of green sea turtles (Chelonia mydas)~\cite{van2022new,van2023soft}. The resulting joint-space velocity commands are executed by the robot’s actuators. Experiments show that OSMPs enable accurate tracking of the biological oracle at moderate speeds. Because of joint-motor velocity and acceleration limits, the system cannot perfectly track shape or speed at higher $s_\omega$ values; yet even when motion diverges slightly, stability is preserved, the trajectory rapidly reconverges to the oracle, and the turtle robot successfully swims. 
We observe that an OSMP trained on the joint-space swimming oracle yields more effective propulsion than one based on the task-space oracle—likely because the latter omits the full 3-D pose of the flippers. In addition, the joint-space OSMP avoids kinematic singularities that can destabilize task-space control.

Next, we test OSMP performance on kinesthetic-teaching demonstrations, which are typically jerkier and less smooth than the oracles above. In periodic demonstrations, the trajectory often fails to close exactly, leaving a spatial offset between start and end poses that complicates limit-cycle fitting. We investigate a whiteboard-erasure task on a UR5 manipulator and a brooming task on a KUKA cobot. For the UR5, we encode only the end-effector positions, whereas for the KUKA, we encode both position and orientation. On the UR5, we observe successful task completion (i.e., cleaning the writing from the whiteboard), rapid convergence to the limit cycle, and strong oracle tracking, with only minor errors where start and end points were fused. On the KUKA, tracking error is somewhat larger—likely due to the demanding six-DOF oracle and low feedback gains in the low-level controller—but the robot still completes the task reliably and repeats it with high consistency, even remaining robust to external disturbances and perturbations as seen in Movie~S2.

\subsection{The Learned Policies Exhibit Compliant and Natural Motion Behavior}

We aim for robots in human-centric environments to demonstrate robust, compliant, and predictable behavior. Specifically, \emph{robustness} means that if a robot deviates from its intended path—perhaps due to a disturbance—it will always converge back to the desired motion. \emph{Compliance} indicates that robots should exert only minimal forces when coming into contact with humans, and \emph{predictability} ensures that their motions are sufficiently consistent for humans to anticipate their behavior and respond appropriately.

In this section, we compare the reaction upon disturbances and perturbations of OSMPs against classical trajectory tracking controllers that rely on a time-parametrized trajectory, given in the form
\begin{equation}
    \dot{x}(t) = \dot{x}^\mathrm{d}(t) + k_\mathrm{p} \, (x^\mathrm{d}(t) - x(t)),
\end{equation}
where $k_\mathrm{p} \in \mathbb{R}$ is a proportional feedback gain that operates on the error between the current position $x(t)$ and the desired position $x^\mathrm{d}(t)$. We stress here the reliance on a time-parametrized trajectory provided in the form $(x^\mathrm{d}(t),\dot{x}^\mathrm{d}(t)) \: \forall \: t \in [t_0, t_\mathrm{f}]$. 
We evaluate three motion controllers: a pure feedforward trajectory tracking controller, which we gather by setting $k_\mathrm{p} = 0$, an error-based feedback controller with $k_\mathrm{p} > 0$, and the learned OSMP.
In this setting, we are particularly interested in analyzing the behavior of the motion controllers upon encountering an external disturbance that perturbs the state of the system with respect to the time reference. For example, in simulation, we shift the time reference when initializing the system by half a period (i.e., a phase shift of $\pi$~rad) and in the real world experiments with the KUKA robot running a low level impedance controller we apply external perturbations to the system that prevents or disturbs the nominal motion.

The results in Fig.~\ref{fig:compliance_results} and Movie~S3 show that the pure feedforward trajectory tracking controller entirely drifts off the desired trajectory. When adding an error-based feedback term, the classical trajectory tracking controller is able to recover and rejoin the demonstrated trajectory after a bit. However, while doing so, the feedback term generates a very aggressive correction action, which could cause incompliant behavior and would not seem natural to humans. Instead, the OSMP, which is solely conditioned on the system state and not time, is not affected by the perturbation of the time reference and perfectly tracks the demonstration, immediately returning to the closest point on the limit cycle after a perturbation, while exhibiting compliant and natural behavior.

\subsection{Achieving Phase Synchronization Across Multiple Motion Primitives}

In many practical applications, such as locomotion or bimanual manipulation, synchronizing multiple motion policies is critical. In this section, we illustrate how our approach can synchronize multiple learned OSMPs by evaluating the polar phase of each and then aligning them via an error-based feedback controller~\cite{dorfler2014synchronization}. Crucially, we only adjust the velocity magnitude without altering the system’s spatial motion, thereby preserving the imitation and convergence properties of each learned motion policy.

In Fig.~\ref{fig:phase_sync_results} and Movie~S4, we show simulation and experimental results for synchronizing between two and six OSMPs. The simulation outcomes illustrate how the controller identifies the most efficient strategy to align the OSMPs, achieving rapid polar phase synchronization. A proportional gain determines the aggressiveness of the synchronization process. The simulation results confirm that the phase synchronization approach is effective not only for two systems but also for three or more.
In complex systems with many DOFs, we have found it can be more effective to train separate OSMPs and synchronize them during execution rather than relying on one large OSMP that covers every DOF. With a joint OSMP, a disturbance in even a few DOFs can pull the rest off the limit cycle—and away from the oracle—until the system reconverges. By contrast, synchronized but independent OSMPs are insulated from such disturbances: if some DOFs are perturbed, the remaining DOFs managed by their own OSMPs can keep tracking the oracle accurately, and the phase synchronization ensures that all OSMPs stay locked in their phase.

Regarding experimental findings, we examined the swimming performance of the Crush turtle robot. Tests in a swimming pool revealed that the robot can swim effectively only when both front flippers—the primary means of locomotion in water~\cite{van2022new, van2023soft}—are fully synchronized. In practice, even aside from external disturbances and inherent differences between the flippers, desynchronization occurs already during initialization when the flipper arms start in slightly different configurations with varying polar phases. Our results show that using our method, the two flipper arms synchronize, even with flippers initialized far from the oracle, within 4.72s to less than $1^\circ$ in phase error, and subsequently for the entire experiment exhibit a mean phase error of less than $0.2^\circ$, thereby enabling the turtle robot to swim effectively.

\subsection{Smooth Interpolation Between Motion Behaviors via Encoder Conditioning}
As robotics shifts toward generalist motion policies that choose among varied behaviors based on task, state, and perception, those policies must support multiple skills rather than one~\cite{o2024open, black2024pi0, gemini2025robotics}. Leveraging semantic cues—e.g., embeddings from vision–language models—could supply such conditioning, yet dynamic motion-primitive work seldom tackles it. Existing methods~\cite{rana2020euclideanizing, perez2023stable, perez2024puma, sochopoulos2024learning, zhi2024teaching} also lack smooth interpolation between trained behaviors.
Examples include applications such as surface cleaning, where the robot must in the future seamlessly switch wiping motions as materials change. In locomotion, blending oracles for flat walking and stair climbing enables natural movement over moderately stepped terrain. Such cases underscore the need for motion policies that transfer to unseen tasks with few- or zero-shot generalization~\cite{jaquier2025transfer}.

We introduce task conditioning in the bijective encoder through a scalar variable $z \in \mathbb{R}$, allowing the desired motion behavior to be selected online by simply setting $z$. To ensure the learned policy transitions smoothly across behaviors (e.g., for $z \in [-1, 1]$), we add a loss term $\mathcal{L}\mathrm{sci}$ during training. As illustrated in Fig.~\ref{fig:conditioning_results} and Movie~S5, both simulation and hardware experiments with the turtle robot show that (a) a single OSMP faithfully reproduces all behaviors encountered during training, and (b) $\mathcal{L}\mathrm{sci}$ promotes smooth interpolation between oracles, enabling meaningful zero-shot performance on unseen tasks that fall within the training distribution. An ablation study—comparing against an OSMP trained without $\mathcal{L}_\mathrm{sci}$—confirms this finding. Crucially, switching behaviors requires no elaborate sequence: once $z$ is updated, the OSMP’s convergence guarantees rapidly steer the system to the new behavior, preserving EOS for constant or slowly varying conditioning values.

\section{Discussion}\label{sec:discussion}

\paragraph{Convergence Guarantees}
The first result introduced in the previous section is theoretical in nature: we establish that the proposed OSMP guarantees almost-global asymptotic orbital stability. Under mild architectural conditions, this further strengthens into almost-global exponential orbital stability via transverse contraction. Yet this result is not merely of theoretical interest—it carries profound practical implications. Indeed, it ensures that, regardless of network weights or conditioning, any trajectory converges exponentially fast to the learned limit cycle. Exponential convergence prevents slow-dynamics plateaus where the system might otherwise stall, while the near-global basin drastically reduces the need for far-reaching training data coverage. Together, these properties yield significant gains in data efficiency, as they obviate the need to densely sample the state space. Beyond training, exponential convergence also facilitates system-level analysis and enhances modularity. Most critically, contractiveness plays a central role in enabling policy reuse and transfer learning: while interconnecting merely stable systems can lead to instability, contracting systems are provably composable and lend themselves naturally to transfer and hierarchical control~\cite{ofir2022sufficient,angeli2025lmi}. 

In the same spirit, it is important to underscore that our stability results hold uniformly for any fixed conditioning value $z$, thereby reinforcing their relevance for transferability. This property suggests robustness not only to different instances within a task family but also to transitions across them. While outside the current scope, it is natural to ask whether convergence still holds under time-varying conditioning $z(t)$, such as during policy switching. Thanks to the global nature of the underlying result, the extension to the case where $z(t)$ eventually stabilizes (i.e., $z(t) \equiv z_\infty$ for $t > t'$) is straightforward. More ambitiously, we believe that the exponential convergence rate can be leveraged to establish stability even under piecewise-constant $z(t)$, including non-smooth transitions over finite horizons. This opens the door to formal tools for designing learning-based schedules of optimal conditioning patterns—a direction we leave for future investigation.

\paragraph{Quantitative Benchmark}

Quantitative benchmarks show that OSMPs surpass most baselines across the majority of dataset categories. They display superior convergence—especially when initialized far from the demonstration—compared with classic neural motion policies such as MLPs, RNNs, LSTMs, NODEs, and even state-of-the-art DPs~\cite{chi2023diffusion}. While other orbitally stable approaches like iFlow~\cite{urain2020imitationflow} and SPDT~\cite{zhi2024teaching} guarantee convergence to a limit cycle, they often struggle to (a) imitate highly curved shapes or discontinuous velocity profiles and (b) ensure that the resulting limit cycle accurately reproduces those complex demonstrations.
For example, compared to SPDT~\cite{zhi2024teaching}, OSMPs are able to imitate the velocity profile, and their limit cycle captures the oracle shape 5x more accurately, as the global convergence analysis shows.
The few cases in which baselines outperform OSMPs highlight avenues for improvement. On the IROS letters dataset~\cite{urain2020imitationflow}, DPs~\cite{chi2023diffusion} achieve lower imitation errors, likely because widely separated demonstrations of the same letter must be captured by a single policy, which favors probabilistic methods like DPs and iFlow over deterministic ones, such as NODEs, SPDT, and OSMPs. For image-contour data, MLP policies slightly edge out OSMPs when starting near the demonstration but degrade markedly when initialized farther away.
Regarding inference time/runtime, OSMPs demand more computation than classical neural policies (which forgo gradient evaluations during inference) and iFlow~\cite{urain2020imitationflow} (given our more expressive bijective encoders), yet they strike a favorable performance-to-inference-time balance and are 10–20× faster than standard diffusion policies~\cite{chi2023diffusion}. Note, too, that the timings in Table~\ref{tab:benchmarking_quantitative_results} stem from unoptimized, eager-mode inference; with numerical gradients and PyTorch ahead-of-time compilation, OSMPs can run at up to 15 kHz on modern CPUs, as shown in Tab.~\ref{tab:inference_time_benchmarking} (Supplementary Text).


\paragraph{Real-World Experiments}
When moving to real-world experiments, we quickly realized how the interpretable latent-dynamics structure allows us to inspect post-training how closely the learned cycle mirrors the periodic demonstration, promoting predictable behavior when the policy is deployed on a real robot.
This aspect will require further quantification in future work.

On the UR5 robot, OSMPs are able to track the oracle with directed Hausdorff distance between $3$~mm and $14$~mm. As the Helix soft robot employs a less accurate low level controller, the shape tracking accuracy drops to $6-11$~mm. Still, OSMPs outperform classical trajectory tracking controller in terms of shape accuracy on the Helix soft robot by approximately $60-65$~\%.

OSMPs enable highly successful locomotion in the turtle robot by leveraging biologically derived swimming oracles. Although oracle tracking is not perfectly precise (directed Hausdorff distance of $0.085$-$0.382$~rad)—especially at higher speeds due to motor limits and external disturbances like water drag—the system remains stable, quickly reconverges to the limit cycle, and keeps the limbs nearly perfectly synchronized.

Next, we trained individual OSMPs on multiple kinesthetic-teaching demonstrations for the UR5 and KUKA manipulators. Although some demonstrations were jerky and uneven, the OSMPs accurately captured the intended motions: the UR5 achieved a 100~\% success rate in cleaning the whiteboard using the learned policy. Similarly, the KUKA robot consistently completed the brooming task over many repetitions—despite limited execution speed and tracking accuracy imposed by the low feedback gains of its impedance controller—and even succeeded when its motion was perturbed by external disturbances.

Although orbital stability is always guaranteed, large deformations in the learned diffeomorphism can push the system far from the demonstrated path. On real robots, that drift is risky because joint position and velocity limits—or a finite task-space workspace—can be violated. We saw this on the Crush turtle robot, whose joint range and velocity/acceleration caps constrain what the low-level controller can track. Two design choices proved helpful: (i) encoder regularization during training and (ii) a sliding-mode–style motion modulation that first draws the system into a neighborhood of the oracle before advancing along the polar phase. Since motion directly on the oracle is usually feasible, these measures prevent most of the problems. Looking ahead, embedding Control Barrier Functions (CBFs) could explicitly keep the system out of infeasible or unsafe regions in the oracle space, echoing recent advances in the DMP literature \cite{davoodi2022rule,nawaz2024learning,mohammadi2024extended,simmoteit2025diffeomorphic}.

Phase synchronization is vital when deploying OSMPs for turtle swimming: without it, the limb controllers—lacking an explicit time parameter—would drift out of phase, sharply increasing the cost of transport. Instead, our phase synchronization strategy is able to keep the mean phase error between the two flippers at less than $0.2^\circ$.
The same synchronization strategy can be applied in the future to other platforms, such as bipedal or quadrupedal robots. We also demonstrate that training separate OSMPs and synchronizing their phases at runtime can outperform a single large OSMP, as disturbances affecting one subsystem do not pull the others away from their limit cycles.

Conditioning DMPs on task information enables them to produce purposeful motions for tasks never encountered during training, marking a paradigm shift that opens the door to zero-shot transfer of convergence-guaranteed motion policies to far more complex, unseen scenarios. Here, we take an initial step toward that vision: our training procedure enables smooth interpolation between motions observed in the dataset. For instance, a single OSMP trained on both forward and reverse turtle swimming can effortlessly blend the two behaviors into a continuum.
Crucially, the orbital stability is always preserved, even when the shape of the stable limit cycle changes as a function of the conditioning $z$.

\paragraph{Limitations}
The proposed OSMP exhibits several limitations that highlight promising directions for future development. First, it presumes prior annotation of the demonstration’s attractor regime, requiring that each trajectory be segmented to isolate the periodic portion and implicitly modeled as a limit cycle. While our approach accommodates sharp turns and discontinuities more robustly than existing baselines, such features still induce sizable deformations in the learned vector field. These deformations manifest as locally aggressive dynamics, where nearby trajectories may temporarily diverge markedly before reconverging toward the limit cycle.

Second, like other time-invariant dynamical motion primitive frameworks~\cite{ijspeert2002learning, ijspeert2013dynamical, rana2020euclideanizing, perez2023stable}, OSMP inherits the limitation of being ill-posed under intersecting demonstrations or oracles—situations that demand multi-valued flows at a single point in state space. In addition to existing ideas in literature~\cite{sun2024directionality}, formulating the motion policy as second-order dynamics or augmenting the dynamics with an explicit notion of trajectory progress could resolve this ambiguity by lifting the state into a higher-dimensional phase space.

Looking ahead, we envision extending the framework to support multiple classes of attractors—such as point attractors (GAS), multistable basins (MS), and limit cycles (OS)—within a single primitive. The supercritical Hopf bifurcation already captures a transition between equilibrium and periodic behavior, suggesting that the inclusion of an additional “attractor type” parameter could generalize the formalism to support richer behaviors~\cite{strogatz2018nonlinear}. Finally, replacing the explicit conditioning variable $z$ with implicit observation-based embeddings (e.g., task images or object states) would make OSMP compatible with vision-language-action models such as $\pi\_0$~~\cite{black2024pi0} and SmolVLA~~\cite{shukor2025smolvla}, offering a path toward more expressive and versatile control policies.

\section{Materials and Methods}\label{sec:materials_and_methods}

In this work, we introduce OSMPs, which can be trained to capture complex periodic motions from demonstrations while ensuring convergence to a limit cycle that aligns with a predefined oracle. To accomplish this, we build on previous research~\cite{rana2020euclideanizing, zhi2024teaching} that combines learned bijective encoders with a prescribed motion behavior in latent space. These latent dynamics generate velocities or accelerations that are subsequently mapped back into the oracle space via a pullback operator~\cite{zhang2022learning}—in the case of a bijective encoder, this operator is the inverse of the encoder’s Jacobian. In this formulation, the motion in latent space exhibits key convergence properties, such as Global Asymptotic Stability (GAS)~\cite{rana2020euclideanizing, perez2023stable, sochopoulos2024learning} or Orbital Stability (OS)~\cite{zhi2024teaching}, while the learned encoder provides the necessary expressiveness to capture complex motions and transfers the convergence guarantees from latent to oracle space through the established diffeomorphism.

However, compared to existing work~\cite{zhi2024teaching}, we introduce several crucial modifications that enhance both the performance and practical utility of the proposed method: (1) we develop a limit cycle matching loss to reduce the discrepancy between the learned limit cycle and the periodic oracle; (2) we design strategies to modulate the learned velocity field online without the need for retraining—for instance, to adjust the convergence behavior; (3) we establish a procedure to synchronize the phase of multiple OSMPs; and (4) we condition the encoder on a task, enabling a single OSMP to exhibit multiple distinct motion behaviors. Moreover, we introduce loss terms that allow the trained OSMP to smoothly interpolate between the learned motion behaviors—something that has not been possible before.

\subsection{Orbitally Stable Motion Primitives}
The dynamical motion policy $\dot{x}=f(x; z)$ is typically defined in task space but can also be defined in other Cartesian or generalized coordinates (e.g., joint space). Therefore, we will in the following refer to these coordinates as \emph{Oracle space}.
In this work, we are specifically interested in cases where we train $f(x)$ to learn periodic motions.
Then, the smooth diffeomorphism between the oracle and latent space is made via a bijective encoder $\Psi: \mathbb{R}^n \to \mathbb{R}^n$, which maps positional states $x \in \mathbb{R}^n$ into the latent states $y \in \mathbb{R}^n$.
Optionally, this encoding is conditioned on a continuous variable $z \in \mathbb{R}$ as a homotopy such that $y = \Psi(x;z)$.
Furthermore, we construct $\Psi$ such that it is invertible and a closed-form inverse function $\Psi^{-1}: y \mapsto x$ allows us to map from latent space back into oracle space.
In latent space, we apply dynamics $\dot{y} = f_\mathrm{y}(y)$ that exhibit a stable limit cycle behavior. In summary, the orbitally stable motion primitive is given as
\begin{equation}\label{eq:dynamics}
    \dot{x} = f(x;z) = f_\mathrm{s}(x) \, J_\Psi^{-1}(x;z) \, f_\mathrm{y} \left (\Psi(x;z) \right ),
\end{equation}\
where $J_\Psi = \frac{\partial \Psi(x;z)}{\partial x}$ defines the Jacobian of the encoder. 
The function $f_\mathrm{s}(x): \mathbb{R}^n \to \mathbb{R}_{>0}$ scales the velocity magnitude and is implemented as $f_\mathrm{s}(x) = e^{\mathrm{MLP}(x)} + \varepsilon$, where $\varepsilon \in \mathbb{R}_{>0}$ is a small value~\cite{rana2020euclideanizing}.
As $\Psi$ is a diffeomorphism (w.r.t. $x$ and $y$), the motion policy is orbitally stable by construction~\cite{rana2020euclideanizing, zhi2024teaching}.

\subsubsection{Diffeomorphic Encoder}

We consider a bijective encoder $\Psi_\theta : \mathbb{R}^n \times \mathbb{R} \to \mathbb{R}^n$, which maps positional states $x \in \mathbb{R}^n$ into the latent states $y \in \mathbb{R}^n$ conditioned on $z \in \mathbb{R}$, where we assume that $n \in \mathbb{N} \geq 2$.
The encoder $y = \Psi_\theta(x;z)$ adopting the Euclideanizing flows~\cite{dinh2017density, rana2020euclideanizing} architecture is parametrized by the learnable weights $\theta \in \mathbb{R}^{n_\theta}$ and consists of $n_\mathrm{b}$ blocks, where each block is analytically invertible.
If a conditioning is used, $z$ is first lifted into an embedding $\Bar{z}$, which is subsequently used to condition each block.
More details on the encoder architecture can be found in the pioneering work~\cite{dinh2017density, rana2020euclideanizing} and in the Supplementary Materials~\cite{methods}.

\subsubsection{Latent Dynamics}
In latent space, we consider the 1st-order dynamics of a supercritical Hopf bifurcation~\cite{strogatz2018nonlinear, khadivar2021learning, zhi2024teaching, nah2025combining}
\begin{equation}\label{eq:latent_dynamics}
    \dot{y} = \begin{bmatrix}
        \dot{y}_1\\
        \dot{y}_2\\
        \dot{y}_{3:n}\\
    \end{bmatrix} = f_\mathrm{y}(y) = \begin{bmatrix}
        -\omega(y) \, y_2 + \alpha \, \left ( 1 - \frac{y_1^2 + y_2^2}{R^2} \right ) \, y_1\\
        + \omega(y) \, y_1 + \alpha \, \left ( 1 - \frac{y_1^2 + y_2^2}{R^2} \right ) \, y_2\\
        -\beta \, y_{3:n}\\
    \end{bmatrix},
\end{equation}
where $\omega(y) = f_\omega(\operatorname{atan2}(y_2, y_1)) > 0$ determines the angular velocity.
Here, the dynamics of $y_1$ and $y_2$ describe the Cartesian-space behavior of a simple limit cycle whose behavior in polar coordinates $y_\mathrm{pol} = \begin{bmatrix}
    r & \varphi & y_{3:n}^\top
\end{bmatrix}^\top$ is expressed as
\begin{equation}\label{eq:latent_dynamics_polar_coordinates}
    \dot{y}_\mathrm{pol} = \begin{bmatrix}
        \dot{r}\\ \dot{\varphi}\\ \dot{y}_{3:n}
    \end{bmatrix} = f_\mathrm{pol}(y_\mathrm{pol}) = \begin{bmatrix}
        \alpha \left ( 1 - \frac{r^2}{R^2} \right ) \, r\\ f_\omega(\varphi)\\ -\beta \, y_{3:n}\\
    \end{bmatrix},
\end{equation}
where $f_\omega(\varphi): [-\pi, \pi) \to \mathbb{R}_{>0}$ computes the positive angular velocity as a function of the polar angle. Often, in particular, when employing $f_\mathrm{s}(x) \neq 1$, it can also be set to a constant: $\omega = 1$. 
If not, we define in practice $\omega(y) = f_\omega(\Bar{y}_{1:2}) = e^{\mathrm{MLP}(\Bar{y}_{1:2})} + \epsilon_\omega$, where $\Bar{y}_{1:2} = \begin{bmatrix}
    \frac{y_1}{\sqrt{y_1^2+y_2^2}} &  \frac{y_2}{\sqrt{y_1^2+y_2^2}}
\end{bmatrix}^\top \in \mathbb{R}^2$ with $\epsilon_\omega > 0$. More details can be found in the supplementary material~\cite{methods}.

$\alpha > 0$ and $\beta > 0$ are positive gains that determine how fast the system converges onto the limit cycle.  When learning the dynamics, we choose $\alpha = \beta = 1$. $R \in \mathbb{R}_{>0}$ expresses the radius of the limit cycle in latent space. Again, it is sufficient to choose $R =1$ or $R=0.5$.

\subsection{Training}
We consider a dataset $\langle T, X^\mathrm{d}, \dot{X}^\mathrm{d}, Z \rangle$ as a tuple between timestamps $T = \langle t(1), \dots, t(k), \dots, t(N) \rangle$ positions $X^\mathrm{d} = \langle x^\mathrm{d}(1), \dots, x^\mathrm{d}(k), \dots, x^\mathrm{d}(\mathrm{N}) \rangle$, the corresponding, demonstrated velocities $\dot{X}^\mathrm{d} = \langle \dot{x}^\mathrm{d}(1), \dots, \dot{x}^\mathrm{d}(k), \dots, \dot{x}^\mathrm{d}(\mathrm{N}) \rangle$ and an optional conditioning $Z = \langle z(1), \dots, z(k), \dots, z(\mathrm{N}) \rangle$, where $k \in \mathbb{N}_\mathrm{N} = \{1, 2, \dots, N \}$.
The total training loss function is then given by
\begin{equation}
    \mathcal{L} = \underbrace{\zeta_\mathrm{vi} \, \mathcal{L}_\mathrm{vi}}_\text{Vel. Imitation} + \underbrace{\zeta_\mathrm{lcm} \, \mathcal{L}_\mathrm{lcm}}_\text{Limit Cycle Matching} + \underbrace{\zeta_\mathrm{tgd} \, \mathcal{L}_\mathrm{tgd}}_\text{Time Guidance} 
    + \underbrace{\zeta_\mathrm{er} \, \mathcal{L}_{\mathrm{er}}}_\text{Encoder Reg.}
    + \underbrace{\zeta_\mathrm{vr} \, \mathcal{L}_{\mathrm{vr}}}_\text{Vel. Reg.}  + \underbrace{\zeta_\mathrm{sci} \, \mathcal{L}_{\mathrm{sci}}}_\text{Cond. Interpolation} 
\end{equation}
where $\zeta_\mathrm{vi}, \zeta_\mathrm{lcm}, \zeta_\mathrm{tgd}, \zeta_\mathrm{er}, \zeta_\mathrm{vr}, \zeta_\mathrm{sci} \in \mathbb{R}$ are scalar loss weights.
$\mathcal{L}_\mathrm{vi}$ is a loss term that enforces that the velocity of the motion primitive matches the one demanded by the demonstration at all samples in the demonstration dataset, $\mathcal{L}_\mathrm{lcm}$ ensures that the encoded demonstration positions lie on the latent limit cycle. The time-guidance loss $\mathcal{L}_\mathrm{tgd}$ can support the limit cycle matching loss for highly curved demonstrations. 
The term $\mathcal{L}_\mathrm{er}$ regularizes the encoder by penalizing deviations from the identity encoder.
The optional $\mathcal{L}_{\mathrm{sci}}$ gives rise to smooth interpolation between different encoder conditioning.
The discretionary velocity regularization loss $\zeta_\mathrm{vr}$ increases the numerical stability by penalizing very high velocities.
More details on the various loss terms can be found in the supplementary materials~\cite{methods}.

\subsection{Phase Synchronization of Multiple Motion Primitives}
In many real-world applications, it is essential to synchronize the phases of multiple learned (orbital) motion primitives~\cite{gams2015accelerating}. For instance, in turtle swimming, the phases of the two limbs must align, while in (human) walking, the periodic movement of the two legs should be offset by $\pi$. To address this, we developed a controller that can synchronize the motion of two or more systems.
Here, we consider that we trained $n_\mathrm{s}$ OSMPs. We refer to the latent state of the $i$th system, where $i \in \mathbb{N}_{n_\mathrm{s}}$, as ${}_{i} y$. The polar phase of each system is given by ${}_{i} \varphi = \operatorname{atan2}({}_{i} y_2, {}_{i} y_1)$. We then construct a symmetric matrix $\delta \Phi^* \in \mathbb{R}^{(n_\mathrm{s}-1) \times (n_\mathrm{s}-1)}$ that contains the desired phase offsets. For example, $\delta \Phi^*_{ij} = \delta \Phi^*_{ji} \in [-\pi, \pi)$ specifies the desired phase offset between the $i$th and the $j$th system. In the case of $\delta \Phi^* = 0^{(n_\mathrm{s}-1) \times (n_\mathrm{s}-1)}$, we ask the phase difference between all systems to be zero.
We then adopt a technique from the field of network synchronization~\cite{dorfler2014synchronization} that allows the alignment of the OSMPs in phase. Namely, we define a feedback controller that acts on the angular velocity of the latent system
\begin{equation}
    {}_{i} \tilde{\omega}(\varphi) = {}_{i} \omega({}_{i}\varphi) \, \left (1  -k_\mathrm{ps} \sum_{j=1}^{n_\mathrm{s}} \sin \left ( \delta \Phi^*_{ij} + {}_{i} \varphi - {}_{j} \varphi \right ) \right ),
\end{equation}
where $\omega(\varphi), \tilde{\omega} \in \mathbb{R}$ are the default and modified polar angular velocities of the systems, respectively. 
$k_\mathrm{ps} \in \mathbb{R}_{>0}$ is the phase synchronization gain that determines how quickly the systems synchronize.

\begin{figure}[h!]
    \centering
    \includegraphics[width=1.0\linewidth]{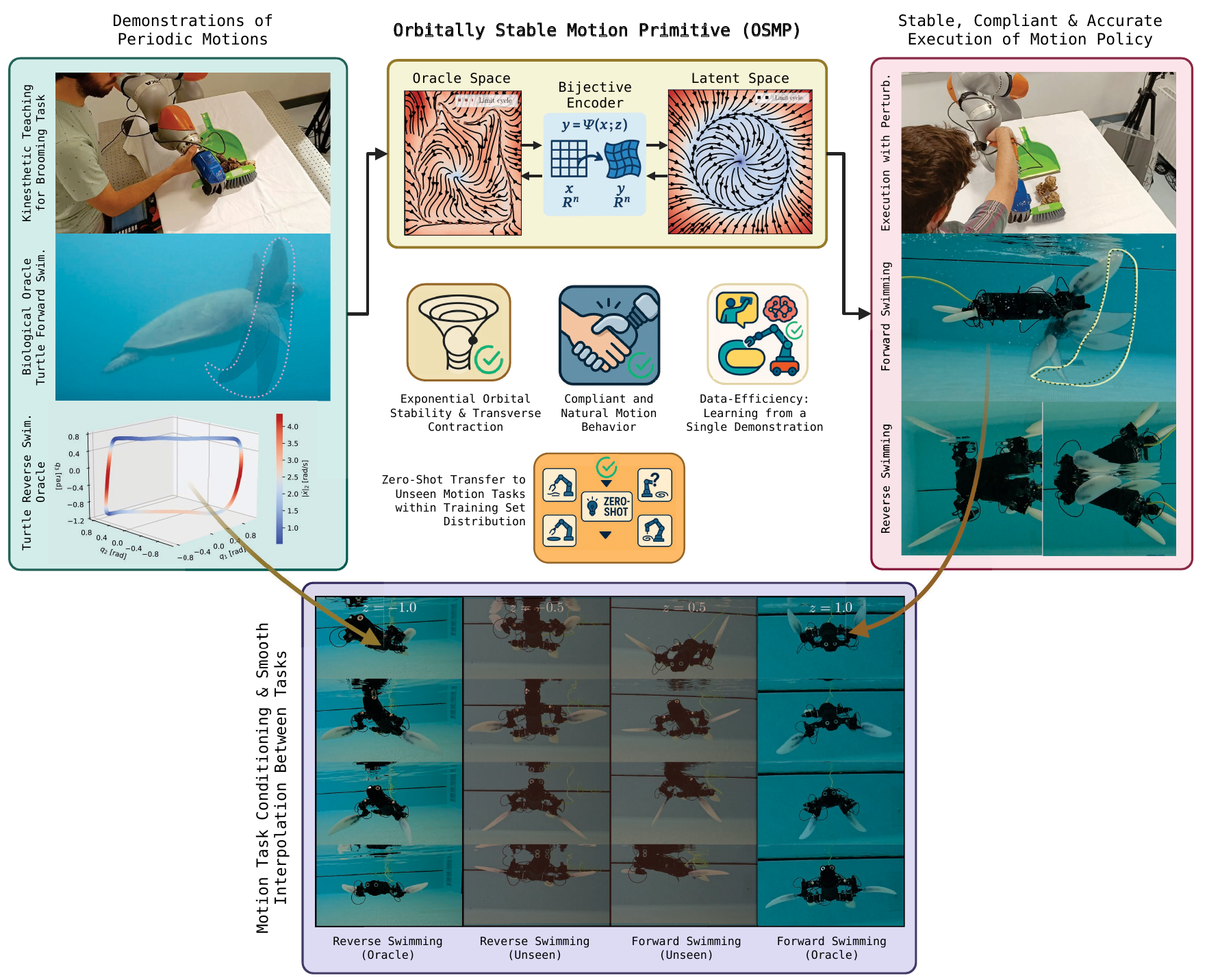}
    \caption{
    \textbf{Overview of Features and Characteristics of Orbitally Stable Motion Primitives (OSMPs).}
    OSMPs enable the learning of periodic motions from demonstrations while guaranteeing exponential orbital stability \& transverse contraction, instilling compliant and natural motion behavior, and outstanding data efficiency as they are able to learn complex motion behaviors from just one demonstration. Furthermore, a motion task conditioning allows the same motion policy to exhibit multiple distinct behaviors while a specially crafted loss term encourages smooth interpolation between motion tasks seen during training leading to meaningful motion behaviors on task unseen during training (i.e., zero-shot transfer).
    }
    \label{fig:concept_overview}
\end{figure}

\begin{figure}[h!]
    \centering
    \includegraphics[width=1.0\linewidth]{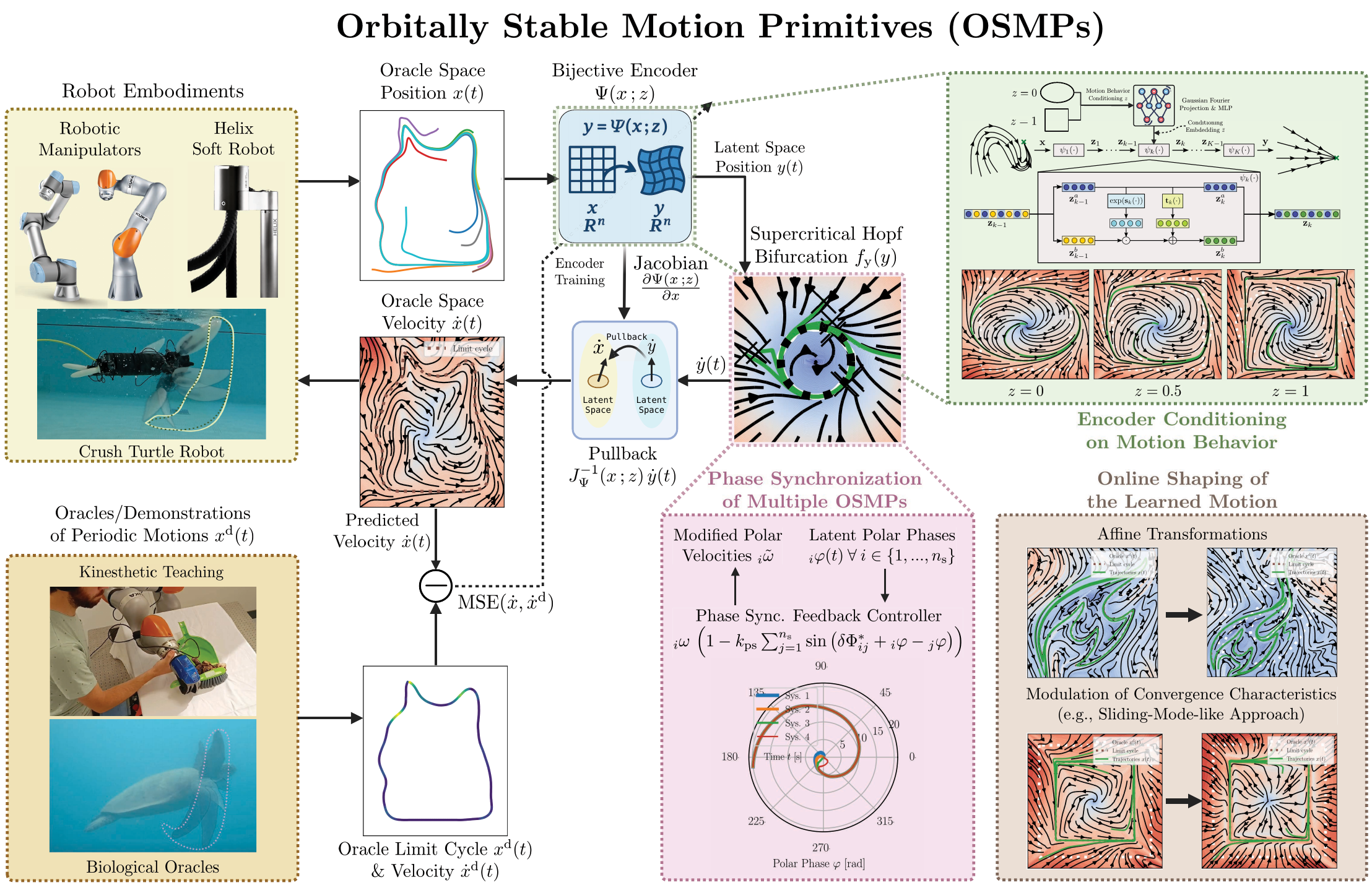}
    \caption{
    \textbf{Methodology of Orbitally Stable Motion Primitives (OSMPs).}
    Periodic motions can be learned from demonstration via a DMP that combines a bijective encoder with latent space dynamics defined by a supercritical Hopf bifurcation. After encoding the current oracle space position of the system into latent space and predicting the latent space velocity, the pullback operator projects the velocity back into oracle space and represents a motion reference for the various robot embodiments. During training, we enforce both the predicted velocity and the exhibited limit cycle to match the demonstrations that are, for example, provided via kinesthetic teaching or based on biological oracles.
    Multiple contributions increase the practical usefulness of the proposed methods, which include a technique to synchronize multiple OSMPs in phase, an approach to shape the learned motion online without requiring retraining via affine transformations or even modulating the convergence characteristics, and a methodology for conditioning the encoder on the oracle which allows the OSMP to capture multiple distinct motion behaviors and even smoothly interpolate between them.
    The depiction of the Euclideanizing flows architecture is adapted from Rana \textit{et al.}~\cite{rana2020euclideanizing}.
    }
    \label{fig:methodology_overview}
\end{figure}

\begin{figure}[ht!]
    \centering
    \includegraphics[width=1.0\linewidth]{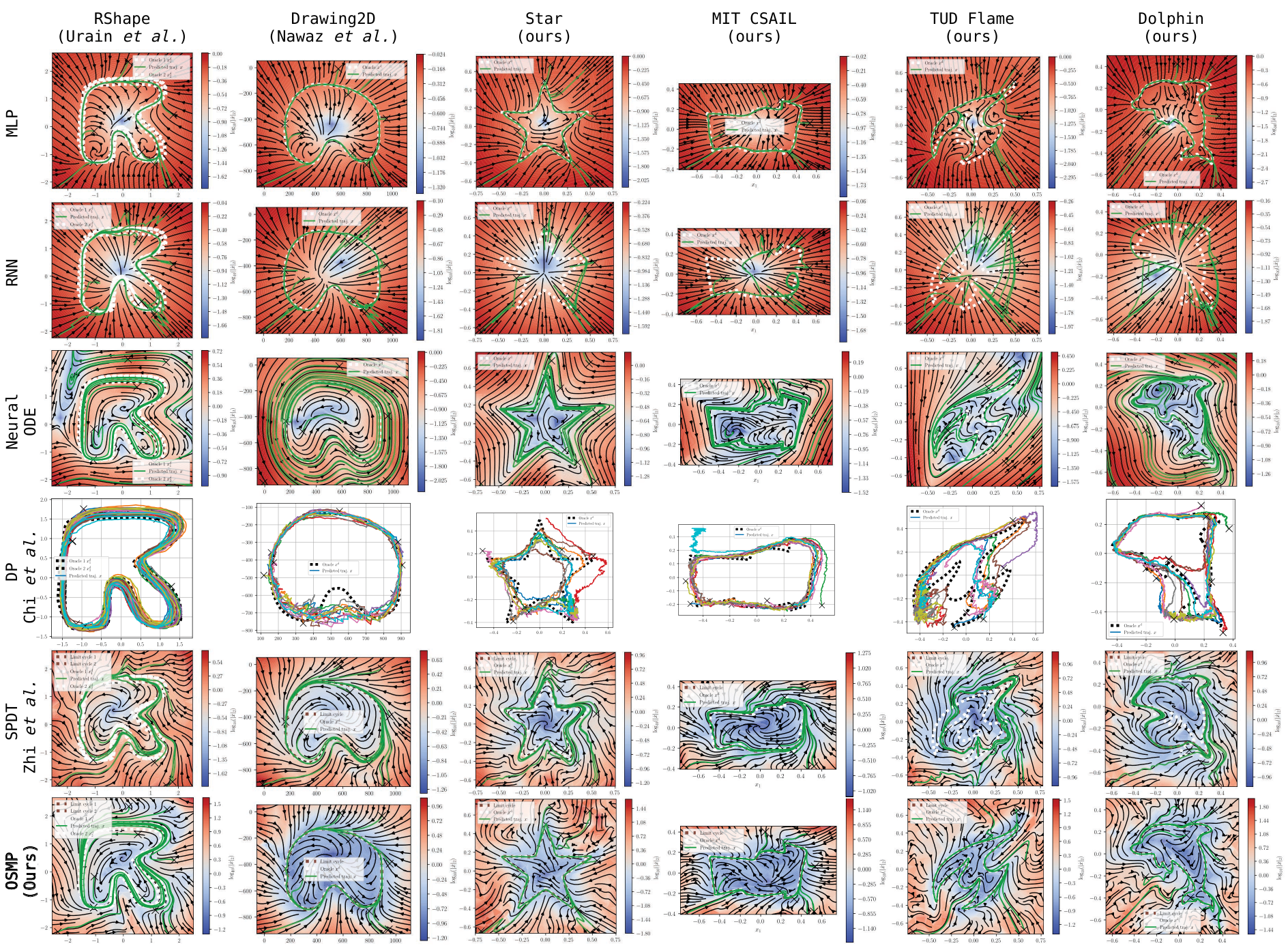}
    \caption{\textbf{Qualitative benchmarking of OSMP against baselines.}
    In this figure, we display the qualitative benchmarking results when comparing the proposed OSMP against baseline methods, such as MLPs, RNNs, NODEs, DPs~\cite{chi2023diffusion}, or SPDT~\cite{zhi2024teaching}. The various columns represent oracles on which the motion policies were separately trained, shown as white dotted lines on the plots. The color map and the streamlines denote the velocity of the learned motion policy when evaluated on a grid. We initialize the trained motion policies at 10 different randomly sampled initial conditions and roll out their trajectory, visualized using solid green lines, for a duration of one period.
    }
    \label{fig:qualitative_benchmarking_results}
\end{figure}

\begin{figure}
    \centering
    \includegraphics[width=1.0\linewidth]{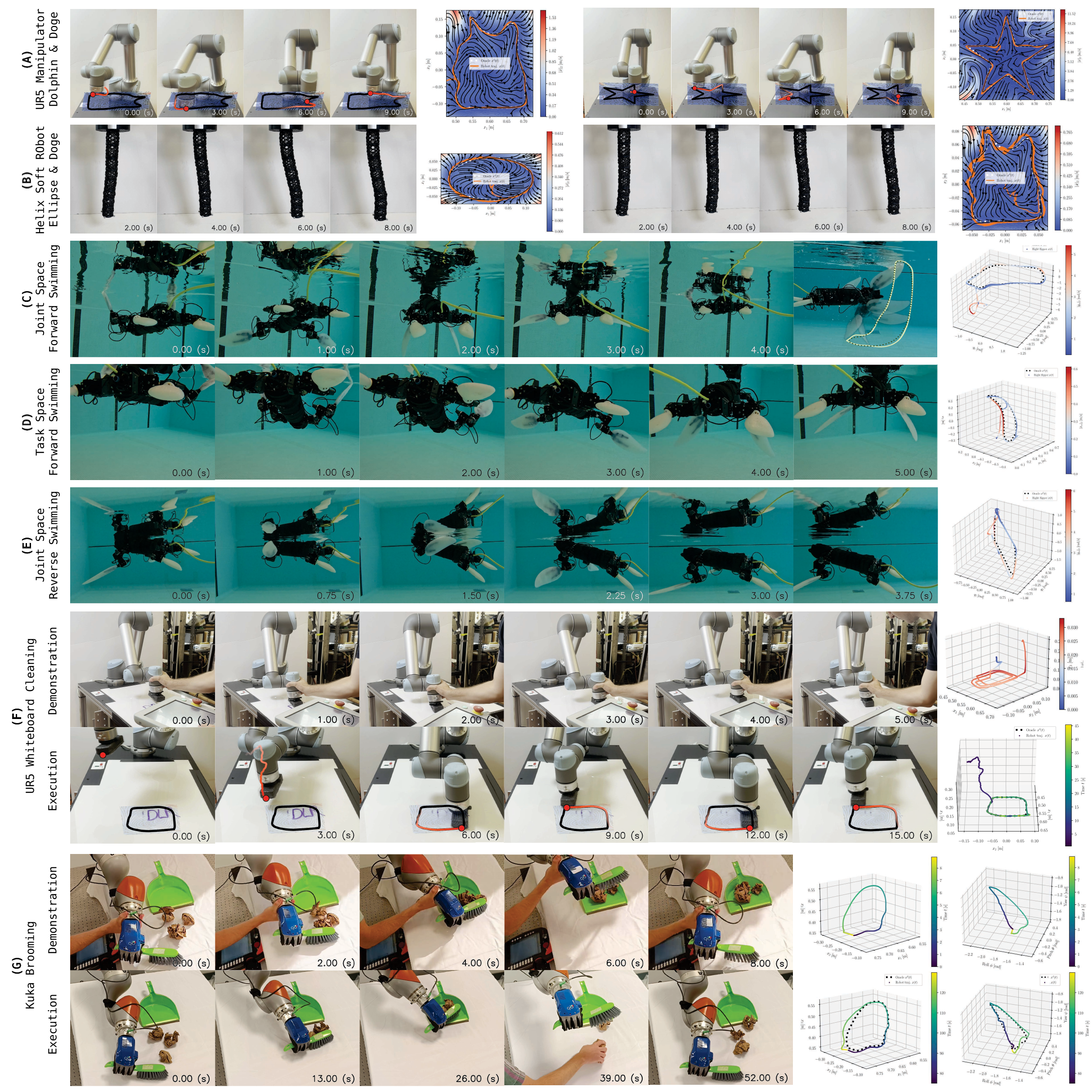}
    \caption{\textbf{Deployment of OSMPs on diverse robot embodiments.}
    We showcase the motion behavior generated by OSMPs deployed on various robot embodiments.
    \textbf{(A \& B)} Tracking of image contours with the UR5 manipulator and the helix soft robot. The black line denotes the oracle shape and the red line the trajectory of the system over a horizon of the past $3$~s.
    \textbf{(C-E)} Forward and reverse turtle robot swimming with the biological oracles defined either in joint or task space.
    \textbf{(F-G)} Demonstrations provided via kinesthetic teaching on whiteboard cleaning (UR5 manipulator) and brooming tasks (Kuka cobot) and subsequent execution of OSMPs trained on these demonstrations.
    }
    \label{fig:robot_embodiments_results}
\end{figure}

\begin{figure}[h!]
    \centering
    \includegraphics[width=0.8\linewidth]{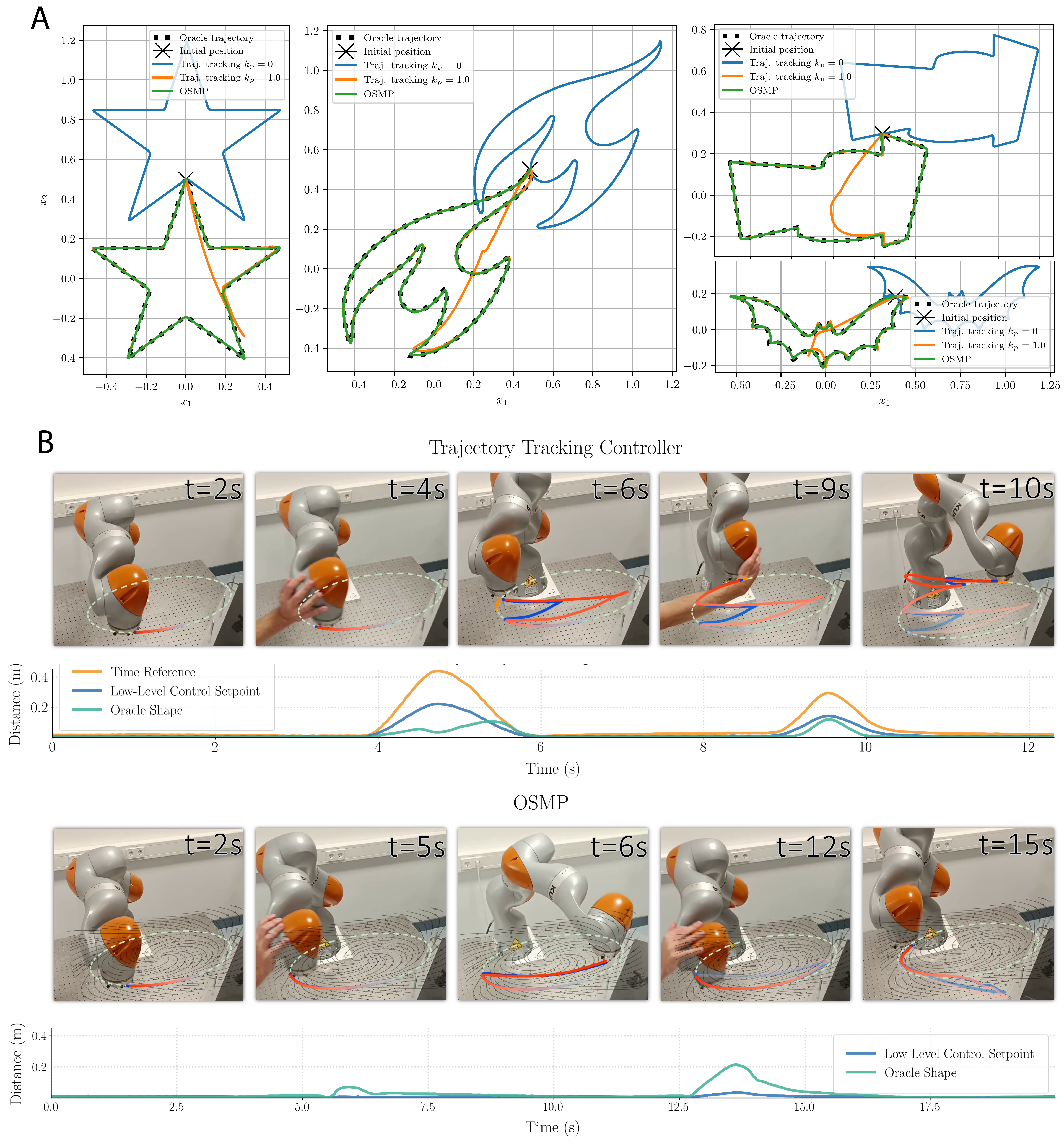}
    \caption{{\textbf{OSMPs exhibit compliant and natural motion behavior.}
    In this figure, we analyze the motion behavior of OSMPs upon perturbations and interaction with humans and the environment. We compare the behavior of OSMPs against classical error-based trajectory tracking controllers $\dot{x}(t) = \dot{x}^\mathrm{d}(t) + k_\mathrm{p} \, (x^\mathrm{d}(t) - x(t))$ that operate on a time-parametrized reference $\dot{x}^\mathrm{d}(t)$, where $k_\mathrm{p} \in \mathbb{R}_{\geq 0}$ is the proportional control gain.
    \textbf{(A)} Simulations with time perturbations where we compare the behavior of traditional, time-parametrized trajectory tracking controllers against the OSMP. Here, the dashed black lines denote the oracles/demonstrations, the solid blue lines the behavior of a pure feedforward trajectory tracking controller with zero feedback gain $k_\mathrm{p} = 0$, the orange line the behavior of an error-based feedback trajectory tracking controller with $k_\mathrm{p} = 1$, and the green line the behavior of the learned OSMP. Compared to nominal scenarios, we perturb the time reference - i.e., the time reference exhibits a $\pi$ offset in phase with respect to the initial position.
    \textbf{(B)} Experiments on a Kuka cobot that runs a compliant low-level impedance controller where a human interacts with the robot and exerts disturbances on the robot, comparing the behavior of the trajectory tracking controller with a feedback gain $k_\mathrm{p} = 1$. We also plot the Cartesian distance with respect to the time reference (trajectory tracking controller only), the next low-level control setpoint $x*(t)$, and the closest point on the oracle over time.
    Experiments on a Kuka cobot equipped with a compliant, low-level impedance controller while a human interacted with the robot and deliberately introduced disturbances/perturbations. The study compares the trajectory-tracking controller’s behavior with a feedback gain of $k_\mathrm{p}=1$ against that of the OSMP. We also plot, over time, the Cartesian distance to the time-indexed reference (trajectory-tracking controller only), the next low-level control setpoint $x^*(t)$, and the distance to the nearest point on the Ellipse oracle.
    In the snapshots, the red line represents the past trajectory, and the blue line the trajectory of past low-level control setpoints.
    }}
    \label{fig:compliance_results}
\end{figure}

\begin{figure}[h!]
    \centering
    \includegraphics[width=1.0\linewidth]{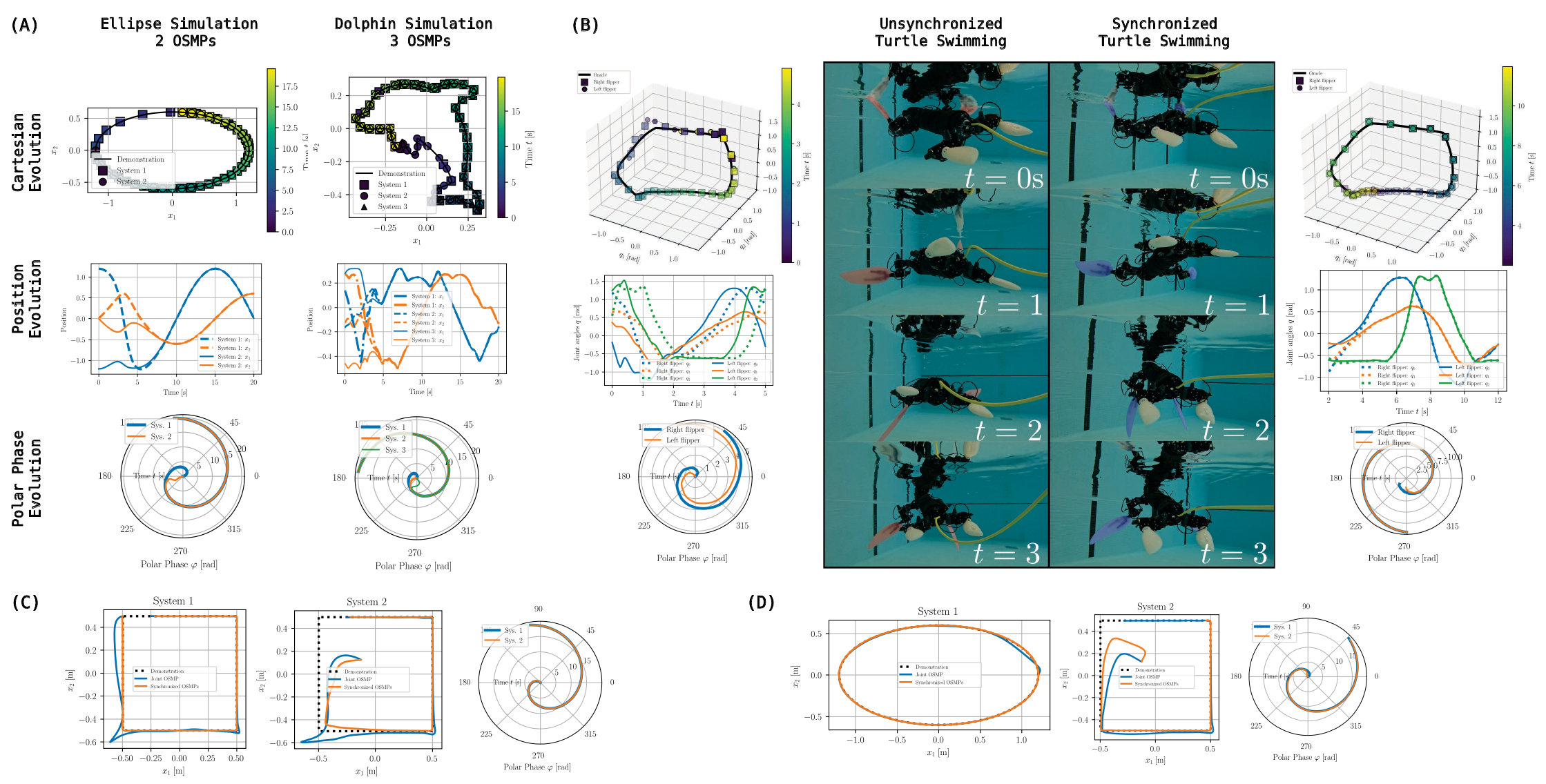}
    \caption{\textbf{The error-based feedback controller successfully synchronizes multiple OSMPs in their phase.}
    Results for the phase synchronization of multiple motion policies.
    \textbf{(A)} Simulation results for the phase synchronization of two and three OSMPs on the Ellipse and Dolphin oracles, respectively. The first row shows the Cartesian-space evolution of each OSMP, where the color of the markers communicates the time information. The second row shows the position vs. time, and the last row shows the polar phase $\varphi$ of the systems over time.
    \textbf{(B)} Experimental results for phase synchronization of the robotic turtle flippers for turtle swimming based on the joint-space oracle. Naturally, the flippers are initialized at the start of the experiment at slightly different positions, causing a phase offset. If this phase offset is not corrected and the flippers remain unsynchronized (left side), the turtle robot doesn't swim or only very slowly. On the right side, the flippers rapidly synchronize, which leads to successful swimming.
    \textbf{(C \& D)} Comparison of a ``\emph{large}'', joined OSMP trained on two systems compared to ``\emph{small}'' OSMPs trained on each system separately and then synchronized during inference. When one of the systems is initialized off the oracle/perturbed, this leads the other system controlled by the joined OSMP to drift off its own limit cycle as well. Contrarily, the synchronized separate OSMPs do not cause each other to drift off the limit cycle.
    }
    \label{fig:phase_sync_results}
\end{figure}

\begin{figure}[h]
    \centering
    \includegraphics[width=1.0\linewidth]{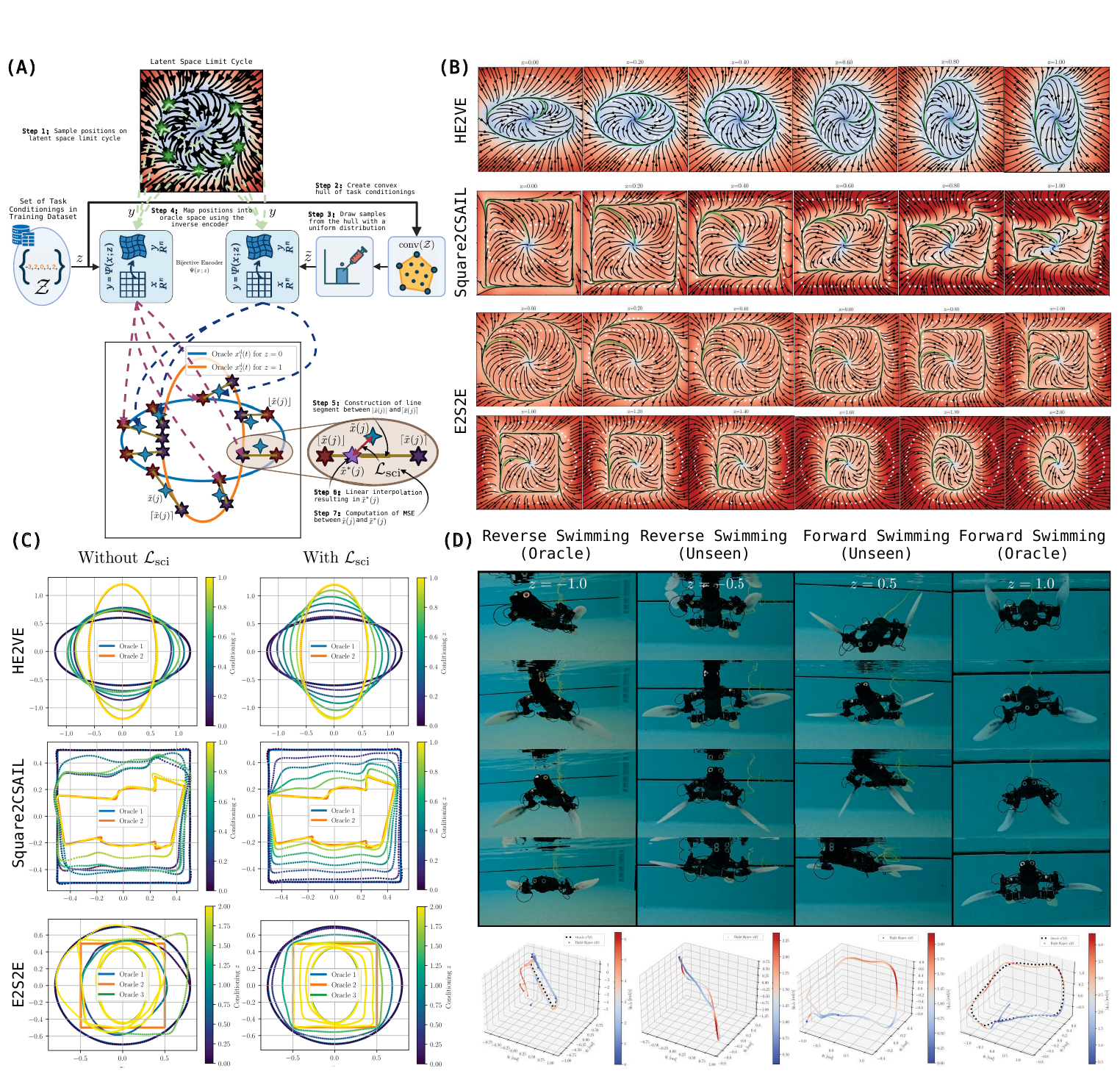}
    \caption{
    \textbf{Smooth interpolation between distinct motion behaviors via encoder conditioning.}
    Results demonstrating the conditioning of the encoder on multiple learned motion behaviors, including smooth interpolation between behaviors.
    \textbf{(A)} Graphic explaining the formulation of the $\mathcal{L}_\mathrm{sci}$ loss.
    \textbf{(B)} OSMPs trained with $\mathcal{L}_\mathrm{sci}$ on multiple image contour transition datasets - horizontal ellipse to vertical ellipse (HE2VE), square to CSAIL logo (Square2CSAIL), and horizontal ellipse to square to - evaluated for various, both seen and unseen during training, task conditioning values. 
    \textbf{(C)} Ablation study for OSMPs trained without and with the smooth conditioning interpolation loss $\mathcal{L}_\mathrm{sci}$, where we smoothly vary the conditioning value over the duration of the trajectory (e.g., $z=0$ at $t=0$ to $z=1$ at $t=140$~s)
    \textbf{(D)} Turtle swimming motion for various conditioning values, ranging from $z=-1$ (reverse swimming joint space oracle) to $z=1$ (forward swimming joint space oracle), with $z=-0.5$ and $z=0.5$ unseen during training.
    }
    \label{fig:conditioning_results}
\end{figure}


	

\begin{table}[htbp]\centering
    \centering
    \caption{\textbf{Orbitally Stable Motion Primitives (OSMPs) provide the best tradeoff between imitation accuracy, global convergence characteristics, and computational time when compared to classical neural motion policies (e.g., MLPs, RNNs, NODEs), current SOTA algorithms such as Diffusion Policies (DP)~\cite{chi2023diffusion}, and prior work on period motion policies with stability guarantees (e.g., iFlow~\cite{urain2020imitationflow}, DT~\cite{zhi2024teaching})}. We report the mean $\pm$ std across three random seeds for each dataset/method evaluation. Entries marked with a $^*$ diverged for some random seeds, and we only report the statistics for the seeds that converged. We label other cases where all random seeds diverged or exhibited extremely high errors using the \emph{$\infty$} symbol.}
    \label{tab:benchmarking_quantitative_results} 
    \setlength{\tabcolsep}{2.5pt}   
    \renewcommand{\arraystretch}{1.2} 
    \begin{tiny}
    \begin{tabular}{|l|l|ccc|cc|cc|c|}   
    \hline
    & & \multicolumn{3}{c|}{\textbf{Imitation Metrics}} & \multicolumn{2}{c|}{\textbf{Local Convergence Metrics}} & \multicolumn{2}{c|}{\textbf{Global Convergence Metrics}} & \textbf{Eval. Time}\\
    \cline{3-9}
    \textbf{Dataset Category} & \textbf{Method} & \textbf{Traj. RMSE~$\downarrow$} & \textbf{Norm. Traj. DTW~$\downarrow$} & \textbf{Vel. RMSE~$\downarrow$} & \textbf{Hausdorff Dist.~$\downarrow$} & \textbf{ICP MED~$\downarrow$} & \textbf{Hausdorff Dist.~$\downarrow$} & \textbf{ICP MED~$\downarrow$} & \textbf{per Step~$\downarrow$}\\
    \hline
    \multirow{8}{*}{IROS Letters~\cite{urain2020imitationflow}} & MLP & $0.257 \pm 0.011$ & $0.0745 \pm 0.0042$ & $0.630 \pm 0.011$ & $0.049 \pm 0.003$ & $0.011 \pm 0.000$ & $0.039 \pm 0.003$ & $0.011 \pm 0.000$ & $0.0017$\\
    & Elman RNN & $0.267 \pm 0.011$ & $0.0935 \pm 0.0286$ & $0.680 \pm 0.004$ & $0.092 \pm 0.011$ & $0.013 \pm 0.002$ & $0.048 \pm 0.014$ & $0.013 \pm 0.003$ & $0.0009$\\
    & LSTM & $0.406 \pm 0.058$ & $0.2719 \pm 0.0473$ & $0.777 \pm 0.055$ & $0.267 \pm 0.104$ & $0.032 \pm 0.008$ & $0.149 \pm 0.050$ & $0.019 \pm 0.008$ & $0.0010$\\
    & NODE & $0.751 \pm 0.573$ & $0.5894 \pm 0.5618$ & $0.856 \pm 0.062$ & $\infty$ & $\infty$ & $\infty$ & $\infty$ & $0.0009$\\
    & DP~\cite{chi2023diffusion} & $\mathbf{0.255 \pm 0.004}$ & $\mathbf{0.0891 \pm 0.0029}$ & $\mathbf{0.648 \pm 0.004}$ & $0.103 \pm 0.001$ & $0.024 \pm 0.001$ & $0.101 \pm 0.001$ & $0.024 \pm 0.001$ & $0.0272$\\
    & iFlow~\cite{urain2020imitationflow} & $0.783 \pm 0.481$ & $0.6832 \pm 0.6326$ & $1.317 \pm 0.895$ & $1.248 \pm 1.201$ & $0.441 \pm 0.449$ & $1.242 \pm 1.204$ & $0.441 \pm 0.449$ & $\mathbf{0.0007}$\\
    & SPDT~\cite{zhi2024teaching} & $0.374 \pm 0.002$ & $0.1406 \pm 0.0017$ & $0.707 \pm 0.004$ & $0.103 \pm 0.002$ & $0.027 \pm 0.001$ & $0.098 \pm 0.003$ & $0.027 \pm 0.001$ & $0.0019$\\
    & \textbf{OSMP (ours)} & $0.344 \pm 0.007$ & $0.1759 \pm 0.0010$ & $0.895 \pm 0.007$ & $\mathbf{0.044 \pm 0.001}$ & $\mathbf{0.010 \pm 0.000}$ & $\mathbf{0.032 \pm 0.001}$ & $\mathbf{0.009 \pm 0.000}$ & $0.0020$\\
    \hline
    \multirow{8}{*}{Drawing2D~\cite{nawaz2024learning}} & MLP & $\mathbf{0.039 \pm 0.004}$ &$\mathbf{0.0060 \pm 0.0002}$ & $\mathbf{0.071 \pm 0.004}$ & $0.046 \pm 0.003$ & $\mathbf{0.005 \pm 0.000}$ & $\mathbf{0.015 \pm 0.001}$ & $\mathbf{0.004 \pm 0.000}$ & $0.0008$\\
    & Elman RNN & $0.109 \pm 0.009$ & $0.0457 \pm 0.0382$ & $0.199 \pm 0.011$ & $0.119 \pm 0.002$ & $0.012 \pm 0.001$ & $0.041 \pm 0.003$ & $0.009 \pm 0.003$ & $0.0007$\\
    & LSTM & $0.384 \pm 0.039$ & $0.2891 \pm 0.0911$ & $0.482 \pm 0.121$ & $0.191 \pm 0.021$ & $0.017 \pm 0.004$ & $0.079 \pm 0.043$ & $0.006 \pm 0.005$ & $0.0008$\\
    & NODE & $0.075 \pm 0.039$ & $0.0279 \pm 0.0300$ & $0.093 \pm 0.028$ & $\infty$ & $\infty$ & $\infty$ & $\infty$ & $0.0008$\\
    & DP~\cite{chi2023diffusion} & $0.178 \pm 0.028$ & $0.0858 \pm 0.0238$ & $0.245 \pm 0.015$ & $0.156 \pm 0.007$ & $0.031 \pm 0.001$ & $0.154 \pm 0.006$ & $0.030 \pm 0.001$ & $0.0494$\\
    & iFlow & $0.337 \pm 0.264$ & $0.2526 \pm 0.3407$ & $0.347 \pm 0.296$ & $0.921 \pm 1.238$ & $0.124 \pm 0.166$ & $0.907 \pm 1.243$ & $0.124 \pm 0.166$ & $\mathbf{0.0007}$\\
    & SPDT~\cite{zhi2024teaching} & $0.462 \pm 0.001$ & $0.2866 \pm 0.0020$ & $0.528 \pm 0.019$ & $0.070^* \pm 0.022^*$ & $0.014^* \pm 0.006^*$ & $0.049^* \pm 0.025^*$ & $0.014^* \pm 0.006^*$ & $0.0017$\\
    & \textbf{OSMP (ours)} & $0.053 \pm 0.002$ & $0.0062 \pm 0.0001$ & $0.081 \pm 0.000$ & $\mathbf{0.040 \pm 0.000}$ & $\mathbf{0.005 \pm 0.000}$ & $\mathbf{0.016 \pm 0.000}$ & $\mathbf{0.004 \pm 0.000}$ & $0.0019$\\
    \hline
    \multirow{8}{*}{Image Cont. (ours)} & MLP & $0.124 \pm 0.023$ & $0.1049 \pm 0.0237$ & $0.106 \pm 0.004$ & $0.051 \pm 0.003$ & $0.006 \pm 0.001$ & $0.030 \pm 0.005$ & $0.004 \pm 0.001$ & $0.0012$\\
    & Elman RNN & $0.318 \pm 0.021$ & $0.2793 \pm 0.0197$ & $0.397 \pm 0.034$ & $0.118 \pm 0.013$ & $0.010 \pm 0.000$ & $0.059 \pm 0.020$ & $0.004 \pm 0.002$ & $0.0012$\\
    & LSTM & $0.494 \pm 0.024$ & $0.5415 \pm 0.0071$ & $0.719 \pm 0.073$ & $0.220 \pm 0.023$ & $0.012 \pm 0.004$ & $0.098 \pm 0.032$ & $0.003 \pm 0.002$ & $0.0385$\\
    & NODE & $\infty$ & $\infty$ & $\infty$ & $\infty$ & $\infty$ & $\infty$ & $\infty$ & $\mathbf{0.0009}$\\
    & DP~\cite{chi2023diffusion} & $0.158 \pm 0.016$ & $0.0928 \pm 0.0078$ & $0.228 \pm 0.008$ & $0.113 \pm 0.005$ & $0.023 \pm 0.002$ & $0.125 \pm 0.005$ & $0.019 \pm 0.001$ & $0.0410$\\
    & iFlow~\cite{urain2020imitationflow} & $0.337 \pm 0.114$ & $0.2342 \pm 0.1799$ & $0.792 \pm 0.280$ & $0.688 \pm 0.608$ & $0.047 \pm 0.045$ & $0.679 \pm 0.613$ & $0.047 \pm 0.045$ & $0.0014$\\
    & SPDT~\cite{zhi2024teaching} & $0.433 \pm 0.000$ & $0.3453 \pm 0.0005$ & $0.454 \pm 0.009$ & $0.114 \pm 0.014$ & $0.024 \pm 0.000$ & $0.094^* \pm 0.002^*$ & $0.023^* \pm 0.001^*$ & $0.0027$\\
    & \textbf{OSMP (ours)} & $\mathbf{0.033 \pm 0.009}$ & $\mathbf{0.0129 \pm 0.0086}$ & $\mathbf{0.050 \pm 0.001}$ & $\mathbf{0.043 \pm 0.002}$ & $\mathbf{0.004 \pm 0.001}$ & $\mathbf{0.016 \pm 0.002}$ & $\mathbf{0.003 \pm 0.000}$ & $0.0028$\\
    \hline
    \multirow{8}{*}{Turtle Swim. (ours)} & MLP & $0.138 \pm 0.001$ & $0.2089 \pm 0.0003$ & $0.294 \pm 0.003$ & $\mathbf{0.084 \pm 0.001}$ & $\mathbf{0.112 \pm 0.041}$ & $\mathbf{0.036 \pm 0.009}$ & $\mathbf{0.089 \pm 0.062}$ & $0.0008$\\
    & Elman RNN & $0.174 \pm 0.021$ & $0.1756 \pm 0.0286$ & $1.165 \pm 0.120$ & $0.177 \pm 0.037$ & $0.126 \pm 0.020$ & $0.109 \pm 0.056$ & $0.120 \pm 0.013$ & $\mathbf{0.0006}$\\
    & LSTM & $0.476 \pm 0.096$ & $0.5712 \pm 0.2284$ & $2.748 \pm 0.518$ & $0.317 \pm 0.095$ & $0.042 \pm 0.020$ & $0.101 \pm 0.070$ & $0.008 \pm 0.009$ & $0.0008$\\
    & NODE & $0.085 \pm 0.022$ & $0.0621 \pm 0.0223$ & $0.236 \pm 0.016$ & $5.133 \pm 1.414$ & $1.085 \pm 0.247$ & $3198 \pm 1800$ & $599 \pm 332$ & $0.0007$\\
    & DP~\cite{chi2023diffusion} & $0.211 \pm 0.014$ & $0.1813 \pm 0.0526$ & $1.511 \pm 0.175$ & $0.183 \pm 0.017$ & $0.124 \pm 0.022$ & $0.190 \pm 0.035$ & $0.129 \pm 0.004$ & $0.0382$\\
    & iFlow~\cite{urain2020imitationflow} & $0.242 \pm 0.144$ & $0.2129 \pm 0.2555$ & $0.739 \pm 0.329$ & $0.590^* \pm 0.323^*$ & $0.183^* \pm 0.123^*$ & $0.391 \pm 0.376$ & $0.187 \pm 0.096$ & $0.0013$\\
    & SPDT~\cite{zhi2024teaching} & $0.321 \pm 0.009$ & $0.2616 \pm 0.0114$ & $0.671 \pm 0.029$ & $0.331 \pm 0.042$ & $0.137 \pm 0.005$ & $0.233 \pm 0.071$ & $0.154 \pm 0.027$ & $0.0016$\\
    & \textbf{OSMP (ours)} & $\mathbf{0.009 \pm 0.000}$ & $\mathbf{0.0056 \pm 0.0008}$ & $\mathbf{0.052 \pm 0.000}$ & $0.135 \pm 0.021$ & $0.114 \pm 0.015$ & $0.045 \pm 0.023$ & $0.112 \pm 0.011$ & $0.0017$\\
    \hline
    \end{tabular}
    \end{tiny}
\end{table}


\clearpage 

%
\bibliography{main} 
\bibliographystyle{sciencemag}

%
%
%
%
%
%


\section*{Acknowledgments}

\paragraph*{Funding:}
The M.S. was supported under the European Union's Horizon Europe Program from Project EMERGE - Grant Agreement No. 101070918 and by the Cultuurfonds Wetenschapsbeurzen. R.P.D was supported by the National Growth Fund program NXTGEN Hightech.
\paragraph*{Author contributions:}
M.S. proposed the research idea. M.S., Z.J.P, T.K.R., C.D.S, and D.R. developed the research idea. 
M.S. developed the framework for training the orbital motion primitives.
Z.J.P. and M.S. performed the turtle robot experiments; M.S. conducted the UR5 robotic manipulator experiments; R.P.D. and M.S. executed the Kuka cobot experiments; M.S. performed the Helix soft robot experiments. 
M.S., Z.J.P, T.K.R., and R.P.D. performed the data analysis.
M.S., Z.J.P, T.K.R., and R.P.D. wrote the manuscript. All authors revised the manuscript. D.R., C.D.S, T.K.R, and Z.J.P supervised the research project. D.R., C.D.S, and J.H. provided funding. All authors gave final approval for publication.

\paragraph*{Competing interests:}
There are no competing interests to declare.
\paragraph*{Data and materials availability:}




Upon publication, we will fully open-source our work. The planned GitHub repositories will host (i) code for training and inference with OSMPs, (ii) all benchmark datasets, (iii) pre-trained models, (iv) evaluation results and plots, (v) MuJoCo turtle-robot simulation code, and (vi) OSMP-based turtle-swimming controllers. Larger supplemental datasets that exceed GitHub limits will be deposited in an open repository such as 4TU.ResearchData.


\subsection*{Supplementary materials}
Materials and Methods\\
Supplementary Text\\
Figs. S1 to S3\\
Tables S1 to S4\\
Movie S1\\


\newpage


\renewcommand{\thefigure}{S\arabic{figure}}
\renewcommand{\thetable}{S\arabic{table}}
\renewcommand{\theequation}{S\arabic{equation}}
\renewcommand{\thepage}{S\arabic{page}}
\setcounter{figure}{0}
\setcounter{table}{0}
\setcounter{equation}{0}
\setcounter{page}{1} 


\begin{center}
\section*{Supplementary Materials for\\ \scititle}

Maximilian~Stölzle,
T.~Konstantin~Rusch$^{\dagger}$,
Zach~J.~Patterson$^{\dagger}$,
Rodrigo~Pérez~Dattari,
Francesco~Stella,
Josie Hughes,
Cosimo~Della~Santina,
Daniela~Rus
\\ 
\small$^\ast$Corresponding author. Email: M.W.Stolzle@tudelft.nl\\
\small$^\dagger$These authors contributed equally to this work.
\end{center}

\subsubsection*{This PDF file includes:}
Materials and Methods\\
Supplementary Text\\
Figures S1 to S2\\
Tables S1 to S4\\
Captions for Movies S1 to S5\\

\subsubsection*{Other Supplementary Materials for this manuscript:}
Movies S1 to S5\\

\newpage


\subsection*{Materials and Methods}




\subsubsection{Diffeomorphic Encoder}

We consider a bijective encoder $\Psi_\theta : \mathbb{R}^n \times \mathbb{R} \to \mathbb{R}^n$, which maps positional states $x \in \mathbb{R}^n$ into the latent states $y \in \mathbb{R}^n$ conditioned on $z \in \mathbb{R}$, where we assume that $n \in \mathbb{N} \geq 2$.
The encoder $y = \Psi_\theta(x;z)$ is parametrized by the learnable weights $\theta \in \mathbb{R}^{n_\theta}$. We omit $\theta$ in most subsequent expressions for simplicity of notation.
If a conditioning is used, the encoder first embeds it with a function $e_\mathrm{z}: \mathbb{R} \to \mathbb{R}^{n_\mathrm{e}}$, which provides $\Bar{z} = e_\mathrm{z}(z)$. Customarily, we choose $n_\mathrm{e} = n$ and construct $e_\mathrm{z}$ using a Gaussian Fourier projection~\cite{chi2023diffusion} into a dimensionality of $4 \, n_\mathrm{e}$ and a two-layer MLP with hidden dimension of $8 \, n_\mathrm{e}$ and a softplus nonlinearity.
%
The encoder is composed by $n_\mathrm{b}$ blocks: $\Psi = \Psi_1 \circ \Psi_2 \dots \Psi_{n_\mathrm{b}}$, where $\Psi_j: \chi_{j-1} \in \mathbb{R}^n \in \mathbb{R} \mapsto \chi_{j} \in \mathbb{R}^n$. Therefore, $\chi_0 = x$ and $\chi_{n_\mathrm{b}} = y$.
In particular, we employ Euclideanizing flows~\cite{dinh2017density, rana2020euclideanizing}, which splits $\chi_j$ into two parts: $\chi_\mathrm{a} \in \mathbb{R}^{n_\mathrm{a}}$ and $\chi_\mathrm{b} \in \mathbb{R}^{n_\mathrm{b}}$ with $n_\mathrm{a} + n_\mathrm{b} = n$.
Then, for odd and even $j$, $\Psi_j$ is given by
\begin{equation}
    \chi_j = \begin{bmatrix}
        \chi_{\mathrm{a}, j}\\
        \chi_{\mathrm{b}, j}
    \end{bmatrix} = \Psi_j(\chi_{j-1}; \Bar{z}) = \begin{bmatrix}
        \chi_{\mathrm{a}, j-1}\\
        \chi_{\mathrm{b}, j-1} \odot \exp(s_{\theta_{\mathrm{s},j}}(\chi_{\mathrm{a}, j-1}; \Bar{z})) + t_{\theta_{\mathrm{t},j}}(\chi_{\mathrm{a}, j-1}; \Bar{z})
    \end{bmatrix},
\end{equation}
and
\begin{equation}
    \chi_j = \begin{bmatrix}
        \chi_{\mathrm{a}, j}\\
        \chi_{\mathrm{b}, j}
    \end{bmatrix} = \Psi_j(\chi_{j-1}; \Bar{z}) = \begin{bmatrix}
        \chi_{\mathrm{a}, j-1} \odot \exp(s_{\theta_{\mathrm{s},j}}(\chi_{\mathrm{b}, j-1}; \Bar{z})) + t_{\theta_{\mathrm{t},j}}(\chi_{\mathrm{b}, j-1}; \Bar{z})\\
        \chi_{\mathrm{b}, j-1}
    \end{bmatrix},
\end{equation}
respectively, which are diffeomorphisms by construction.
For conciseness and without loss of generality, we consider in the following only the case of odd $j$.
In this setting, $s_j(\chi_{\dot, j-1}, \Bar{z}): \mathbb{R}^{n_\mathrm{a}} \to \mathbb{R}^{n_\mathrm{b}}$ and $t_j(\chi_{\dot, j-1}, \Bar{z}): \mathbb{R}^{n_\mathrm{a}} \to \mathbb{R}^{n_\mathrm{b}}$ are two learned functions expressing the scaling and translation.
In our experiments, we leverage Random Fourier Features Networks (RFFNs) consisting of a clamped linear layer with randomly sampled, untrained weights, a cosine activation function and a learned, linear output layer to parametrize $s_{\theta_{\mathrm{s},j}}(\chi_{\mathrm{a}, j-1}; \Bar{z})$ and $t_{\theta_{\mathrm{t},j}}(\chi_{\mathrm{a}, j-1}; \Bar{z})$, where $\theta_{\mathrm{s},j}$ and $\theta_{\mathrm{t},j}$ are the learnable parameters~\cite{rana2020euclideanizing}.

\subsubsection*{OSMP Dynamics}
In this section, we will provide expressions that are helpful in defining and analyzing the dynamics underlying OSMP. Specifically, many of these expressions will be later used for the stability and contraction analysis (Supplementary Text).

\paragraph{\small Latent Polar vs. Cartesian Dynamics}
The map $h_{\mathrm{p2y}}(y_\mathrm{pol}): \mathbb{R}^n \to \mathbb{R}^n$ from latent polar coordinates to latent Cartesian coordinates and its inverse $h_{\mathrm{y2p}}(y) = h_{\mathrm{p2y}}^{-1}(y)$ is given by
\begin{equation}
    y = h_{\mathrm{p2c}}(y_\mathrm{pol}) = \begin{bmatrix}
        r \, \cos(\varphi)\\
        r \, \sin(\varphi)\\
        y_{3:n}
    \end{bmatrix},
    \qquad
    y_\mathrm{pol} = h_{\mathrm{c2p}}(y) = \begin{bmatrix}
        \sqrt{y_1^2+y_2^2}\\
        \operatorname{atan2}(y_2, y_1)\\
        y_{3:n}
    \end{bmatrix}.
\end{equation}
Then, the Jacobian of the Polar-to-Cartesian map $\frac{\partial h_{\mathrm{p2c}}}{\partial y_\mathrm{pol}} \in \mathbb{R}^{n \times n}$ is given by
\begin{equation}
    \frac{\partial h_{\mathrm{p2c}}}{\partial y_\mathrm{pol}} = \begin{bmatrix}
        \cos(\varphi) & -r \, \sin(\varphi) & 0_{1 \times (n-2)}\\
        \sin(\varphi) & r \, \cos(\varphi) & 0_{1 \times (n-2)}\\
        0_{(n-2) \times 1} & 0_{(n-2) \times 1} & \mathbb{I}_{n-2}
    \end{bmatrix}
\end{equation}
We can also substitute $y_\mathrm{pol} = y$ into the Jacobian and determine its inverse
\begin{equation}
    \frac{\partial h_{\mathrm{p2c}}}{\partial y_\mathrm{pol}} = \begin{bmatrix}
        \frac{y_1}{\sqrt{y_1^2 + y_2^2}} & -y_2 & 0_{1 \times (n-2)}\\
        \frac{y_2}{\sqrt{y_1^2 + y_2^2}} & y_1 & 0_{1 \times (n-2)}\\
        0_{(n-2) \times 1} & 0_{(n-2) \times 1} & \mathbb{I}_{n-2}
    \end{bmatrix},
    \qquad
    \frac{\partial h_{\mathrm{p2c}}}{\partial y_\mathrm{pol}}^{-1} = \begin{bmatrix}
        \frac{y_1}{\sqrt{y_1^2 + y_2^2}} & \frac{y_2}{\sqrt{y_1^2 + y_2^2}} & 0_{1 \times (n-2)}\\
        -\frac{y_2}{y_1^2 + y_2^2} & \frac{y_1}{y_1^2 + y_2^2} & 0_{1 \times (n-2)}\\
        0_{(n-2) \times 1} & 0_{(n-2) \times 1} & \mathbb{I}_{n-2}
    \end{bmatrix},
\end{equation}
which are both full-rank for $r = \sqrt{y_1^2 + y_2^2} > 0$.

\subsubsection*{Loss Functions}
\paragraph{\small Velocity Imitation Loss}
Analog to the literature on stable point-to-point motion primitives~\cite{rana2020euclideanizing}, the predicted oracle space velocity is supervised by a smooth $\ell_1$ loss~\cite{girshick2015fast}, which computes a squared term if the absolute error falls below $\beta_{\ell_1}$ and the $\ell_1$ term otherwise, and can be formally defined as
\begin{equation}
    \mathcal{L}_\mathrm{vi} = \frac{1}{N} \, \sum_{k = 1}^{N} \begin{cases}
		\frac{\left ( f(x^\mathrm{d}(k);z(k)) - \dot{x}^\mathrm{d}(k) \right )^2}{2 \, \ell_1}, & \text{if} \: \left \lVert f(x^\mathrm{d}(k);z(k)) - \dot{x}^\mathrm{d}(k) \right \rVert_1 < \beta_{\ell_1} \\
        \left \lVert f(x^\mathrm{d}(k);z(k)) - \dot{x}^\mathrm{d}(k) \right \rVert_1 - \frac{\beta_{\ell_1}}{2} , & \text{otherwise}
    \end{cases},
\end{equation}
where we choose $\beta_{\ell_1} = 1$.

\paragraph{\small Limit Cycle Matching Loss}
Next, we consider a subset of the demonstrations $\mathcal{P} \subset \mathbb{N}_\mathrm{N}$ that exhibit a periodic motion. To guarantee the OS of the learned system, we need to make sure that the learned limit cycle matches the periodic demonstration.
For this purpose, we design a \emph{limit cycle matching} loss $\mathcal{L}_\mathrm{lcm}$ in latent space
\begin{equation}\small
    y_\mathrm{p}(k) = \begin{bmatrix}
        \sqrt{y_1^2(k) + y_2^2(k)}\\ y_{3:n}
    \end{bmatrix} \in \mathbb{R}^{n-1},
    \quad
    y_\mathrm{p}^\mathrm{d}(k) = \begin{bmatrix}
        R\\ 0_{n-2}
    \end{bmatrix},
    \quad
    \mathcal{L}_\mathrm{lcm} = \sum_{k \in \mathcal{P}} \frac{\big \lVert y_\mathrm{p}^\mathrm{d}(k) - y_\mathrm{p}(k) \big \rVert_2^2}{|\mathcal{P}|},
\end{equation}
where $|\mathcal{P}|$ is the cardinality of $\mathcal{P}$, and $y=\Psi(x^\mathrm{d}; z) \in \mathbb{R}^n$ is the latent encoding. 

\paragraph{\small Time Reference Guidance Loss}
For strongly curved oracles, we have found it advantageous to use the oracle’s time parameterization to steer how its positions map onto the latent polar angle. Doing so spreads the oracle samples uniformly around the latent-space limit cycle, preventing the encoders from bunching up in one angular sector while leaving other portions of the cycle without corresponding oracle points.

First, we compute for each position contained in the rhythmic demonstration a desired latent polar angle based on the normalized time reference. Simultaneously, we evaluate the actual encoded polar angle of $x^\mathrm{d}(k)$ as
\begin{equation}
     \varphi^\mathrm{d}(k) = \varphi_0 + 2 \, \pi \, \frac{t(k)}{P},
     \qquad
     \varphi(k) = \operatorname{atan2}\left ( \Psi(x(k);z(k))_2, \{ \Psi(x(k);z(k)) \}_1 \right),
\end{equation}
where $\varphi_0$ is the polar angle anchor and $P$ is the period of the rhythmic demonstration.
Subsequently, we define a positive alignment loss between the 
\begin{equation}
    \mathcal{L}_\mathrm{tgd} = \sum_{k \in \mathcal{P}} \frac{\max \left ( \left | \operatorname{mod} \bigl( \varphi^\mathrm{d}(k) - \varphi(k) + \pi,\; 2\pi\bigr) - \pi \right | - m_\mathrm{tgd}, 0 \right )^2}{|\mathcal{P}|},
\end{equation}
$m_\mathrm{tgd} \in \mathbb{R}_{>0}$ is the allowed margin between the time reference and the actual polar phase and the function $n_{e_\varphi}(e_\varphi) = \operatorname{mod} \bigl( e_\varphi(k) + \pi,\; 2\pi\bigr) - \pi$ normalizes the polar angle error $e_\varphi(k) = \varphi^\mathrm{d}(k) - \varphi(k) $ into the interval $[-\pi, \pi)$.

\paragraph{\small Encoder Regularization}
Similar to Zhi \textit{et al.}~\cite{zhi2024teaching}, we employ an encoder regularization loss $\mathcal{L}_{\Psi, \mathrm{reg}}$ that penalizes deviations from an identity encoder $y = \Psi(x) = x$.
We draw in each epoch $N$ random positions samples $x(k) ~\sim \mathcal{U}(x_\mathrm{min}, x_\mathrm{max}) \: \forall \, k \in \mathbb{N}_N$ from a uniform distribution within the workspace $[x_\mathrm{min}, x_\mathrm{max}]$ of the system. Then, the loss is computed as
\begin{equation}
    \mathcal{L}_{\mathrm{er}} = \sum_{k=1}^{N} \frac{\lVert x - \Psi(x;z) \rVert}{N}.
\end{equation}

\paragraph{\small Velocity Regularization}
The velocity-imitation loss $\mathcal{L}_\mathrm{vi}$ constrains the predicted velocities only along the demonstration trajectory. When demonstrations are sparse and clustered, large regions of the workspace receive no direct supervision on velocity magnitude, even though orbital stability and transverse contraction are still guaranteed. However, in practice, large predicted velocities frequently lead to numerical instability. Therefore, it can be helpful to regularize the predicted velocities.

For this purpose, we draw in each epoch $N$ random positions samples $x(k) ~\sim \mathcal{U}(x_\mathrm{min}, x_\mathrm{max}) \: \forall \, k \in \mathbb{N}_N$ from a uniform distribution within the workspace $[x_\mathrm{min}, x_\mathrm{max}]$ of the system. Then, the loss is computed as
\begin{equation}
    \mathcal{L}_\mathrm{vr} = \sum_{k=1}^N \frac{\max(\lVert f(x(k);z)\rVert_2 - m_\mathrm{vr}, 0_N)}{N},
\end{equation}
where $m \in \mathbb{R}_{\geq0}$ is the margin. In practice, we choose $m_\mathrm{vr}$ to be 50\% higher than the maximum velocity magnitude in the dataset in order not to conflict with the $\mathcal{L}_\mathrm{vi}$ objective.

\paragraph{\small Smooth Conditioning Interpolation Loss}
Next, optionally, we can add a loss term that encourages a smooth interpolation of the learned limit cycle between conditionings $z$. We assume that all conditionings in the dataset $z(k) \in \mathcal{Z}$, where $\mathcal{Z} = \{ z(1), \dots, z(k), \dots, z(N) \}$, are bounded in the interval $[z_\mathrm{min}, z_\mathrm{max}]$.
Now, consider a convex hull $\operatorname{conv}(\mathcal{Z}) = [\min(\mathcal{Z}), \min(\mathcal{Z})] = [z_\mathrm{min}, z_\mathrm{max}]$.
Next, we draw $N_\mathrm{sci}$ random conditionings from a uniform distribution: $\tilde{z}(j) \sim \mathcal{U}(\operatorname{conv}(\mathcal{Z})) \in \mathbb{R}$ with $j \in \mathbb{N}_{N_\mathrm{sci}}$.
Additionally, we also generate $N_\mathrm{sci}$ samples on the latent limit cycle by uniformly sampling polar angles $\varphi(j) \sim [-\pi, \pi)$ and subsequently first map into Cartesian latent coordinates and then into oracle space using the inverse encoder
\begin{equation}
    y(j) = \begin{bmatrix}
        R \, \cos(\varphi(j)) & R \, \sin(\varphi(j)) & 0_{n-2}^\top
    \end{bmatrix}^\mathrm{T},
    \quad
    \tilde{x}(j) = \Psi^{-1}(y(j) \, ; \tilde{z}(j)).
\end{equation}
Now, we define the floor $\lfloor \cdot \rfloor$ and ceil $\lceil \cdot \rceil$ functions that round down, or up to the next conditioning $z \in \mathcal{Z}$ in the dataset
\begin{equation}
    \lfloor \tilde{z} \rfloor = \max\{ z \in \mathbb{Z} \mid z \le \tilde{z} \},
    \qquad
    \lceil \tilde{z} \rceil = \min\{ z \in \mathbb{Z} \mid z \ge \tilde{z} \}.
\end{equation}
Then, the target for $\tilde{x}(j)$ that represents a smooth linear interpolation between conditioning $\lfloor \tilde{z} \rfloor $ and $\lceil \tilde{z} \rceil$ is given by
\begin{equation}
    \tilde{x}^*(j) = \lfloor \tilde{x}(j) \rfloor + \frac{\tilde{z}(j) - \lfloor \tilde{z}(j) \rfloor}{\lceil \tilde{z}(j) \rceil - \lfloor \tilde{z}(j) \rfloor} \left ( \lceil \tilde{x}(j) \rceil - \lfloor \tilde{x}(j) \rfloor \right )
\end{equation}
where
\begin{equation}
    \lfloor \tilde{x}(j) \rfloor = \Psi^{-1}(y(j) \, ; \lfloor \tilde{z}(j) \rfloor),
    \qquad
    \lceil \tilde{x}(j) \rceil = \Psi^{-1}(y(j) \, ; \lceil \tilde{z}(j) \rceil).
\end{equation}
Finally, the smooth conditioning interpolation loss can be formulated as
\begin{equation}
    \mathcal{L}_{\mathrm{sci}} = \sum_{j = 1}^{N_\mathrm{sci}} \frac{\left ( \tilde{x}^*(j) - \tilde{x}(j)\right )^2}{N_\mathrm{sci}}.
\end{equation}

\subsubsection*{Online Shaping of the Learned Motion}
In order to improve the practicality of using the learned orbital motion primitives, we introduce in this section approaches that allow us to modulate the learned velocity field to adjust the task or modify its characteristics without having to retrain the OSMP. 

First, we introduce variables that allow us to spatially translate and scale the learned velocity field
\begin{equation}
    \dot{x}(t) = \tilde{f}(x \, ;z) \coloneq s_\mathrm{f} \, f \left ( \frac{x(t)-x^\mathrm{o}}{s_\mathrm{f}}; z \right ).
\end{equation}
Here, $s_\mathrm{f} \in \mathbb{R}_{>0}$ controls the spatial scale of the velocity field. When $s_\mathrm{f} = 1$, the executed motion primitive is equal to the learned motion primitive. $x^\mathrm{o} \in \mathbb{R}^{n}$ defines the origin of the velocity field.

However, we are not limited to affine transformations such as translation and scaling. Additionally, we can adjust the period and convergence characteristics of the velocity field online. Specifically, by scaling the polar angular velocity $\omega(\varphi)$ with the factor $s_\omega \in \mathbb{R}_{>0}$, we can either slow-down ($0 < s_\omega < 1$) or speed-up ($s_\omega > 1$) the periodic motion.
Furthermore, the convergence of trajectories onto the $\mathbb{S}^1$ limit cycle can be made more or less aggressive by adjusting the convergence gain $k_\mathrm{conv} \in \mathbb{R}_{>0}$. Usually, we set $\alpha = \beta = k_\mathrm{conv} \, s_\omega$.

Finally, constraints in the oracle or actuation space (e.g., joint limits, environment obstacles) might pose challenges to the deployment of the orbital motion primitive in real-world settings when the system is initialized (far) off the oracle.
In these situations, we would not want to start our periodic motion directly, but instead, we would first converge into a neighborhood around the oracle that is collision-free. We devise a strategy that is able to accomplish such behavior by scaling the polar angular velocity $\tilde{\omega}$ as a function $\sigma: \mathbb{R}_{>0} \to \mathbb{R}$ of the distance from the limit cycle $d_\mathrm{lc}$
\begin{equation}\small
    d_\mathrm{lc} = \sqrt{\frac{\left (\sqrt{y_1^2 + y_2^2} - R \right)^2 + \sum_{i=2}^{n} y_i^2}{n-1}},
    \qquad
    \tilde{\omega} = \sigma(\varphi, d_\mathrm{lc}) = \exp \left ( - \frac{\max(d_\mathrm{lc} - R_\mathrm{sm}, 0)^2}{2 \, \sigma_\mathrm{sm}^2} \right ) \, \omega(\varphi),
\end{equation}
where $d_\mathrm{lc} \in \mathbb{R}_{>0}$ the Euclidean distance of the latent state $y$ from the limit cycle normalized by the DOF.
The mapping $d_\mathrm{lc} \mapsto \tilde{\omega}$ can be intuitively interpreted as follows: in a tube of radius $R_\mathrm{sm}$ around the limit cycle, we apply the nominal polar angular velocity $\omega(\varphi)$. Outside of that tube, we reduce the angular velocity using a Gaussian function with RMS width $\sigma_\mathrm{sm} \in \mathbb{R}_{>0}$. In the limit $d_\mathrm{lc} \to \infty$, the polar angular velocity is zero: $\lim_{d_\mathrm{lc} \to \infty} \sigma(d_\mathrm{lc}) = 0$.

\subsubsection*{OSMP Training}
The bijective encoder based on Euclideanizing flow~\cite{rana2020euclideanizing} / Real NVP~\cite{dinh2017density} uses coupling layers with the scaling and translation functions parametrized by Random Fourier Features Networks (RFFNs) that integrate ELU activation functions and a hidden dimension of $100$.
The number of encoder layers/blocks varies by the complexity of the task and ranges from $10$ for simple tasks such as the Ellipse, Square and Doge contours and the IROS~\cite{urain2020imitationflow} and Drawing2D~\cite{nawaz2024learning} to $25$ for very complex tasks such as the TUD-Flame, Bat, and Eagle contours.
At the start of the training, the encoder is initialized as an identity mapping.

The optional velocity scaling network $e^{\mathrm{MLP}(x)}$ is parametrized by a three-to-five-layer MLP (depending on the nonlinearities and discontinuities that the demonstration velocity profile exhibits) with a hidden dimension of $128$ and LeakyReLU activation functions.

Similarly, the angular velocity network $\omega(y) = f_\omega(\Bar{y}_{1:2}) = e^{\mathrm{MLP}(\Bar{y}_{1:2})} + \epsilon_\omega$, where $\Bar{y}_{1:2} = \begin{bmatrix}
    \frac{y_1}{\sqrt{y_1^2+y_2^2}} &  \frac{y_2}{\sqrt{y_1^2+y_2^2}}
\end{bmatrix}^\top \in \mathbb{R}^2$ with $\epsilon_\omega = 10^{-6}$, is parametrized by a five-layer MLP with a hidden dimension of $128$ and LeakyReLU activation functions.

In case we employ the time reference guidance loss $\mathcal{L}_\mathrm{tgd}$ for a given dataset, we usually set $\zeta_\mathrm{lcm} = \zeta_\mathrm{tgd}$.

The OSMP motion policy is trained by a AdamW optimizer~\cite{kingma2014adam, loshchilov2018decoupled} with $(\beta_1, \beta_2) = (0.9, 0.999)$ and a default weight decay of $\lambda = 1 0^{-10}$.
We employ a learning rate scheduler that sequences a linear warmup phase (usually $10$ epochs), with a relatively short constant learning rate period and subsequent long cosine annealing~\cite{loshchilov2016sgdr} period for the remaining epochs. 
We remark that we don't use a minibatch-based training strategy but instead process all given demonstrations in a single batch.
Please note that we also apply the described training procedure to the baseline methods if not explicitly mentioned otherwise.

\subsubsection*{OSMP Inference}
Maintaining discrete-time stability demands that the OSMP—or any DMP—runs at sufficiently high control rates. This requirement becomes even tougher when computational resources are limited, as in our turtle-swimming setup where the OSMPs ran on a Raspberry Pi 5. To minimise latency, we sought to shorten the OSMP’s inference time by exploiting PyTorch’s compilation and export utilities. Unfortunately, most current PyTorch compilers/exporters are incompatible with autograd, which we still need at inference to obtain the encoder Jacobian $J_\Psi = \frac{\partial}{\partial x} \Psi(x;z)$. Consequently, we explored modern options in the \texttt{torch.func} namespace—including combinations of \texttt{vmap} with the forward and reverse functional Jacobian operators (\texttt{jacfwd}, \texttt{jacrev}) and the vector-Jacobian product (\texttt{vjp}). Our analysis, presented in the Supplementary Text, shows that simple two-point finite-difference schemes for the Jacobian compile and export cleanly, run quickly, and yield Jacobians whose accuracy is very close to the analytic solution.
In practice, we use an (absolute) step size $\delta_x = 5 \, 10^{-4}$ such that
\begin{equation}
    J(x;z) \approx \begin{bmatrix}
        \frac{\Psi(x+\delta_x \, e_1;z) - \Psi(x;z)}{\delta_x} & \dots &         \frac{\Psi(x+\delta_x \, e_j;z) - \Psi(x;z)}{\delta_x} & \dots & \frac{\Psi(x+\delta_x \, e_n;z) - \Psi(x;z)}{\delta_x}
    \end{bmatrix},
\end{equation}
where $e_j \in \mathbb{R}^n$ is the $j$th canonical basis vector in $x$-coordinates.
This allows us to exploit \emph{AOTInductor}, a specialized version of \emph{TorchInductor}, to export a compiled executable, which runs at up to $15,000$~Hz on the M4 Max CPU - a 9x increase over the standard eager inference mode.

In case the Jacobian of the encoder $J_\Psi(x;z)$ exhibits close-to-singular values, the numerical stability of the inference, can be improved by computing the inverse as $J_\Psi^{-1}(x;z) \equiv \left ( J_\Psi(x;z) + \epsilon_\mathrm{inv} \, \mathbb{I}_n \right )^{-1}$, where, for example, $\epsilon_\mathrm{inv} = 10^{-6}$.

\subsubsection*{Datasets}

We note that before training, all positions contained in the datasets are normalized to the interval $[-0.5, 0.5]$ with zero mean.

\paragraph{\small IROS Letters}
Originally published by Urain \textit{et al.}~\cite{urain2020imitationflow} and later adopted for benchmarking by Nawaz \textit{et al.}~\cite{nawaz2024learning}, the IROS Letters dataset provides several demonstrations for each character (\emph{IShape}, \emph{RShape}, \emph{OShape}, and \emph{SShape}), sometimes spanning multiple cycles.
A noteworthy characteristic is that demonstrations of the same task are widely separated in state space, posing a challenge for deterministic policies. We smooth every trajectory with a Savitzky–Golay filter (order 3, window 8). Because the \emph{IShape} sequence contains very few sample points and large gaps between consecutive states, we upsample it by a factor of 5. The duration of the demonstrations is chosen as 20~s.

\paragraph{\small Drawing2D}
The Drawing2D dataset introduced by Nawaz \textit{et al.}~\cite{nawaz2024learning} offers four demonstrations of a kidney‐bean-shaped periodic drawing. We train a separate motion policy for each demonstration, applying the same Savitzky–Golay filter (order 3, window 8) for smoothing. Trajectories are normalised to a 20 s duration.

\paragraph{\small Image Contour}
%
We contribute a new benchmark category based on image contours that range from simple shapes (Ellipse, Square) to highly intricate outlines such as the TU Delft flame logo, Bat, and Eagle. Compared to prior benchmarks like IROS Letters~\cite{urain2020imitationflow} and Drawing2D~\cite{nawaz2024learning}, these oracles introduce sharp corners (e.g., Star, Bat), strongly concave curves (e.g., TUD‐Flame), and discontinuous velocity profiles (e.g., Star, Bat, Eagle), all of which are difficult for most DMP-based approaches.
We extract each outline with OpenCV~\cite{opencv_library} and lightly smooth it using a Savitzky–Golay filter (order 3, window 25). Every trajectory lasts 20 s.
The list of image contours is: Ellipse, Square, Star, MIT‐CSAIL, TUD‐Flame, Doge, Bat, Dolphin, and Eagle. 

\paragraph{\small Turtle Swimming}


This category contains four datasets.
\textbf{(i)} The first two comprise Cartesian trajectories of the flipper tip of wild green sea turtles (Chelonia mydas) captured by marine biologists (van der Geest \textit{et al.}~\cite{van2022new}) and represented by Fourier series. We train on two variants: position only (3-D oracle) and position plus twist angle (4-D oracle), each with a period of 4.2~s.
\textbf{(ii)} A subsequent dataset from the same authors applies inverse kinematics to those trajectories, yielding a three-joint-space oracle for bioinspired robotic turtles\cite{van2023soft}. After smoothing with a 30th-degree polynomial, we compute velocities; this oracle has a 4.3~s period.
\textbf{(iii)} Finally, we include a reverse-swimming template defined in joint space by sinusoidal functions, with a 4~s period. This template was designed by constructing waypoints to produce "reverse rowing" movement patterns, interpolating between them with a spline, and finally approximating the trajectory with a fourier fit.

A key distinction from previous periodic-motion datasets is the pivotal influence of the velocity profile on swimming performance: if the velocity profile with which the trajectory is executed is wrong, the cost of transport rises, and in extreme cases, the turtle robot either stalls or even moves in the opposite direction.

\subsubsection*{Baseline Methods}

\paragraph{\small Trajectory Tracking PD Controller}

A classical error-based feedback controller tracking a time-parametrized trajectory $(x^\mathrm{d}(t), \dot{x}^\mathrm{d}(t), \ddot{x}^\mathrm{d}(t)) \: \forall \: t \in [t_0, t_\mathrm{f}]$ is usually given in the form
\begin{equation}\label{eq:osmp:trajectory_tracking_controller}
    \dot{x}(t) = \dot{x}^\mathrm{d}(t) + K_\mathrm{p} \, (x^\mathrm{d}(t) - x(t)),
\end{equation}
where $K_\mathrm{p} \in \mathbb{R}^{n \times n}$ is a proportional feedback gain that operates on the error between the current position $x(t)$ and the desired position $x^\mathrm{d}(t)$. In practice, we choose a scalar $k_\mathrm{p} \in \mathbb{R}_+$ such that $K_\mathrm{p} = k_\mathrm{p} \, \mathbb{I}_{n}$.

\paragraph{\small Multilayer Perceptron (MLP)}
As the most basic neural motion policy, we consider an MLP that predicts the next position/state of the system according to the discrete transition function $x(k+1) = f_\mathrm{MLP}(x)$. We employ a five-layer MLP with a hidden dimension of $128$ and a LeakyReLU activation function.
During training, we randomly sample at the start of each epoch $N_\mathrm{init} = 64$ initial positions from the oracles contained in the dataset. Subsequently, we roll out each trajectory for $T = 25$ steps and enforce an MSE loss between the predicted $x_i(k)$ and demonstrated trajectory $x_i^\mathrm{d}(k)$
\begin{equation}
    \mathcal{L}_\mathrm{ro} = \frac{\sum_{i = 1}^{N_\mathrm{init}}\sum_{k=1}^{T} \left ( x_i(k) - x_i^\mathrm{d}(k) \right )^2}{N_\mathrm{init} \, K}.
\end{equation}

\paragraph{\small Recurrent Neural Networks (RNNs, LSTM)}
For the RNN-like networks (i.e., Elman RNN \& LSTM), we employ a five-layer recurrent neural network with a hidden dimension of $128$.
For example, in the case of the Elman RNN, the transition function of the hidden state $h_j \in \mathbb{R}^{128}$ of the $j$th layer is given by
\begin{equation}
    h_j(k) = \tanh \left (W_\mathrm{xh} \, u_j(k) + W_\mathrm{hh} \, h_j(k-1) + b \right ).
\end{equation}
Here, $u_j(k)$ is the input to the $j$th layer such that $u_1(k) = x(k)$ and $u_j(k) = h_{j-1}(k) \: \forall \, j \in 2, \dots, 5$ and the output of the RNN (i.e., the next state prediction) is given by $x(k+1) = W_\mathrm{o} \, h_5(k)$, where $W_\mathrm{o} \in \mathbb{R}^{n \times 128}$ is a learned matrix.
For training, we use the same rollout procedure and loss $\mathcal{L}_\mathrm{ro}$ as in the case of the discrete MLP motion policy, with the difference that we initialize the RNN's concatenated hidden state $h = \begin{bmatrix}
    h_1^\top & \dots & h_5^\top
\end{bmatrix}^\top$ as $h(1) = 0_{640}$ at the beginning of the trajectory and subsequently propagate through each of the $25$ transitions. 

\paragraph{\small Neural ODE (NODE)}
A NODE-based motion policy can be defined as $\dot{x} = f_\mathrm{NODE}(x)$ where $f_\mathrm{NODE}(x): \mathbb{R}^n \to \mathbb{R}^n$ is parametrized by an MLP. Specifically, we choose a five-layer MLP with hidden dimensions of $128$ and LeakyReLU activation functions.
In addition to supervising the predicted velocity via $\mathcal{L}_\mathrm{vi}$, we compute position-based losses based on rolled-out trajectories analogous to the MLP and RNNs obtained via forward Euler integration of $f_\mathrm{NODE}(x)$.

\paragraph{\small Diffusion Policy (DP)}
We use the official open-source implementation of DPs~\cite{chi2023diffusion} and train them on the respective datasets for $250$ epochs while employing a batch size of $256$, an AdamW~\cite{kingma2014adam, loshchilov2018decoupled} optimizer configured with a learning rate of $10^{-4}$, a weight decay of $10^{-6}$, and $(\beta_1, \beta_2) = (0.95, 0.999)$ and a cosine learning rate scheduler with $10$ warmup steps. For this task, we define the observation as the current and last position of the system $o(k) = \begin{bmatrix}
    x^\top(k-1) & x^\top(k)
\end{bmatrix}^\top \in \mathbb{R}^{2n}$ and the action as the positional state of the system $a = x \in \mathbb{R}^{n}$. For each observation $o(k)$, the DP is trained to predict a sequence of $h=16$ actions $a(k+1), \dots a(k+h)$, of which during inference only eight actions are executed before replanning takes place.
The sampled noise is denoised in a sequence of $100$ steps via a Denoising Diffusion Probabilistic Models (DDPMs) \texttt{squaredcos\_cap\_v2} scheduler~\cite{ho2020denoising}. 
The noise samples are clipped to the range $[-1, 1]$ and the noise prediction network is parametrized by a 1D UNet~\cite{ronneberger2015u} with dimension $(256, 512, 1024)$, kernel size $5$ and global conditioning on the observation $o(k)$. 

\paragraph{\small Imitation Flows (iFlow)}
We leverage the official iFlow~\cite{urain2020imitationflow} implementation for the training on the considered dataset. Specifically, we train the model with a normalizing flows~\cite{rezende2015variational} bijective encoder consisting of $15$ ResNet coupling layers for $1000$ epochs with a batch size of $100$ using an Adamax~\cite{kingma2014adam} optimizer with learning rate $10^{-3}$ and $(\beta_1, \beta_2) = (0.9, 0.999)$.
To define the dynamics in latent space, the method uses a stochastic variant of linear polar limit cycle dynamics transformed into Cartesian coordinates.

\paragraph{\small Stable Periodic Diagrammatic Teaching (SPDT)}
The model architecture of SPDT~\cite{zhi2024teaching} is very similar to OSMPs, apart from the parametric angular velocity. Furthermore, in addition to  advanced features, such as online shaping of the learned motion, phase synchronization, and motion/task conditioning, that OSMPs exhibit, the main difference in training a SPDT lies in the loss function: Instead of employing a velocity imitation loss $\mathcal{L}_\mathrm{vi}$, a limit cycle matching loss $\mathcal{L}_\mathrm{lcm}$, a time reference guidance loss $\mathcal{L}_\mathrm{trg}$, and velocity regularization loss $\mathcal{L}_\mathrm{vr}$, SPDT in mainly relies on on the Hausdorff distance to ensure that the encoded demonstrations match the limit cycle defined in latent space
\begin{equation}
     \mathcal{L}_\mathrm{haus} \;:=\; \max \left \{ \max_{k\in\mathbb{N}_N}\; \min_{\kappa\in\mathbb{N}_N} \; d_y(k, \kappa), \max_{\kappa \in\mathbb{N}_N}\; \min_{k\in\mathbb{N}_N} \; d_y(\kappa, k) \right \},
\end{equation}
where
\begin{equation}
    d_y(k, \kappa) = \lVert \Psi(x(k);z(k)) - y^\mathrm{d}(\kappa) \rVert_2,
\end{equation}
and $y^\mathrm{d}(k) = \begin{bmatrix}
    r \, \cos(\varphi(k)) & r \, \sin(\varphi(k)) & 0_{n-2}^\top
\end{bmatrix}^\top \in \mathbb{R}^n \: \forall \, k \in \mathbb{N}_n$ is a sequence of length $N$ of positions on the latent-space limit cycle obtained via arranging equally-spaced polar angles $\varphi^\mathrm{d}(k) \in [-\pi, \pi)$ and subsequently projecting them back into cartesian space with radius $r = R$.
Additionally, the method also employs the encoder regularization loss $\mathcal{L}_\mathrm{er}$.

\subsubsection*{Imitation Metrics}
We adopt several metrics from the literature that measure how well the motion policy is able to track the given demonstration, both along the time and spatial dimensions.
For all metrics, we initialize the motion policy at the starting position of the demonstration (i.e., $x(1) = x^\mathrm{d}(1)$) and subsequently roll out the motion for $N$ steps (i.e., as many steps as included in the demonstration). Subsequently, we evaluate the mismatch between the actual $x(k)  \in \mathbb{R}^n \: \forall \, k \in \mathbb{N}_N$ and desired trajectory $x^\mathrm{d}(k) \in \mathbb{R}^n \: \forall \, k \in \mathbb{N}_N$, where we define $\mathbb{N}_N = 1,\dots,N$. In case a dataset contains multiple demonstrations, we separately initialize the motion policy at the starting point of each demonstration and report the mean of the metrics across all demonstrations.
Please note that we compute all evaluation metrics on the normalized datasets (i.e., with the positions normalized into the range $[-0.5, 0.5]$), which makes it easier to compare and aggregate metrics across several datasets within the same dataset category.
For all evaluation metrics, we rely on the Python implementation by Jekel \textit{et al.}~\cite{jekel2019similarity} open-sourced in the \texttt{similaritymeasures} package.

\paragraph{\small Trajectory RMSE}
The \emph{Trajectory Root Mean Square Error (Traj. RMSE)}, used, for example, by Khadivar \textit{et al.}~\cite{khadivar2021learning}, measures the deviation of the actual from the nominal trajectory in position space and is computed as
\begin{equation}
    \mathrm{RMSE}_\mathrm{x}\left (\{x(k)\}_{k\in\mathbb{N}_N},\{x^{\mathrm d}(k)\}_{k\in\mathbb{N}_N} \right ) := \sum_{k=1}^{N} \sum_{i=1}^{n} \frac{\left ( x_i^\mathrm{d}(k) - x_i(k) \right )^2}{N \, n}.
\end{equation}

\paragraph{\small Normalized Trajectory DTW}
Dynamic Time Warping (DTW)~\cite{sakoe1978dynamic} searches across all allowable temporal warpings to identify the alignment that minimizes the Euclidean distance between two time series and is used frequently to benchmark imitation learning and behavioral cloning algorithms in robotics~\cite{urain2020imitationflow, perez2023stable, nawaz2024learning} is defined as
\begin{equation}\label{eq:dtw}
    \mathrm{DTW}_\mathrm{x}\left (\{x(k)\}_{k\in\mathbb{N}_N},\{x^{\mathrm d}(k)\}_{k\in\mathbb{N}_N} \right ) := \min_{\pi}\;\sum_{(k, \kappa) \in \pi} \left \lVert x^{d}(k)-x(\kappa) \right \rVert_2.
\end{equation}
where $\pi = ((k_1, \kappa_1), \dots, (k_N, \kappa_N))$ is commonly referred to as the alignment path of length $N$ that contains the sequence of index pairs\footnote{Although DTW is formally not a valid mathematical metric and instead a similarity measure as the triangle inequality doesn't hold, we refer to it here for convenience as one of the evaluation metrics.}. In order for $\pi$ to be a valid alignment path, it needs to fulfill the following constraints
\begin{equation}
\begin{split}
    \pi_1 = (k_1, \kappa_1) = (1,1),
    \qquad
    \pi_N = (k_1, \kappa_1) = (N,N),\\
    k_{\iota+1} - k_{\iota} \in \{0, 1\}, 
    \quad 
    \kappa_{\iota+1} - \kappa_{\iota} \in \{0, 1\},
    \quad
    k_{\iota+1} - k_{\iota} + \kappa_{\iota+1} - \kappa_{\iota} \geq 1,
\end{split}
\end{equation}
where the first row contains constraints that ensure that the beginning and the end of each sequence are connected, and the second row verifies that the indices are monotonically increasing in both $k$ and $\kappa$ and are contained in the sequence at least once.
However, as it can be seen in \eqref{eq:dtw}, the magnitude of the DTW is proportional to the demonstration length $N$, which makes it challenging to aggregate the measure across several datasets of varying demonstration lengths. Therefore, we instead compute a normalized version of the DTW, the \emph{Normalized Trajectory Dynamic Time Warping} measure, as
\begin{equation}
    \mathrm{NDTW}_\mathrm{x}\left (\{x(k)\}_{k\in\mathbb{N}_N},\{x^{\mathrm d}(k)\}_{k\in\mathbb{N}_N} \right ) = \frac{1}{N} \, \mathrm{DTW}_\mathrm{x}\left (\{x(k)\}_{k\in\mathbb{N}_N},\{x^{\mathrm d}(k)\}_{k\in\mathbb{N}_N} \right ).
\end{equation}

\paragraph{\small Velocity RMSE}
The \emph{Velocity Root Mean Square Error (Vel. RMSE)} measures the deviation of the actual from the nominal velocities along the trajectory and is computed as
\begin{equation}
    \mathrm{RMSE}_\mathrm{\dot{x}}\left (\{\dot{x}(k)\}_{k\in\mathbb{N}_N},\{\dot{x}^{\mathrm d}(k)\}_{k\in\mathbb{N}_N} \right ) := \sum_{k=1}^{N}  \sum_{i=1}^{n} \frac{\left ( \dot{x}_i^\mathrm{d}(k) - \dot{x}_i(k) \right )^2}{N \, n},
\end{equation}
where $\dot{x}(k)$ is either directly given by the motion policy (e.g., DMP, OSMP, NODE) or obtained via finite differences as $\dot{x}(k) = \frac{x(k) - x(k-1)}{\delta t}$ (e.g., Diffusion Policy~\cite{chi2023diffusion}).
This metric is particularly relevant for tasks in which it is crucial that the demonstrated velocity, and not just trajectory, is accurately tracked, as in the case of turtle swimming.

\subsubsection*{Convergence Metrics}
Our convergence study pursues two questions: (i) Does the method display orbital stability, converging to a stable limit cycle from any initial state in the workspace? and (ii) Does that limit cycle coincide with the desired oracle? Unlike the imitation metrics, we do not examine timing or velocity profiles; our attention is restricted to the geometric path of the motion. To the best of our knowledge, no standardized protocol yet exists for gauging the orbital-convergence behavior of rhythmic motion policies. Therefore, we devised a protocol based on several rollouts from randomly sampled initial conditions that verifies both the local and global convergence characteristics of the motion policy, which we describe in more detail in the next paragraph.
We evaluate the convergence metrics on all rollouts and subsequently compute the mean value over the resulting trajectories.

\paragraph{\small Local vs. Global Convergence}
We evaluate both local and global convergence characteristics of the motion policies.
In both cases, we sample randomly $25$ positions $x^\mathrm{d}(\kappa) \: \forall \kappa \in 1,\dots,25$ from the set of positions contained in the demonstration $x^\mathrm{d}(k) \forall k \in 1,\dots,N$. Subsequently, we sample an offset $\Delta x_\kappa \sim \mathcal{N}(0,\sigma) \in \mathbb{R}^n$. Now, we consider $x_\kappa(1) = x^\mathrm{d}(\kappa) + \Delta x_\kappa$ as the initial position for the rollout of the motion policy.
For the local convergence analysis, we initialize close to the demonstration by choosing $\sigma = 0.05$. For the global convergence analysis, we select instead $\sigma = 0.15$.

Subsequently, we roll out the motion policy for $N$ and $2N$ steps for the local and global convergence analysis, respectively, resulting in the trajectory sequences $(x_\kappa(1), \dots, x_\kappa(N))$ and $(x_\kappa(1), \dots, x_\kappa(2N))$. For the global convergence, we strive to give the motion policy sufficient time to converge to its limit cycle, and therefore, only evaluate the metrics on the trajectory after the demonstration duration - i.e., we compute the metric on the trajectory sequence $(x_\kappa(N+1), \dots, x_\kappa(2N))$. For simplicity of notation, we assume in the following a reset of time indices such that $k = k-N$.

\paragraph{\small Directed Hausdorff Distance}
The directed Hausdorff distance~\cite{hausdorff1914grundzuge, huttenlocher1993comparing} computes the largest distance between closest-neighbor correspondences from the actual convergence trajectory to the desired limit cycle
\begin{equation}
    h^{\rightarrow}\! \left (\{x(k)\}_{k \in\mathbb{N}_N},\{x^{\mathrm d}(\kappa)\}_{\kappa\in\mathbb{N}_N} \right ) \;:=\; \max_{k\in\mathbb{N}_N}\; \min_{\kappa\in\mathbb{N}_N} \; \left \lVert \,x(k)-x^{\mathrm d}(\kappa) \right \rVert_2.
\end{equation}
The undirected/symmetric version of the Hausdorff distance was used by Zhi \textit{et al.}~\cite{zhi2024teaching} for evaluating the similarity between the desired and actual trajectory shape.

\paragraph{\small Iterative Closest Point MED}
We use an Iterative Closest Point (ICP) algorithm~\cite{besl1992method} to identify the optimal alignment between the two sequences $x^\mathrm{d}(k) \: \forall \, k \in 1,\dots,N$ and $x(\kappa) \: \forall \, \kappa \in 1,\dots,N$ containing the desired limit cycle shape (i.e., the demonstration) and the actual (asymptotic) behavior generated by the motion policy, respectively, by iteratively estimating the transformation, consisting of a translation and a rotation, between the two shapes.
After initializing the translation between the points as $p_0 = 0_n$ and the rotation as $R = \mathbb{I}_n$, each iteration of the ICP algorithm performs two steps
\begin{equation}
\begin{aligned}
\textbf{(correspondence-search step)}\qquad
      c_{m}(k) &:= 
      \arg\min_{\kappa \in N}\,
          \bigl\|\,x(k)\;-\;
          \bigl(R_{m-1}\,x^\mathrm{d}(\kappa)+p_{m-1}\bigr)\bigr\|_2 ,
      \\[6pt]
\textbf{(alignment step)}\qquad
      (R_{m},p_{m}) &:=
      \arg\min_{\substack{R\in SO(n)\\ p \in\mathbb{R}^{n}}}
          \sum_{k = 1}^n
          \bigl\|\,x(k)\;-\;
          \bigl(R\,x^\mathrm{d}\bigl(c_{m}(k)\bigr)+p\bigr)\bigr\|_2^{2},
\end{aligned}
\end{equation}
where the nearest-neighborhood correspondence-search returns a correspondence map $c: N \to N$ that provides for each point $x(k)$ on the actual trajectory the closest point on the desired shape $x^\mathrm{d}(c(k))$. The \emph{alignment-step} identifies the optimal transformation that aligns the two point sequences and is, in practice, implemented using a closed-form SVD step.
After the algorithm has converged to $\hat{c}:=c_{m^{\star}}$, $\hat{R}:=R_{m^{\star}}$, and $\hat{p}:=p_{m^{\star}}$, the Mean Euclidean Distance (MED) between the two aligned shapes is computed as
\begin{equation}
    \mathrm{MED}_\mathrm{conv}\left (\{x(k)\}_{k\in\mathbb{N}_N},\{x^{\mathrm d}(k)\}_{k\in\mathbb{N}_N} \right ) = \sum_{k=1}^{N} \frac{\left \lVert x(k)\;-\;
          \bigl(\hat{R}\,x^\mathrm{d}\bigl(c_{m}(k)\bigr)+\hat{p}\bigr) \right \rVert_2}{N}.
\end{equation}
Please note that ICP-based MED is designed to compare the shape of the desired and measured limit cycles. Because the algorithm also finds a single translation that minimizes the point-to-point distances, two identical shapes that are simply shifted in space will still produce an MED of zero, even though the measured limit cycle is displaced relative to the demonstration.
Therefore, it is crucial to also always evaluate other metrics, such as the directed Hausdorff distance.

\subsubsection*{Evaluation Time per Step}
We define the evaluation time as the computational time/runtime $t_\mathrm{exec}$ required to run inference of the motion policy on the full sequence of a (single) demonstration of length $N$. We then report the evaluation time per step by dividing the time by the sequence length: $\frac{t_\mathrm{exec}}{N}$. Please note that all reported values represent the runtime of an unoptimized inference call of the respective method. For example, we show in the supplementary material that the inference time of OSMP can be optimized by several orders of magnitude by leveraging the PyTorch compile operations (see also the \emph{OSMP Inference} section).

We measured the evaluation time on a desktop workstation with a 12th-Gen i9 12900KF CPU with 16 cores, 64 GB of RAM, and an Nvidia GeForce RTX 3090 GPU with 24 GB of VRAM. For each method, we chose the most suitable processor type: while most methods were executed on the CPU due to the batch size of one and to avoid CPU-GPU communication latency, we ran the diffusion policy on the GPU due to the neural network size involved.

\subsubsection*{Robot Control with Orbital Motion Primitives}
In this work, we illustrate several examples of how the orbital motion policy’s output can be employed to control real-world robots. The range of robots examined in this study covers a diverse spectrum of robotic embodiments—from rigid manipulators (e.g., UR5, Kuka) to soft robotic manipulators (e.g., Helix soft robot~\cite{guan2023trimmed}) and even locomotion systems (e.g., crush turtle robot).

As in the case of OSMPs, the robot action is defined as the system velocity; a large variety of existing approaches can be interchangeably used to generate motor/actuator commands that track this system velocity reference. In many cases, the velocity can directly serve as a reference to a low-level trajectory tracking controller. Alternatively, when considering setpoint regulation controllers, the setpoint/the desired next system position can be computed as
\begin{equation}
    x^*(t) = x(t) + \frac{k_\mathrm{v2p}}{\omega_\mathrm{ctrl}} \, \tilde{f}(x(t) \, ; z),
\end{equation}
where $x(t) \in \mathbb{R}^n$ is the current system, $x^* (t) \in \mathbb{R}^n$ is the system setpoint/goal, $\omega_\mathrm{ctrl} = 200 \, \mathrm{Hz}$ is the control frequency, and $k_\mathrm{v2p} = 1500$ is a gain to map the desired velocity into the next position goal.

When the oracle is defined in join-space, as is the case for some of the turtle swimming oracles, the specified reference can be tracked directly by a motor controller. If the oracle and reference are provided in task-space, one can either (a) project the task-space velocity into configuration space using the pseudo-inverse of the kinematic Jacobian, (b) identify a configuration that matches the desired task-space position setpoint, and subsequently track this configuration space setpoint using a computed torque or PD+ controller, or (c) track the reference using operational-space impedance controllers~\cite{khatib1987unified}.

Below, we specify the low-level control implementation on each system, which explains how the output of the OSMP is mapped into an actuation on the system.

\paragraph{\small UR5 Robotic Manipulator}
We deploy the periodic motion primitives in task space on the UR5 robotic manipulator
\begin{equation}
    x^*(t) = x(t) + \frac{k_\mathrm{v2p}}{\omega_\mathrm{ctrl}} \, \tilde{f}(x(t) \, ; z),
    \qquad
    f_\mathrm{v}(t) = k_\mathrm{p} \, (x^*(t) - x(t)) - k_\mathrm{d} \, \dot{x}(t),
\end{equation}
where $x(t) \in \mathbb{R}^3$ is the current task-space position of the end-effector as computed by the forward kinematic model, $x^* (t) \in \mathbb{R}^3$ is the task-space setpoint/goal, $\omega_\mathrm{ctrl} = 200 \, \mathrm{Hz}$ is the control frequency, and $k_\mathrm{v2p} = 1500$ is a gain to map the desired task-space velocity into the next task-space position goal.
$x^*$ is tracked by a virtual Cartesian impedance controller~\cite{scherzinger2017forward} to generate a virtual Cartesian force $f_\mathrm{v}(t)$ with a proportional gain of $0.05 \, \mathrm{N/m}$ and a damping gain of $0.005 \, \mathrm{Ns/m}$.

\paragraph{\small KUKA Cobot}
We also verified the OSMP on a KUKA LBR iiwa 14 collaborative robot with $n_\mathrm{q} = 7$ DOFs.
In contrast to the UR5 manipulator, in addition to the end-effector position $p \in \mathbb{R}^3$, we also consider here the orientation of the end-effector represented by a rotation matrix $C \in SO(3)$. Here, both $p(t)$ and $C(t)$ are computed by the forward kinematics $\mathrm{FK}(q): R^\mathrm{n_\mathrm{q}} \to \mathbb{R}^3 \times SO(3)$. 
However, the orientation does not live in Euclidean space, preventing us from directly applying the OSMP.
Instead, we formulate, inspired by Urain \textit{et al.}~\cite{urain2022learning}, the motion primitive in the tangent space of $SO(3)$. Specifically, we apply the LogMap to the rotation matrix $C(t)$ resulting in the motion primitive state $x(t) = \begin{bmatrix}
    p^\top(t) & \mathrm{Log}(C(t))^\top
\end{bmatrix}^\top \in \mathbb{R}^6$ such that the OSMP is defined in dimensionality $n=6$.
In order to train this motion primitive via the velocity prediction loss $\mathcal{L}_\mathrm{vi}$, we apply finite differences in tangent space
to gather the velocity reference $\dot{x}^\mathrm{d}(k)$ of the oracle.
Then, analogous to the UR5, the internal goal is computed by the OSMP as
\begin{equation}
    x^*(t) = x(t) + \frac{k_\mathrm{v2p}}{\omega_\mathrm{ctrl}} \, \tilde{f}(x(t); z),
\end{equation}
where $\omega_\mathrm{ctrl} = 150 \, \mathrm{Hz}$ is the control frequency, and $k_\mathrm{v2p} \in \mathbb{R}_+$ is a gain to map the desired velocities into the next (internal) position goal.
Subsequently, we define the internal position goal as $p_i^*(t) = x_i^*(t)$ for all $i \in \{1, 2, 3\}$, and assume that the trajectory in $SO(3)$ remains close to the identity—an assumption consistent with our dataset. This allows us to define the internal orientation goal as $ C^*(t) = \mathrm{Exp}(x_{4:6}^*(t))$, by applying the $SO(3)$ exponential map to the orientation component without introducing significant distortion. We note that for trajectories with larger rotational variability, it would be necessary to account for the nonlinear geometry of $SO(3)$, for instance by using group composition with the logarithmic and exponential maps~\cite{sola2018micro}.


Lastly, the reference $T^*(t) \in SE(3)$ consisting of the end-effector position $p^*(t) \in \mathbb{R}^3$ and orientation $C^*(t) \in SO(3)$ is tracked by combining an inverse kinematics solver designed for fast and smooth tracking~\cite{wang2023rangedik} with a joint-space impedance controller
\begin{equation}
     \tau = G(q) + K_\mathrm{p} \left( \mathrm{IK}(p^*, C^*) - q \right) - K_\mathrm{d} \, \dot{q},
\end{equation}
where $G(q) \in \mathbb{R}^{n_\mathrm{q}}$ captures the joint space gravitational forces, $\mathrm{IK}: \mathbb{R}^3 \times SO(3) \to \mathbb{R}^{n_\mathrm{q}}$ represents a solver that computes a reference joint state $q^{*}(t)$ from $x^{*}(t)$, i.e., $q^{*} = \mathrm{IK}(p^*, C^*)$, and $K_\mathrm{p}, K_\mathrm{d} \in \mathbb{R}^{n_\mathrm{q} \times n_\mathrm{q}}$ are proportional and derivative control gains, respectively, defining the impedance behavior of the robot.

\paragraph{\small Helix Continuum Soft Robot}
The Helix Robot~\cite{guan2023trimmed} is a continuum soft manipulator consisting of three independently actuated segments that can bend in the $x$,$y$ plane and adjust their length. Each segment is actuated via three independent tendons, controlled in length via a set of motors installed at the base. Each section is modeled as a constant curvature arc with variable length~\cite{guan2023trimmed, stella2023piecewise}, with this configuration space $q \in \mathbb{R}^{9}$ serving as an intermediary mapping between the tendon space and Cartesian space.

Call $x = h_\mathrm{x}(q)$ the standard forward kinematics mapping the shape coordinates $q$ into the end effector locations $x$, defined in Della Santina \textit{et al.}~\cite{della2020improved}.
For each independent section (i.e., the $i$th segment), the mapping between tendon lengths $l_{3(i-1) + j}$ and shape coordinates $q_i$ can be stated as
\begin{equation}
    q_i =
    \begin{bmatrix}
        \Delta_{\mathrm{x},i} \\
        \Delta_{\mathrm{y},i} \\
        \Delta L_i
    \end{bmatrix}
    =
    \begin{bmatrix}
        \frac{2}{3} & -\frac{1}{3} & -\frac{1}{3} \\
        0 & \frac{1}{\sqrt{3}} & \frac{-1}{\sqrt{3}} \\
        \frac{1}{3} & \frac{1}{3} & \frac{1}{3}
    \end{bmatrix}
    \begin{bmatrix}
        l_{3(i-1) + 1} \\
        l_{3(i-1) + 2} \\
        l_{3(i-1) + 3}
    \end{bmatrix}
    -
    \begin{bmatrix}
        0 \\
        0 \\
        l_0
    \end{bmatrix}
\end{equation}
\noindent
where $l_\mathrm{0} \in \mathbb{R}_{>0}$ is the length of the undeformed structure and $R \in \mathbb{R}_{>0}$ is the radius between the center of the robot and the tendon guide. By repeating the same mapping for each of the three sections, we obtain $q=h_\mathrm{q}(l)$, which defines the full robot shape as a function of the tendon lengths.
Combining the two maps yields $x = h_\mathrm{x}(h_\mathrm{q}(l))$. This composed end-to-end forward kinematics has the Jacobian $J(l) = J_\mathrm{x} J_\mathrm{q}$, with $J_\mathrm{x}$ and $J_\mathrm{q}$  being the Jacobian of $f_\mathrm{x}$ and $h_\mathrm{q}$, respectively. Thus, $J$ is such that
$\dot{x}(t) = J(l) \, \dot{l}(l)$ .
Having a well-defined differential map, we can then use the following kinematic inversion
scheme to find the $\overline{l}$ such that $||\overline{x}- f_\mathrm{x}(h_\mathrm{q}(\overline{l}))||^2$ is minimum, for any given $\overline{x}: \dot{\overline{l}}= \alpha J^+ \left ( \overline{x}-f_\mathrm{x}( h_\mathrm{q} \overline{t} ) \right )$ where $\alpha$ is a scalar gain and $^+$ indicates pseudoinversion. The pseudoinverse is weighted against a model-based estimation of the robot's shape stiffness, i.e., $J^+= (J^\top K J)^{-1}J^\top K$, where $K \in \mathbb{R}^{9 \times 9}$ is the diagonal stiffness matrix of the soft robot, so to avoid unbalanced deformations of the body and approaching the target end-effector location with the minimal elastic energy stored. 

We position the Helix robot within a motion capture cage equipped with six Optitrack Flex 13 cameras and track the 3D pose of its end-effector, denoted as $x(k) \in \mathbb{R}^3$. Additionally, we implement a task-space control law where the internal target for the end-effector position, $x^*(k) \in \mathbb{R}^3$, is updated iteratively.
\begin{equation}
    x^*(k) = x^*(t_{k-1}) + \frac{k_\mathrm{v2p}}{\omega_\mathrm{ctrl}} \, \tilde{f}(x(k); z),
\end{equation}
where $\omega_\mathrm{ctrl} = 50 \, \mathrm{Hz}$ is the control frequency and $k_\mathrm{v2p} = 0.45$ is a gain to map the desired task-space velocity into the next task-space position goal. 
Subsequently, the previously introduced statics-aware inverse kinematic algorithm is employed to map $x^*(k)$ to tendon-lengths $l^* \in \mathbb{R}^{9}$, which are then tracked by a Dynamixel position controller.

\paragraph{\small Crush Turtle Robot}
The Crush turtle robot is a bioinspired hybrid soft-rigid system developed by MIT’s Distributed Robotics Lab (DRL) that can swim and emulates the swimming motion of green sea turtles (Chelonia mydas)~\cite{van2022new, van2023soft}. It comprises a main body housing a Raspberry Pi 5 for computation and a battery for power, two flipper arms with three joints—each actuated by a Dynamixel XW540-T260-R motor—and a soft-rigid flipper design. Additionally, the robot features two rear flippers, each powered in a serial chain by two Dynamixel XW540-T260-R motors. The motors are controlled at the lowest level by Dynamixel's current control loop. Utilizing the approximately linear relationship between torque and current, we perform simple PD control in the joint space with torque as the input. 
We consider both joint and task space control settings. Here, the joint space is defined as the flipper arm motor angles and the task space corresponds to the position and twist angle of the turtle flipper tip.

\textbf{Joint Space Control.} In the first scenario, we directly train the motion primitive in joint space based on a bioinspired oracle published by van der Geest \textit{et al.}~\cite{van2023soft}. 
We perform direct velocity control on the motors by commanding a joint velocity
\begin{equation}
    \dot{q}^*(t) = \frac{1}{\omega_\mathrm{ctrl}} \, \tilde{f}(\tilde{q}(t) \, ;z) \in \mathbb{R}^3,
\end{equation}
that is tracked by the low-level motor controller. Here, $\tilde{q}(t) = (q(t) + \pi) \bmod 2 \pi$ are the flipper arm joint angles normalized into the interval $[-\pi, \pi]$.

\textbf{Task Space Control.} In the second scenario, we execute control of the pose of the tip of the front flipper that we define as $x(t) = \begin{bmatrix}
    x_1 & x_2 & x_3 & \theta
\end{bmatrix}^\mathrm{T} \in \mathbb{R}^4$, where $x_{1:3}$ are the positional coordinates of the flipper end-effector in Cartesian space and $\theta$ is the twist angle (i.e., the position of the last joint).
We train the motion primitive on an oracle based on the measurements of the swimming of green sea turtles (Chelonia mydas), where the flipper pose was estimated based on video recordings~\cite{van2022new}. Therefore, the motion primitive is formulated as
\begin{equation}
    \dot{x}^*(t) = \frac{1}{\omega_\mathrm{ctrl}} \, \tilde{f}(\tilde{x}(t) \, ; z),
    \qquad
    \dot{q}^* = \left ( \mathrm{diag}(1,1,1,w_\theta) \, J_{q \rightarrow x}(q) \right )^{-1} \, \dot{x}^*(t)
\end{equation}
where $J_{q \rightarrow x} \in \mathbb{R}^{4 \times 3}$ is the Jacobian relating joint- to task-space velocity $\dot{x} = J_{q \rightarrow x} \, \dot{q}$ and $w_\theta \in \mathbb{R}$ is a weight specifying how accurately commanded twist angular velocity $\theta^*$ should be tracked, as the system is overconstrained by one DOF. The corresponding Jacobian is computed as
\begin{equation}
    J_{q \rightarrow x}(q) = \begin{bmatrix}
        \frac{\partial x_1}{\partial q_1} & \frac{\partial x_1}{\partial q_2} & \frac{\partial x_1}{\partial q_3}\\
        \frac{\partial x_2}{\partial q_1} & \frac{\partial x_2}{\partial q_2} & \frac{\partial x_2}{\partial q_3}\\
        \frac{\partial x_3}{\partial q_1} & \frac{\partial x_3}{\partial q_2} & \frac{\partial x_3}{\partial q_3}\\
        0 & 0 & 1\\
    \end{bmatrix}.
\end{equation}
Analog to the joint space control approach, the desired joint-space velocity $\dot{q}^* \in \mathbb{R}^3$ is tracked by the low-level motor controller.
Note: Even though taking into account the twist angle renders the projection from task into joint space to be overconstrained, we found that it helps to avoid the kinematic singularities of the mechanism.




\clearpage


\subsection*{Supplementary Text}



\subsubsection*{Proof of Asymptotic Orbital Lyapunov Stability}

\begin{Theorem}\label{theorem:asymptotic_orbital_stability}
    Let $\alpha, \beta, R > 0$ and $f_\mathrm{s}(x): \mathbb{R}^n \to \mathbb{R}_{>0}$ and $z$ be constant. Then, the dynamics $\dot{x} = f(x;z)$ are asymptotically orbitally stable under the transverse Lyapunov function
    \begin{equation}
        V_\mathrm{x}(x;z) = \frac{\alpha \, R^2}{4} \left ( y_1^2+y_2^2 - R^2 \right )^2 + \frac{1}{2} y_{3:n}^\top \, \beta \, y_{3:n} \Bigg |_{y=\Psi(x;z)}.
    \end{equation}
\end{Theorem}
\begin{proof}
    \textbf{Step 1: Proof of Orbital Stability in Polar Latent Dynamics.}
    Consider the transverse Lyapunov candidate~\cite{manchester2011transverse} for the polar latent dynamics $\dot{y}_\mathrm{pol} = f_\mathrm{pol}(y_\mathrm{pol})$
    \begin{equation}
        V_\mathrm{pol}(y_\mathrm{pol}) = \alpha \, \frac{R^2}{4} \left ( r^2 - R^2 \right )^2 + \frac{1}{2} y_{3:n}^\top \, \beta \, y_{3:n}.
    \end{equation}
    We can demonstrate that $V_\mathrm{pol}$ is a valid Lyapunov candidate with respect to the limit cycle $\mathcal{O}_\mathrm{pol} = \{ r, \varphi \in \mathbb{R}, y_{3:n} \in \mathbb{R}^{n-2} | r = R, \varphi \in [-\pi, \pi), y_{3:n} = 0_{n-2} \}$ via
    \begin{equation}
        V_\mathrm{pol}(y_{\mathrm{pol}}) = \alpha \, \frac{R^2}{4} \left ( R^2 - R^2 \right )^2 = 0,
        \quad
        \forall \, y_{\mathrm{pol}} \in \mathcal{O},
    \end{equation}
    and
    \begin{equation}
        V_\mathrm{pol}(y_{\mathrm{pol}}) = \frac{\alpha \, R^2}{4} \underbrace{\left ( r^2 - R^2 \right )^2}_{\geq 0} + \underbrace{\frac{1}{2} y_{3:n}^\top \, \beta \, y_{3:n}}_{\geq 0} > 0,
        \quad
        \forall y_{\mathrm{pol}} \in \mathbb{R}^{n} \setminus \mathcal{O}_\mathrm{pol}.
    \end{equation}
    The time derivative of the Lyapunov candidate perpendicular to the orbital flow $\begin{bmatrix}
        0 & f_\omega(\varphi) & 0
    \end{bmatrix}^\top$ is given by
    \begin{equation}
    \begin{split}
         \dot{V}_{\mathrm{pol},f_{\mathrm{pol},\perp}}(y_\mathrm{pol}) =& \: \frac{\partial V_\mathrm{y}}{\partial y_\mathrm{pol}} \, f_{\mathrm{pol},\perp} = \begin{bmatrix}
            -\alpha \left ( 1-\frac{r^2}{R^2} \right ) r & 0 & \beta \, y_{3:n}
        \end{bmatrix} \, \begin{bmatrix}
            \alpha \left ( 1-\frac{r^2}{R^2} \right ) r\\
            0\\
            -\beta \, y_{3:n}
        \end{bmatrix},\\
        =& \: -\underbrace{\alpha^2 \left (1 - \frac{r^2}{R^2} \right )^2 r^2}_{\geq 0} - \underbrace{y_{3:n}^\top \, \beta^2 \, y_{3:n}}_{\geq 0} < 0,
        \quad
        \forall y_\mathrm{pol} \in \mathbb{R}^{n} \setminus \mathcal{O}_\mathrm{pol}.
    \end{split}
    \end{equation}
    with $\dot{V}_{\mathrm{pol},f_{\mathrm{pol},\perp}}(y_\mathrm{pol}) = 0 \: \forall y_\mathrm{pol} \in \mathcal{O}_\mathrm{pol}$.\\
    \textbf{Step 2: Proof of Orbital Stability in Cartesian Latent Dynamics.}
    First, we define the limit cycle in latent space as $\mathcal{O}_\mathrm{y} = \{ y \in \mathbb{R}^n | \sqrt{y_1^2 + y_2^2} = R, y_{3:n} = 0_{n-2} \}$
    The Lyapunov function in Cartesian latent coordinates $y = h_\mathrm{p2c}(y_\mathrm{pol})$ can be stated as 
    \begin{equation}
        V_\mathrm{y}(y) = V_\mathrm{pol}(h_\mathrm{y2p}(y)) = \frac{\alpha \, R^2}{4} \left ( y_1^2+y_2^2 - R^2 \right )^2 + \frac{1}{2} y_{3:n}^\top \, \beta \, y_{3:n}.
    \end{equation}
    Naturally, the two conditions on the Lyapunov candidate $V_\mathrm{y}(y) = 0 \: \forall y \in \mathcal{O}_\mathrm{y}$ and $V_\mathrm{y}(y) > 0 \: \forall y \in \mathbb{R}^n \setminus \mathcal{O}_\mathrm{y}$ still hold. 
    As the flow orthogonal to the orbit is given by
    \begin{equation}
        f_{\mathrm{y},\perp}(y) = \frac{\partial h_\mathrm{p2c}}{\partial y_\mathrm{pol}} \, f_{\mathrm{pol},\perp}(h_\mathrm{c2p}(y)) = \begin{bmatrix}
        \alpha \, \left ( 1 - \frac{y_1^2 + y_2^2}{R^2} \right ) \, y_1\\
        \alpha \, \left ( 1 - \frac{y_1^2 + y_2^2}{R^2} \right ) \, y_2\\
        -\beta \, y_{3:n}\\
    \end{bmatrix},
    \end{equation}
    the corresponding time derivative of the Lyapunov function can be computed as
    \begin{equation}
    \begin{split}
        \dot{V}_{\mathrm{y},f_{\mathrm{y},\perp}}(y) =& \: \frac{\partial V_\mathrm{y}(y)}{\partial y} \, f_{\mathrm{y},\perp}(y) = \frac{\partial V_\mathrm{pol}(h_{c2p}(y))}{\partial y} \, \frac{\partial h_\mathrm{p2c}(y_\mathrm{pol})}{\partial y_\mathrm{pol}} \, f_{\mathrm{pol},\perp}(y_\mathrm{pol}) \Bigg |_{y_\mathrm{pol} = h_{c2p}(y)},\\
        =& \: \frac{\partial V_\mathrm{pol}(y_\mathrm{pol})}{\partial y_\mathrm{pol}} \, \underbrace{\frac{\partial h_\mathrm{c2p}}{\partial y} \, \frac{\partial h_\mathrm{p2c}}{\partial y_\mathrm{pol}}}_{\mathbb{I}_n} \, f_{\mathrm{pol},\perp}(y_\mathrm{pol}) \Bigg |_{y_\mathrm{pol} = h_{c2p}(y)},\\
        =& \: \frac{\partial V_\mathrm{pol}(y_\mathrm{pol})}{\partial y_\mathrm{pol}} \, f_{\mathrm{pol},\perp}(y_\mathrm{pol}) \Bigg |_{y_\mathrm{pol} = h_{c2p}(y)} < 0, \quad \forall \, y \in \mathbb{R}^n \setminus \mathcal{O}_\mathrm{y}.
    \end{split}
    \end{equation}
    \textbf{Step 3: Proof of Asymptotic Orbital Stability of OSMP Dynamics.}
    Similar to prior work~\cite{rana2020euclideanizing, urain2020imitationflow, zhi2024teaching}, we transfer the orbital stability guarantees back into oracle space.
    The orbit/limit cycle in oracle space is given by
    $\mathcal{O}_\mathrm{x} = \{ x \in \mathbb{R}^n | y=\Psi(x;z), \sqrt{y_1^2 + y_2^2} = R, y_{3:n} = 0_{n-2} \}$.
    Then, define the Lyapunov function by substituting $y = \Psi(x;z)$ into $V_\mathrm{y}(y)$
    \begin{equation}
         V_\mathrm{x}(x;z) = V_\mathrm{y}(\Psi(x;z)) = \frac{\alpha \, R^2}{4} \left ( y_1^2+y_2^2 - R^2 \right )^2 + \frac{1}{2} y_{3:n}^\top \, \beta \, y_{3:n} \Bigg |_{y=\Psi(x;z)},
    \end{equation}
    which still admits to the properties $V_\mathrm{x}(x) = 0 \: \forall x \in \mathcal{O}_\mathrm{x}$ and $V_\mathrm{x}(x) > 0 \: \forall x \in \mathbb{R}^n \setminus \mathcal{O}_\mathrm{x}$.
    The flow orthogonal to the limit cycle can be expressed as 
    \begin{equation}
        f_{\mathrm{x},\perp}(x) = f_\mathrm{s}(x) \, J_\Psi^{-1}(x;z) \, f_{\mathrm{y},\perp}(\Psi(x;z)).
    \end{equation}
    Similar to Step 2, but now also considering the oracle-space velocity scaling $f_\mathrm{s}(x)$, we derive the time derivative of the Lyapunov function orthogonal to the limit cycle flow as
    \begin{equation}
    \begin{split}
        \dot{V}_{\mathrm{x},f_{\mathrm{x},\perp}}(x) =& \: \frac{\partial V_\mathrm{x}(x)}{\partial x} \, f_{\mathrm{x},\perp}(x) = \frac{\partial V_\mathrm{y}(\Psi(x;z))}{\partial x} \, f_\mathrm{s}(x) \, J_\Psi^{-1}(x;z) \, f_{\mathrm{y},\perp}(\Psi(x;z)) \Bigg |_{y = \Psi(x;z)},\\
        =& \: f_\mathrm{s}(x) \, \frac{\partial V_\mathrm{y}(y)}{\partial y} \, \underbrace{J_\Psi(x;z) \, J_\Psi^{-1}(x;z)}_{\mathbb{I}_n} \, f_{\mathrm{y},\perp}(\Psi(x;z)) \Bigg |_{y = \Psi(x;z)},\\
        =& \: \underbrace{f_\mathrm{s}(x)}_{>0} \, \underbrace{\frac{\partial V_\mathrm{y}(y)}{\partial y} \, f_{\mathrm{y},\perp}(y_\mathrm{pol})}_{\leq 0} \bigg |_{y = \Psi(x;z)} < 0, \quad \forall \, x \in \mathbb{R}^n \setminus \mathcal{O}_\mathrm{x}.
    \end{split}
    \end{equation}
\end{proof}

\subsubsection*{Proof of Transverse Contraction}

In this section, we include supplementary material that is required for the proof that OSMPs are transverse contracting~\cite{manchester2014transverse} - an extension of the contraction criteria proposed by Lohmiller \textit{et al.}~\cite{lohmiller1998contraction, lohmiller1999phdthesis} for periodic/rhythmic systems.

\paragraph{\small (Transverse) Contraction Theory}
\begin{Definition}\label{def:transverse_contraction}
    An autonomous system $\dot{x} = f(x)$ with is said to be transverse contracting in the region $x \in \mathcal{X} \subseteq \mathbb{R}^n$ with rate $\zeta \in \mathbb{R}_{>0}$ if a positive definite metric $M(x) \succ 0 \in \mathbb{R}^{n \times n}$ exists such that
    \begin{equation}
        \delta_{x}^\top \left ( \frac{\partial f(x)}{\partial x}^\top \, M(x) + M(x) \, \frac{\partial f(x)}{\partial x} + \dot{M}(x) + 2 \, \zeta \, M(x) \right ) \delta_{x} \leq 0
        \quad
        \forall \, x \in \mathcal{X}
    \end{equation}
    for all $\delta_{x} \neq 0$ orthogonal to the flow satisfying $\delta_{x}^\top \, M(x) \, f(x) = 0$~\cite{manchester2014transverse}.
\end{Definition}

\begin{Definition}\label{def:exponential_orbital_stability}
    Let $\dot{x} = f(x)$ be an autonomous, transverse contracting system in the region $x \in \mathcal{X} \subseteq \mathbb{R}^n$ with contracting rate $\zeta \in \mathbb{R}_{>0}$. 
    Also, consider the non-trivial $T$-periodic solution $x_\mathrm{lc}(t) \in \mathcal{X}$ that defines the solution curve $X_\mathrm{lc} = \{ x \in \mathbb{R}^n : \exists t \in [0,T):x=x_\mathrm{lc} \}$.
    Then, the solution $x_\mathrm{lc}(t)$ is said to be \emph{exponentially orbitally stable} as there exists a $k > 0$ such that for any $x_0 \in \mathcal{X}$
    \begin{equation}
        \inf_{x_\mathrm{lc} \in X_\mathrm{lc}} \lVert x(t) - x_\mathrm{lc} \rVert_2 \leq k \inf_{x_\mathrm{lc} \in X_\mathrm{lc}} \lVert x_0 - x_\mathrm{lc} \rVert_2 \, e^{-\zeta t}.
    \end{equation}
\end{Definition}
Please note that the notion of \emph{exponential orbital stability}, also referred to as transverse exponential stability, is stronger than the commonly used \emph{asymptotic orbital stability} as it guarantees \textbf{exponential} convergence to the orbit/limit cycle~\cite{manchester2011transverse}.

\paragraph{\small Proof of Transverse Contraction of Latent Dynamics in Polar Coordinates}
First, we prove that the latent dynamics in polar coordinates $\dot{y}_\mathrm{pol} = f_\mathrm{pol}(y_\mathrm{pol})$ are transversely contracting: i.e., they are not contracting along the polar phase variable $\varphi$, but contracting orthogonal to the flow $f_\mathrm{pol}(y_\mathrm{pol})$~\cite{manchester2014transverse}. Inspired by a recent proof of transverse contraction for the Andronov-Hopf oscillator with state $(r, \varphi)$~\cite{nah2025combining}, we define the following Proposition.
\begin{Proposition}\label{prop:polar_latent_dynamics_transverse_contraction}
    Let $\alpha, \beta > 0$, $R > 0$, $f_\omega(r): [-\pi, \pi) \to \mathbb{R}_{>0}$, and $y_\mathrm{pol} = \begin{bmatrix}
    r & \varphi & y_{3:n}^\top
\end{bmatrix}^\top \in \mathcal{Y}_\mathrm{pol}$. Then, the latent dynamics in polar coordinates from \eqref{eq:latent_dynamics_polar_coordinates} are transverse contracting under the metric
    \begin{equation}\label{eq:contraction_metric_polar_latent_dynamics}
        M_\mathrm{pol}(y_\mathrm{pol}) = \begin{bmatrix}
            \frac{1}{r^2} & -\frac{\alpha \left( 1 - \frac{r^2}{R^2} \right )}{f_\omega(\varphi) \, r} & 0_{1 \times (n-2)}\\
            -\frac{\alpha \left( 1 - \frac{r^2}{R^2} \right )}{f_\omega(\varphi) \, r} & m_{\varphi \varphi}(r) & 0_{1 \times (n-2)}\\
            0_{(n-2) \times 1} & 0_{(n-2) \times 1} & \mathbb{I}_{n-2}
        \end{bmatrix} \succ 0 \in \mathbb{R}^{n \times n},
    \end{equation}
    with the contraction rate $\zeta_\mathrm{pol} \geq \left (\frac{2\alpha}{R^2} + \beta \right ) \, \frac{r_\epsilon^2}{r_\epsilon^2 + 1}$ in the region $\mathcal{Y}_\mathrm{pol} = \left \{ y_\mathrm{pol} \in \mathbb{R}^n | r_\epsilon \in \mathbb{R}_{>0}, r \geq r_\epsilon, \varphi \in [-\pi, \pi) \right \}$.
\end{Proposition}
\begin{proof} The proof consists of three steps: proof of positive-definiteness of the contraction metric, and the fulfillment of the orthogonality and contraction conditions in order to meet the conditions for transverse contraction stated in Theorem~3 of \cite{manchester2014transverse}.\\
    \textbf{Step 1: Positive definite contraction metric.} Positive definite contraction metric $M_\mathrm{pol}(y_\mathrm{pol})$. In order for $M_\mathrm{pol}(y_\mathrm{pol}) \in \mathbb{R}^{n \times n}$ to be a valid contraction metric, we need to ensure that it is positive definite (i.e., that the real part of its Eigenvalues is positive $\forall y_\mathrm{pol} \in \mathcal{Y}_\mathrm{pol}$). 
    For this to be the case, the following condition, derived from the smallest Eigenvalue of the contraction metric $\zeta_\mathrm{m}(M_\mathrm{pol}(y_\mathrm{pol}))$, must hold
    \begin{equation}\small
        \frac{m_{\varphi \varphi}(r) \, r^2 + 1}{2 \, r^2} - \frac{\sqrt{\left ( R^4 f_\omega^2(\varphi) r^4 \right ) m_{\varphi \varphi}^2 + \left ( -2 R^4 f_\omega^2(\varphi) r^2 \right ) m_{\varphi \varphi}(r) + \left ( 4 \alpha^2 r^6 - 8 R^2 \alpha^2 r^4 + 4 R^2 \alpha^2 r^2 + R^4 f_\omega^2(\varphi) \right )}}{2 \, R^2 \, f_\omega(\varphi) \, r^2} \geq 0
    \end{equation}
    which can be ensured if the following two conditions hold
    \begin{equation}\small
    \begin{split}
        m_{\varphi \varphi} \geq 0,
        \quad
        \left ( R^4 \right ) m_{\varphi \varphi}^2(r) + \left ( -2 R^4 f_\omega^2(\varphi) r^2 \right ) \, m_{\varphi \varphi}(r) + \left ( 4 \alpha^2 r^6 - 8 R^2 \alpha^2 r^4 + 4 R^2 \alpha^2 r^2 + R^4 f_\omega^2(\varphi) \right ) \leq 0.\\
    \end{split}
    \end{equation}
    Solving the quadratic equation results in
    \begin{equation}
    \begin{split}
        0 \leq m_{\varphi \varphi}(r) \leq m_{\varphi \varphi}^\mathrm{ub}(r) = \frac{R^2 f_\omega(\varphi) + 2 \, \alpha \, \sqrt{- r^4 + 2R^2 r^2 -R^4 } \, |r|}{R^2 \, f_\omega(\varphi) \, r^2}.
    \end{split}
    \end{equation}
    Thus, for example, the choice of $m_{\varphi \varphi}(r) = m_{\varphi \varphi}^\mathrm{ub}(r)$ admits to the stated condition.
    Then, the Eigenvalues of the $M_\mathrm{pol}(y_\mathrm{pol})$ are given by
    \begin{equation}
    \begin{split}
        \zeta_{1,2}(M_\mathrm{pol}(r)) =  \: \frac{R^2 f_\omega(\varphi) + \alpha \sqrt{- r^4 + 2R^2 r^2 -R^4 } \, |r|}{R^2 \, f_\omega(\varphi) \, r^2},
        \qquad
        \zeta_{3:n}M_\mathrm{pol}(r)) = \: 1.
    \end{split}
    \end{equation}
    with $\mathrm{Re}(\zeta_{1,2}(M_\mathrm{pol}(r))) > 0 \: \forall r \in R_{>0}$. Therefore, $M_\mathrm{pol}(y_\mathrm{pol}) \succ 0 \: \forall r \in R_{>0}$.
    \\
    \textbf{Step 2: Orthogonality condition.}
    Let $\delta_{y_\mathrm{pol}} = c \, \begin{bmatrix}
        1 & 0 & 1_{n-2}
    \end{bmatrix}^\top$ with $c \in \mathbb{R}_+$ be the incremental motion orthogonal to the flow. Then, the contraction metric \eqref{eq:contraction_metric_polar_latent_dynamics} fulfills the orthogonality condition for the stated choice of $\delta y_\mathrm{pol}$~\cite{manchester2014transverse}
    \begin{equation}\label{eq:orthogonality_condition_polar_latent_dynamics}
    \begin{split}
        \delta_{y_\mathrm{pol}}^\top \, M_\mathrm{pol}(y_\mathrm{pol}) \, f_\mathrm{pol}(y_\mathrm{pol}) = & \: \begin{bmatrix}
            c & 0 & c \, 1_{1 \times (n-2)}
        \end{bmatrix} \, \begin{bmatrix}
            0\\
            m_{\varphi \varphi}(r) \, f_\omega(\varphi) - \alpha^2 \, \frac{(R^2 - r^2)^2}{R^4 \, f_\omega(\varphi)}\\
            -\beta \, y_{3:n}
        \end{bmatrix} = -c \, \beta \, y_{3:n}.\\
    \end{split}
    \end{equation}
    Since $-c \, \beta \, y_{3:n}$ converges uniformly to zero, the orthogonality condition defined in \eqref{eq:orthogonality_condition_polar_latent_dynamics} converges to zero and is, therefore, fulfilled\footnote{See proof of Theorem~5 in \cite{manchester2014transverse}.}.\\
    \textbf{Step 3: Transverse contraction condition.} The transverse contraction condition is given by~\cite{manchester2014transverse}
    \begin{equation}\label{eq:contraction_condition_polar_latent_dynamics}
    \begin{split}
        \delta_{y_\mathrm{pol}}^\top \left ( \frac{\partial f_\mathrm{pol}}{\partial y_\mathrm{pol}}^\top M_\mathrm{pol}(y_\mathrm{pol}) + M_\mathrm{pol}(y_\mathrm{pol}) \, \frac{\partial f_\mathrm{pol}}{\partial y_\mathrm{pol}} + \dot{M}_\mathrm{pol}(y_\mathrm{pol}) + 2 \, \zeta_\mathrm{pol} \, M_\mathrm{pol}(y_\mathrm{pol}) \right ) \delta_{y_\mathrm{pol}} &\leq 0\\
        \frac{2 \, c^2}{r^2} \left (\zeta_\mathrm{pol} \left ( r^2 + 1 \right ) - \left ( \frac{2 \, \alpha}{R^2} + \beta \right ) r^2 \right ) & \leq 0
    \end{split}
    \end{equation}
    where
    \begin{equation}
    \begin{split}
        \frac{\partial f_\mathrm{pol}}{\partial y_\mathrm{pol}} &= \begin{bmatrix}
            \alpha - 3 \alpha \frac{r^2}{R^2} & 0 & 0_{1 \times (n-2)}\\
            0 & \frac{f_\omega(\varphi)}{\partial \varphi} & 0_{1 \times (n-2)}\\
            0_{(n-2) \times 1} & 0_{(n-2) \times 1} & -\beta \, \mathbb{I}_{n-2}
        \end{bmatrix},\\
        \dot{M}_\mathrm{pol} &= \begin{bmatrix}
            -\frac{2 \alpha}{r^2} + \frac{2 \alpha}{R^2} & \frac{\alpha \left ( R^2-r^2 \right ) \left ( \alpha(R^2+r^2) + R^2 \frac{\partial f_\omega}{\partial \varphi} \right )}{R^4 \, r \,  f_\omega(\varphi)} & 0_{1 \times (n-2)}\\
           \frac{\alpha \left ( R^2-r^2 \right ) \left ( \alpha(R^2+r^2) + R^2 \frac{\partial f_\omega}{\partial \varphi} \right )}{R^4 \, r \,  f_\omega(\varphi)} & \alpha \frac{R^2-r^2}{R^2} r \frac{\partial m_{\varphi \varphi}}{\partial r} & 0_{1 \times (n-2)}\\
            0_{(n-2) \times 1} & 0_{(n-2) \times 1} & 0_{(n-2) \times (n-2)}
        \end{bmatrix}.
     \end{split}
    \end{equation}
    We can simplify \eqref{eq:contraction_condition_polar_latent_dynamics} to
    \begin{equation}
    \begin{split}
        2 \, c^2 \left ( \zeta_\mathrm{pol} \left ( 1 + \frac{1}{r^2} \right ) - \left ( \frac{2\alpha}{R^2}  + \beta\right ) \right ) \leq 0,\\
        \zeta_\mathrm{pol} \, \left ( r^2 + 1 \right ) - \left ( \frac{2 \, \alpha}{R^2} + \beta \right ) r^2 \leq 0,\\
        \zeta_\mathrm{pol}  \leq  \left ( \frac{2 \, \alpha}{R^2} + \beta \right ) \, \frac{r^2}{r^2 + 1}.
    \end{split}
    \end{equation}
    Given $y_\mathrm{pol} \in \mathcal{Y}_\mathrm{pol}$ (i.e., $r \geq r_\epsilon$), we can, therefore, guarantee that the actual contraction rate admits to the lower bound
    \begin{equation}
        \zeta_\mathrm{pol} \geq \left ( \frac{2 \, \alpha}{R^2} + \beta \right ) \, \frac{r_\epsilon^2}{r_\epsilon^2 + 1}.
    \end{equation}
\end{proof}
In practical robotics settings, the contraction rate is particularly relevant for the region $r \geq R^2$ (i.e., the system is outside the defined limit cycle contour). In such a setting with $r_\epsilon = R$, the contraction rate is given $\forall r \geq R$ by $\zeta_\mathrm{pol} \geq \frac{2 \alpha + \beta R^2}{R^2 + 1}$ and for $R = 1$ by $\zeta_\mathrm{pol} \geq \alpha + \frac{\beta}{2}$. This illustrates well how the latent dynamics, $\alpha, \beta$, allow us to modulate the contraction behavior of the system.

\paragraph{\small Proof of Transverse Contraction of Latent Dynamics in Cartesian Coordinates}
\begin{Proposition}\label{prop:cartesian_latent_dynamics_transverse_contraction}
    Let $\omega \geq 0$, $\alpha, \beta > 0$, $R > 0$, and $y = \begin{bmatrix}
    y_1 & y_2 & y_{3:n}^\top
\end{bmatrix}^\top \in \mathcal{Y}$. Then, the latent dynamics in Cartesian coordinates from \eqref{eq:latent_dynamics} are transverse contracting under the metric
    \begin{equation}\label{eq:contraction_metric_cartesian_latent_dynamics}
        M_y(y) = \frac{\partial h_{\mathrm{p2c}}}{\partial y_\mathrm{pol}}^{-\top} \, M_\mathrm{pol} \, \frac{\partial h_{\mathrm{p2c}}}{\partial y_\mathrm{pol}}^{-1} \Bigg |_{y_\mathrm{pol} = h_\mathrm{p2c}^{-1}(y)} \succ 0 \in \mathbb{R}^{n \times n},
    \end{equation}
    with the contraction rate $\zeta_y \geq \left (\frac{2\alpha}{R^2} + \beta \right ) \, \frac{r_\epsilon^2}{r_\epsilon^2 + 1}$ in the region $\mathcal{Y} = \left \{ y \in \mathbb{R}^n | r_\epsilon \in \mathbb{R}_{>0}, \sqrt{y_1^2 + y_2^2} \geq r_\epsilon \right \}$.
\end{Proposition}
\begin{proof}
    Again, the proof consists of three steps: proof of positive-definiteness of the contraction metric, and the fulfillment of the orthogonality and contraction conditions in order to meet the conditions for transverse contraction.\\
    \textbf{Step 1: Positive definite contraction metric.} As shown in Proposition \ref{prop:polar_latent_dynamics_transverse_contraction}, $M_\mathrm{pol}(y) \succ 0 \: \forall \, y \in \mathbb{R}^n$. As $M_\mathrm{pol}$ is square and $\mathrm{rank} \left ( \frac{\partial h_{\mathrm{p2c}}}{\partial y_\mathrm{pol}}^{-1} \right ) = n \: \forall y \in \mathcal{Y}$, $M_y(y)$, as defined in \eqref{eq:contraction_metric_cartesian_latent_dynamics}, is positive definite~\cite{petersen2008matrix}.\\
    \textbf{Step 2: Orthogonality condition.} Let the incremental motion orthogonal to the flow be defined as
    \begin{equation}
        \delta_y = \frac{\partial h_{\mathrm{p2c}}}{\partial y_\mathrm{pol}} \, \delta_{y_\mathrm{pol}} \bigg |_{y_\mathrm{pol} = h_\mathrm{p2c}^{-1}(y)} = c \, \begin{bmatrix}
        \frac{y_1}{\sqrt{y_1^2 + y_2^2}} & \frac{y_2}{\sqrt{y_1^2 + y_2^2}} & 1
    \end{bmatrix}^\top.
    \end{equation}
    The orthogonality condition is then given by
    \begin{equation}
        \delta_y^\top \, M_y(y) \, f_y(y) = -c \, \beta \, y_{3:n}.
    \end{equation}
    Since $-c \, \beta \, y_{3:n}$ converges uniformly to zero, the orthogonality condition defined in \eqref{eq:orthogonality_condition_polar_latent_dynamics} converges to zero.\\
    \textbf{Step 3: Transverse contraction condition.} 
    The transverse contraction condition can be stated as
    \begin{equation}
    \begin{split}
        \delta_{y}^\top \left ( \frac{\partial f_y}{\partial y}^\top M_y(y) + M_y(y) \, \frac{\partial f_y}{\partial y} + \dot{M}_y(y) + 2 \, \zeta_y \, M_y(y) \right ) \delta_{y} &\leq 0\\
        \frac{2 \, c^2}{R^2 \, (y_1^2 + y_2^2)} \left ( - 2 \alpha (y_1^2 + y_2^2) - \beta R^2 (y_1^2 + y_2^2) + \zeta_y R^2 (y_1^2 + y_2^2) + \zeta_y R^2 \right ) & \leq 0,\\
        -(2\alpha + \beta R^2) (y_1^2 + y_2^2) + \zeta_y R^2 (y_1^2 + y_2^2 + 1)& \leq 0.
    \end{split}
    \end{equation}
    Subsequently, we can isolate the contraction rate in the inequality condition
    \begin{equation}
        \zeta_y \leq \left ( \frac{2 \alpha}{R^2} + \beta \right ) \frac{y_1^2+y_2^2}{y_1^2 + y_2^2 + 1}.
    \end{equation}
    Now, if we define the contraction region as $\sqrt{y_1^2+y_2^2} \geq r_\epsilon > 0$, we can guarantee a transverse contraction rate $\zeta_y \geq \left (\frac{2\alpha}{R^2} + \beta \right ) \, \frac{r_\epsilon^2}{r_\epsilon^2 + 1} \: \forall y \in \mathcal{Y}$.
\end{proof}

\subsubsection*{Ablation Study Loss Functions}
In order to quantify the impact of each loss term on the overall performance of OSMP, we conduct an ablation study and report the statistics of the evaluation metrics across three random seeds in Tab.~\ref{tab:ablation_study_loss_functions} and qualitative results in Fig.~\ref{fig:ablation_study_loss_functions}.
It can be seen that the combination of velocity imitation loss $\mathcal{L}_\mathrm{vi}$, the encoder regularization loss $\mathcal{L}_\mathrm{er}$, the limit cycle matching loss $\mathcal{L}_\mathrm{lcm}$, and the time guidance loss $\mathcal{L}_\mathrm{tgd}$ works the best for very complex and curved oracles, such as the TUD-Flame. For ``\emph{easier}'' oracles, such as the Ellipse and the Star, the encoder regularization loss and the time guidance loss can slightly degrade performance and can be left out of training. For very basic oracles, such as the Ellipse, even the limit cycle matching loss is not necessary, and it is sufficient to rely on the velocity imitation loss.
Finally, while the Hausdorff~\cite{hausdorff1914grundzuge} loss is suitable for simple oracle shapes~\cite{zhi2024teaching}, the limit cycle matching loss proposed in this paper is better suited for complex oracle shapes.

\subsubsection*{Inference Time Benchmarking}
We benchmarked multiple inference modes and Jacobian‐estimation schemes for an OSMP trained with the CorneliaTurtleRobotJointSpace oracle; the outcomes are listed in Tab.~\ref{tab:inference_time_benchmarking}. Here, “inference time” is the wall-clock duration—measured on an Apple MacBook M4 Max CPU—for a single forward pass (batch size = 1), averaged over 1,000 runs.

Because numerical Jacobians reduce accuracy, we also computed the velocity RMSE with respect to the demonstrated velocities. In addition, starting from the oracle’s initial states, we determined the trajectory RMSE. Unlike Sec.~\ref{sub:osmp_benchmarking}, we did not enforce the demonstration’s fixed integration step; instead, each rollout used a step size equal to the model’s measured inference time. Consequently, faster models are integrated with finer temporal resolution, potentially achieving higher numerical accuracy.


The results reveal an interesting trade-off: models with an exact Jacobian run more slowly, whereas those using approximate Jacobians support higher control rates and can sometimes track the trajectory more precisely thanks to the finer integration grid. Switching from analytical to numerical Jacobians increases the velocity RMSE by 12.5\%, yet the lowest trajectory RMSE is obtained with an AOT-compiled model that employs numerical Jacobians — 36 \% lower than an eager-mode baseline with autograd Jacobians. This improvement stems from the roughly ninefold reduction in integration step size made possible by the compilation speed-up.

We note that this Trajectory RMSE tradeoff highly depends on the specific trained model, the oracle, and the available inference hardware. For example, the advantages of a high control rate are particularly pronounced where fast velocities are needed, such as in the case of a turtle swimming. Therefore, we recommend a separate analysis for each use case if losing accuracy via numerical Jacobians can be outweighed by a faster inference time.

\subsubsection*{Quantitative Evaluation of Real-World Experiments}
In addition to extended qualitative evaluation of the real-world experimental results in Fig.~\ref{fig:robot_embodiments_extended_results}, w also quantitatively evaluate the real-world experiments - specifically, the imitation metrics and the shape similarity metrics between the actual and desired motion for the experiments with the UR5, Helix soft robot and turtle robot.
For the shape similarity metrics, we adopt the same metrics as for the convergence analysis.

We perform this quantitative evaluation by first aligning the experimental data temporarily with the oracle data. For this purpose, we compute the error between the actual position sequence and the first position of the oracle. Subsequently, we identify the minimas of this alignment error and use this information to temporarily shift the experimental data such that it aligns with the starting position of the oracle.
Next, we resample the experimental data such that it aligns with the oracle steps. For the imitation metrics, we resample the experimental data such that each time step aligns with the oracle trajectory. Instead, for the shape similarity metrics, we resample the experimental data containing one trajectory period to contain the same number of time steps as the oracle without temporarily aligning each step.

The results shown in Tab.~\ref{tab:robot_embodiments_quantitative_evaluation} show the robots controlled by the OSMPs are able to quite accurately imitate the demonstration and also, that the actual trajectory shapes matches closely the desired oracle shape.
For the Helix soft robot, we additionally quantitatively compare the performance against a classical trajectory tracking (TT) controller.
The result show that the OSMPs are able to track the oracle shape much more accurately. However, the imitation performance can be sometimes worse than the trajectory tracking controller as the OSMPs being a subclass of DMPs do not contain an explicit time reference and, therefore, any low-level control delays, external disturbances, etc. lead to a time delay in the tracking of the oracle, which cannot be compensated as OSMPs do not contain any error-based feedback correction term that is operating on the trajectory/oracle phase.

\newpage




\begin{figure}[ht!]
    \centering
    \includegraphics[width=1.0\linewidth]{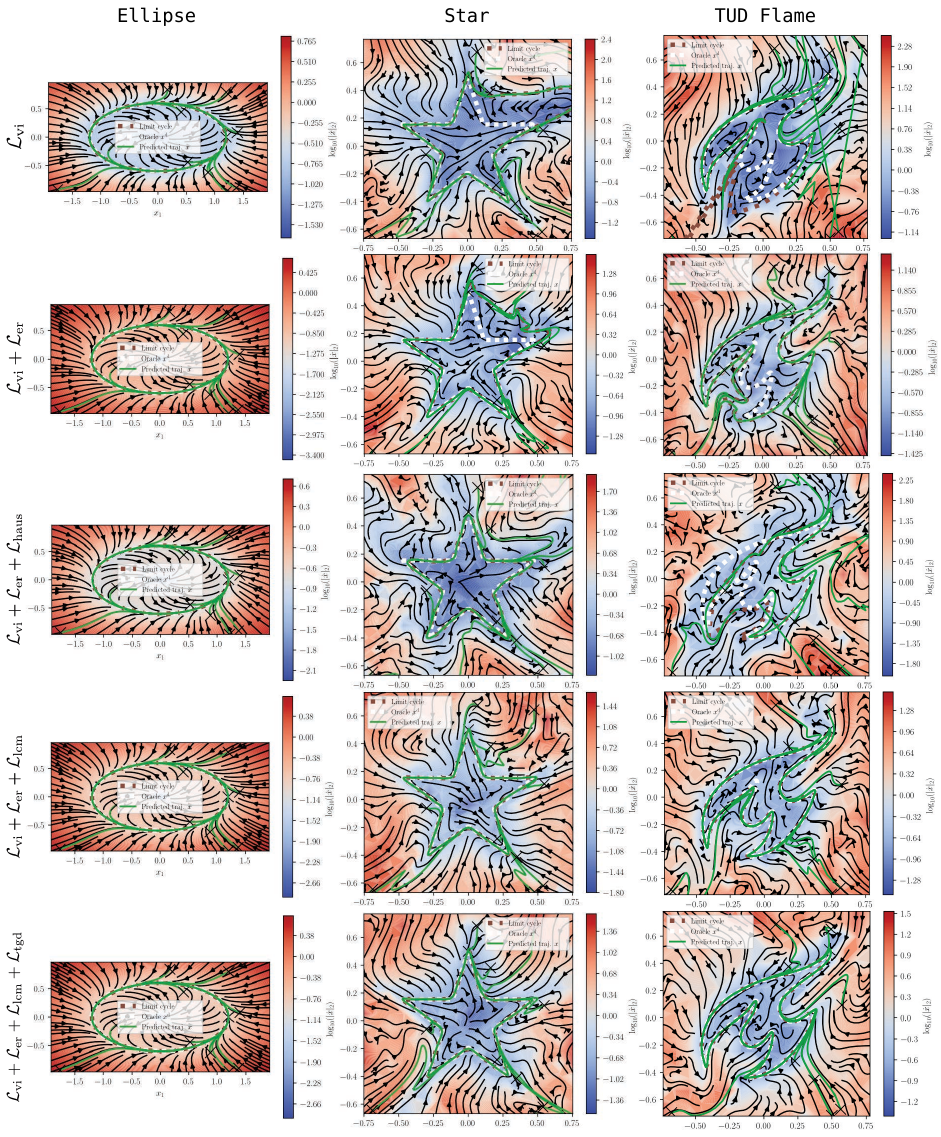}
    \caption{\textbf{Qualitative results for ablation study on loss functions.}
    This figure presents the qualitative results for the ablation study on the effect of the proposed loss functions on the imitation and convergence characteristics of the learned motion policy on a selection of datasets demonstrating that the full set of loss functions improves the quality of the learned velocity field on very complex and highly-curved oracles, such as the TUD-Flame logo dataset. 
    }
    \label{fig:ablation_study_loss_functions}
\end{figure}

\begin{figure}[ht!]
    \centering
    \includegraphics[width=0.9\linewidth]{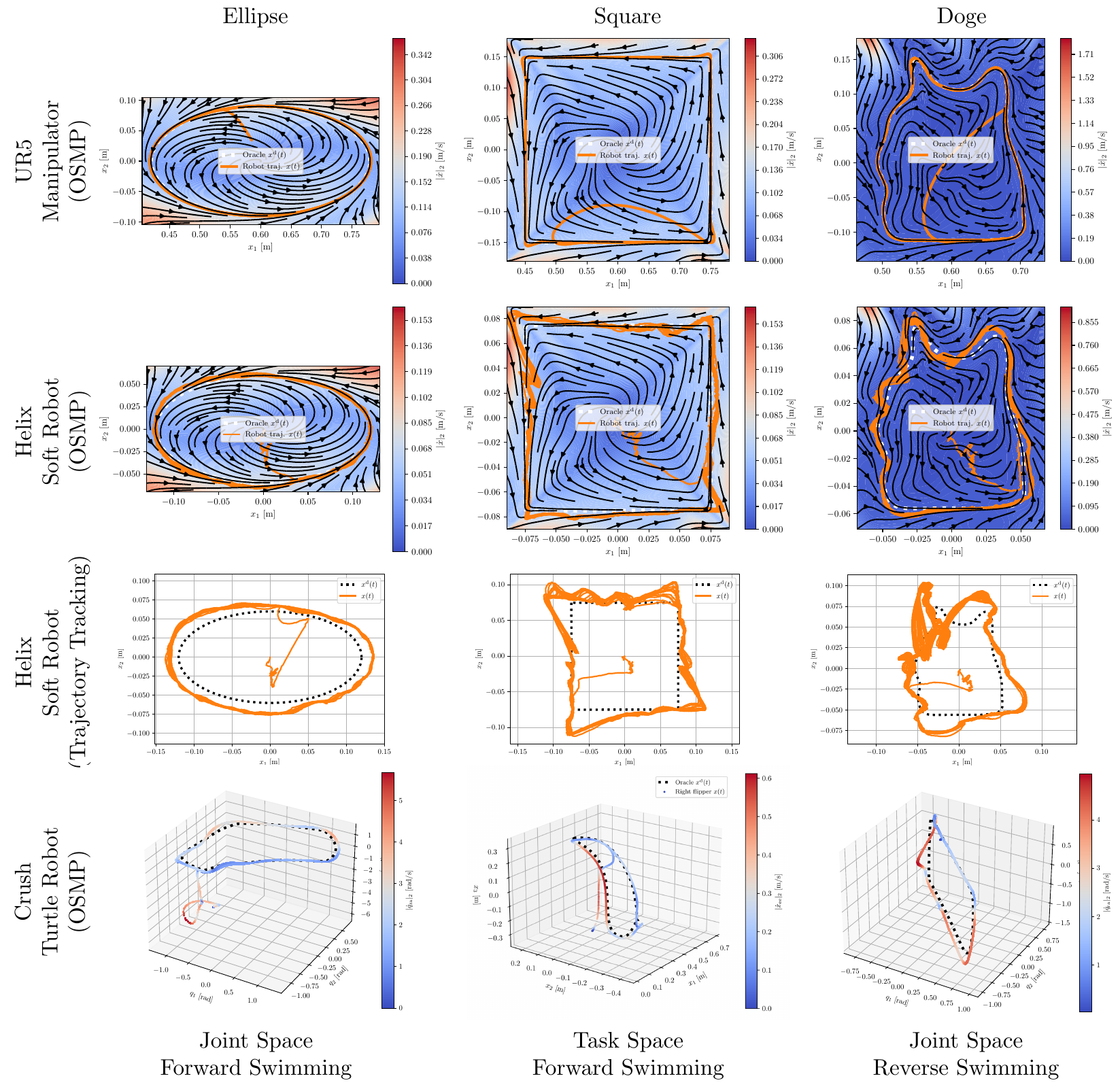}
    \caption{\textbf{Supplementary results for stable tracking of oracles across robot embodiments in the real world.}
    This figure shows the trajectories of various robot embodiments controlled with OSMPs based on data recorded during real-world experiments.
    The first row shows the behavior of a UR5 manipulator in task space where the OSMP is trained on various geometric shape oracles, including an ellipse, a square, and a doge. The black dotted lines denote the oracle, the orange line is the actual trajectory of the UR5's end-effector, and the velocity field is based on the learned OSMP.
    The second row considers the same trained OSMPs, but this time evaluates their behavior on a Helix soft robot.
    The third row presents the behavior of classical trajectory tracking controllers on the Helix soft robot for the same oracles.
    Finally, the fourth row contains measurements of the Crush turtle robot operated (in water) by OSMPs trained on biological oracles, where the black dotted line denotes the oracle and the colored line the actual behavior of the right flipper arm. 
    }
    \label{fig:robot_embodiments_extended_results}
\end{figure}




\begin{table}
    \centering
    \caption{\textbf{Comparison of characteristics of proposed methods against relevant baseline methods.} 
    In some cases, we denote with a \emph{x$^*$} that the feature could (probably) be developed for the respective method, but was not presented in the original paper.
    \emph{Note:} The CLF-CBF-NODE only guarantees convergence to a target trajectory as predicted by the learned NODE. However, the learned NODE is not guaranteed to exhibit a stable limit cycle behavior.
    }
    \label{tab:osmp_characteristics_vs_baselines} 
    \begin{scriptsize}
    \setlength\tabcolsep{2.5pt}
    \begin{tabular}{l ccccccc}\\
        \hline
        \textbf{Method} & \textbf{Model} & \textbf{Orbital Stability} & \textbf{Velocity} & \textbf{$N$-Policies} & \textbf{Task} & \textbf{Smooth Task}\\
        & \textbf{Expressiveness} & \textbf{Guarantees} & \textbf{Imitation} & \textbf{Synchronization} & \textbf{Conditioning} & \textbf{Interpolation}\\
        \hline
        MLP \& RNN & Moderate & x & Noisy & x & $\checkmark$ & x\\
        Neural ODE (NODE) & Moderate & x & $\checkmark$ & x & $\checkmark$ & x\\
        Diffusion Policy (DP)~\cite{chi2023diffusion} & High & x & Noisy & x & $\checkmark$ & x\\
        Classical Rythmic DMPs~\cite{ijspeert2002learning, wensing2017sparse, abu2024learning}  & Limited & $\checkmark$ & $\checkmark$ & $\checkmark$ & x & x\\
        HB-GMR~\cite{khadivar2021learning} & Limited & $\checkmark$ & x & x$^*$ & x & x\\
        Imitation Flow (iFlow)~\cite{urain2020imitationflow} & Moderate & $\checkmark$ &  $\checkmark$ & x$^*$ & x$^*$ & x\\
        CLF-CBF-NODE~\cite{nawaz2024learning} & Moderate & x & $\checkmark$ & x & x$^*$ & x\\
        SPDT~\cite{zhi2024teaching} & Moderate & $\checkmark$ & x & x$^*$ & x & x\\
        \textbf{OSMP (ours)} & Moderate & $\checkmark$ & $\checkmark$ & $\checkmark$ & $\checkmark$ & $\checkmark$\\
        \hline
    \end{tabular}
    \end{scriptsize}
\end{table}

\begin{table}[htbp]\centering
    \centering
    \caption{\textbf{Quantitative results for ablation study on loss functions.} We report the mean $\pm$ std across three random seeds for each dataset evaluation.
    As a particular point of emphasis, we compare the limit cycle matching loss $\mathcal{L}_\mathrm{lcm}$ proposed in this work with the Hausdorff distance loss as employed by Zhi~\textit{et al.}~\cite{zhi2024teaching} as they serve a similar purpose.
    }
    \label{tab:ablation_study_loss_functions} 
    \setlength{\tabcolsep}{3pt}   
    \renewcommand{\arraystretch}{1.2} 
    \begin{tiny}
    \begin{tabular}{|l|l|ccc|cc|cc|}   
    \hline
    & & \multicolumn{3}{c|}{\textbf{Imitation Metrics}} & \multicolumn{2}{c|}{\textbf{Local Convergence Metrics}} & \multicolumn{2}{c|}{\textbf{Global Convergence Metrics}}\\
    \cline{3-9}
    \textbf{Dataset} & \textbf{Loss Config.} & \textbf{Traj. RMSE~$\downarrow$} & \textbf{Norm. Traj. DTW~$\downarrow$} & \textbf{Vel. RMSE~$\downarrow$} & \textbf{Hausdorff Dist.~$\downarrow$} & \textbf{ICP MED~$\downarrow$} & \textbf{Hausdorff Dist.~$\downarrow$} & \textbf{ICP MED~$\downarrow$}\\
    \hline
    \multirow{5}{*}{Ellipse} & $\mathcal{L}_\mathrm{vi}$ & $0.001 \pm 0.000$ & $0.0010 \pm 0.0000$ & $0.004 \pm 0.000$ & $\mathbf{0.042 \pm 0.000}$ & $\mathbf{0.002 \pm 0.000}$ & $0.002 \pm 0.000$ & $0.001 \pm 0.000$\\
    & $\mathcal{L}_\mathrm{vi} + \mathcal{L}_\mathrm{er}$ & $0.001 \pm 0.000$ & $0.0007 \pm 0.0001$ & $0.004 \pm 0.000$ & $\mathbf{0.042 \pm 0.000}$ & $\mathbf{0.002 \pm 0.000}$ & $\mathbf{0.001 \pm 0.000}$ & $\mathbf{0.000 \pm 0.000}$\\
    & $\mathcal{L}_\mathrm{vi} + \mathcal{L}_\mathrm{er} + \mathcal{L}_\mathrm{haus}$ & $0.006 \pm 0.002$ & $0.0025 \pm 0.0006$ & $0.006 \pm 0.000$ & $\mathbf{0.042 \pm 0.000}$ & $0.003 \pm 0.000$ & $0.005 \pm 0.001$ & $0.002 \pm 0.001$\\
    & $\mathcal{L}_\mathrm{vi} + \mathcal{L}_\mathrm{er} + \mathcal{L}_\mathrm{lcm}$ & $\mathbf{0.000 \pm 0.000}$ & $\mathbf{0.0006 \pm 0.0001}$ & $\mathbf{0.004 \pm 0.000}$ & $\mathbf{0.042 \pm 0.000}$ & $\mathbf{0.002 \pm 0.000}$ & $\mathbf{0.001 \pm 0.000}$ & $\mathbf{0.000 \pm 0.000}$\\
    & $\mathcal{L}_\mathrm{vi} + \mathcal{L}_\mathrm{er} + \mathcal{L}_\mathrm{lcm} + \mathcal{L}_\mathrm{tgd}$ & $\mathbf{0.000 \pm 0.000}$ & $\mathbf{0.0006 \pm 0.0001}$ & $\mathbf{0.004 \pm 0.000}$ & $\mathbf{0.042 \pm 0.000}$ & $\mathbf{0.002 \pm 0.000}$ & $\mathbf{0.001 \pm 0.000}$ & $\mathbf{0.000 \pm 0.000}$\\
    \hline
    \multirow{5}{*}{Star} & $\mathcal{L}_\mathrm{vi}$ & $0.153 \pm 0.026$ & $0.0817 \pm 0.0193$ & $0.124 \pm 0.018$ & $0.368 \pm 0.065$ & $0.068 \pm 0.014$ & $0.350 \pm 0.055$ & $0.062 \pm 0.017$\\
    & $\mathcal{L}_\mathrm{vi} + \mathcal{L}_\mathrm{er}$ & $0.172 \pm 0.041$ & $0.1024 \pm 0.0246$ & $0.095 \pm 0.008$ & $0.393 \pm 0.118$ & $0.072 \pm 0.018$ & $0.340 \pm 0.081$ & $0.067 \pm 0.014$\\
    & $\mathcal{L}_\mathrm{vi} + \mathcal{L}_\mathrm{er} + \mathcal{L}_\mathrm{haus}$ & $0.077 \pm 0.020$ & $0.0144 \pm 0.0040$ & $0.089 \pm 0.018$ & $\mathbf{0.041 \pm 0.001}$ & $0.006 \pm 0.000$ & $0.015 \pm 0.001$ & $0.005 \pm 0.000$\\
    & $\mathcal{L}_\mathrm{vi} + \mathcal{L}_\mathrm{er} + \mathcal{L}_\mathrm{lcm}$ & $\mathbf{0.024 \pm 0.006}$ & $\mathbf{0.0031 \pm 0.0007}$ & $\mathbf{0.041 \pm 0.007}$ & $0.044 \pm 0.002$ & $\mathbf{0.003 \pm 0.000}$ & $\mathbf{0.013 \pm 0.002}$ & $\mathbf{0.002 \pm 0.000}$\\
    & $\mathcal{L}_\mathrm{vi} + \mathcal{L}_\mathrm{er} + \mathcal{L}_\mathrm{lcm} + \mathcal{L}_\mathrm{tgd}$ & $0.027 \pm 0.007$ & $0.0035 \pm 0.0007$ & $0.049 \pm 0.008$ & $0.049 \pm 0.005$ & $0.004 \pm 0.000$ & $0.017 \pm 0.003$ & $0.003 \pm 0.000$\\
    \hline
    \multirow{5}{*}{TUD-Flame} & $\mathcal{L}_\mathrm{vi}$ & $0.116 \pm 0.127$ & $0.0879 \pm 0.1198$ & $0.107 \pm 0.058$ & $0.069 \pm 0.042$ & $0.009 \pm 0.007$ & $\infty$ & $\infty$\\
    & $\mathcal{L}_\mathrm{vi} + \mathcal{L}_\mathrm{er}$ & $0.084 \pm 0.070$ & $0.0287 \pm 0.0352$ & $0.114 \pm 0.058$ & $0.082 \pm 0.057$ & $0.011 \pm 0.010$ & $0.066 \pm 0.072$ & $0.010 \pm 0.011$\\
     & $\mathcal{L}_\mathrm{vi} + \mathcal{L}_\mathrm{er} + \mathcal{L}_\mathrm{haus}$ & $0.375 \pm 0.048$ & $0.4067 \pm 0.0289$ & $0.180 \pm 0.006$ & $0.094 \pm 0.020$ & $0.016 \pm 0.004$ & $0.113 \pm 0.030$ & $0.010 \pm 0.005$\\
    & $\mathcal{L}_\mathrm{vi} + \mathcal{L}_\mathrm{er} + \mathcal{L}_\mathrm{lcm}$ & $0.113 \pm 0.084$ & $0.0249 \pm 0.0282$ & $0.114 \pm 0.049$ & $0.053 \pm 0.018$ & $0.006 \pm 0.004$ & $\infty$ & $\infty$\\
    & $\mathcal{L}_\mathrm{vi} + \mathcal{L}_\mathrm{er} + \mathcal{L}_\mathrm{lcm} + \mathcal{L}_\mathrm{tgd}$ & $\mathbf{0.033 \pm 0.008}$ & $\mathbf{0.0025 \pm 0.0002}$ & $\mathbf{0.073 \pm 0.010}$ & $\mathbf{0.039 \pm 0.003}$ & $\mathbf{0.003 \pm 0.000}$ & $\mathbf{0.009 \pm 0.001}$ & $\mathbf{0.002 \pm 0.000}$\\
    \hline
    \end{tabular}
    \end{tiny}
\end{table}

\begin{table}
    \centering
    \caption{\textbf{Numerical Jacobians allow for compilation of the motion policy and inference rates of up to 15,000~Hz on modern CPUs.}
    We evaluated the OSMP model trained on the \emph{CorneliaTurtleRobotJointSpace} dataset with $n=3$ and report the inference time, the Velocity RMSE, and the Trajectory RMSE for various inference modes and methods for computing the Jacobian of the encoder $J_\Psi$. 
    It is crucial to describe the evaluation procedure of the Trajectory RMSE. We begin by measuring the inference time for a single motion-policy step, from which we derive the control frequency at which the policy can run. We then execute the policy using a time step equal to this measured inference time, so a faster model, with its shorter step, can integrate the trajectory more finely and may, therefore, attain higher accuracy as long as it doesn't compromise the performance with a lowered velocity prediction accuracy.
    }
    \label{tab:inference_time_benchmarking} 

    \begin{small}
    \begin{tabular}{cc ccc} 
        \\
        \hline
        \textbf{Inference Mode} & \textbf{Jacobian Method} & \textbf{Inference Time [ms] $\downarrow$} & \textbf{Vel. RMSE $\downarrow$} & \textbf{Traj. RMSE $\downarrow$}\\
        \hline
        Eager & Autograd & $0.563$ & $\mathbf{0.00841}$ & $0.0315$\\
        Eager & VJP & $0.474$ & $\mathbf{0.00841}$ & $0.0326$\\
        Compiled & VJP & $0.116$ & $\mathbf{0.00841}$ & $0.0270$\\
        Eager & Numerical & $0.441$ & $0.00946$ & $0.0236$\\
        Compiled & Numerical & $0.117$ & $0.00946$ & $0.0212$\\
        Export & Numerical & $0.135$ & $0.00946$ & $0.0212$\\
        AOTInductor & Numerical & $\mathbf{0.064}$ & $0.00946$ & $\mathbf{0.0202}$\\
        \hline
    \end{tabular}
    \end{small}
\end{table}

\begin{table}
    \centering
    \caption{
    \textbf{Imitation and shape similarity evaluation metrics for real-world robot experiments.}
    We report the imitation metrics that evaluate the similarity between the actual robot behavior and the oracle/demonstration. Furthermore, we report shape similarity metrics between the actual trajectory shape and the desired oracle. In order to compute the metrics, we first identify the alignment between the first oracle time step and the actual trajectory by identifying the time shift that minimizes the error between the two. Subsequently, we evaluate the similarity between the oracle and the system trajectory that was resampled to the same timings for the imitation metrics and to the same number of steps for the shape similarity metrics, respectively.
    For the Helix soft robot experiments, we also compare the performance of OSMP against a classical trajectory tracking (TT) controller.
    }
    \label{tab:robot_embodiments_quantitative_evaluation} 
    \setlength{\tabcolsep}{2.5pt}   
    \renewcommand{\arraystretch}{1.2} 
    \begin{tiny}
    \begin{tabular}{|c | ccc | cccc | cc |}
        \hline
        \textbf{Robot} & \textbf{Task/Oracle} & \textbf{Speed-up} & \textbf{Motion Policy} & \multicolumn{4}{c|}{\textbf{Imitation Metrics}} & \multicolumn{2}{c|}{\textbf{Shape Similarity Metrics}}\\
        & & \textbf{Factor $s_\omega$} & & \textbf{Traj. RMSE $\downarrow$} & \textbf{Norm. Traj. DTW $\downarrow$} & \textbf{Traj. Frechet Dist. $\downarrow$} & \textbf{Vel. RMSE $\downarrow$} & \textbf{Hausdorff Dist.~$\downarrow$} & \textbf{ICP MED~$\downarrow$}\\
        \hline
        \multirow{6}{*}{UR5} & Square & $2.0$ & OSMP & $0.0094$~m & $0.0014$ & $0.016$~m & $0.0225$~m/s & $0.0077$~m & $0.0012$~m \\
        & Star & $2.0$ & OSMP & $0.0390$~m & $0.0083$ & $0.111$~m & $0.0802$~m/s & $0.0102$~m & $0.0019$~m \\
        & Doge & $2.0$ & OSMP & $0.0090$~m & $0.0018$ & $0.007$~m & $0.0230$~m/s & $0.0071$~m & $0.0016$~m \\
        & MIT-CSAIL & $2.0$ & OSMP & $0.0178$~m & $0.0027$ & $0.041$~m & $0.0426$~m/s & $0.0063$~m & $0.0015$~m \\
        & Dolphin & $0.5$ & OSMP & $0.0152$~m & $0.0016$ & $0.037$~m & $0.0166$~m/s & $0.0033$~m & $0.0007$~m \\
        & Whiteboard Cleaning & $0.5$ & OSMP & $0.0136$~m & $0.0024$ & $0.031$~m & $0.0240$~m/s & $0.0142$~m & $0.0016$~m \\
        \hline
        \multirow{4}{*}{Turtle Robot} & Joint Space Forw. Swim. & $0.5$ & OSMP & $0.2237$~rad & $0.0398$ & $0.536$~rad & $0.3673$~rad/s & $0.0854$~rad & $0.0215$~rad \\
        & Joint Space Forw. Swim. & $1.0$ & OSMP & $0.2544$~rad & $0.0660$ & $0.590$~rad & $0.8805$~rad/s & $0.1790$~rad & $0.0419$~rad \\
        & Joint Space Rev. Swim. & $1.5$ & OSMP & $0.6715$~rad & $0.3609$ & $1.835$~rad & $2.1995$~rad/s & $0.3815$~rad & $1.4600$~rad \\
        & Task Space Forw. Swim. & $0.5$ & OSMP & $0.0620$~m & $0.0244$ & $0.083$~m & $0.0609$~m/s & $0.0517$~m & $0.0125$~m \\
        \hline
        \multirow{6}{*}{Helix Soft Robot} & Ellipse & $2.0$ & TT & $0.0096$~m & $0.0126$ & $0.017$~m & $0.0119$~m/s & $0.0175$~m & $0.0122$~m \\
        & Ellipse & $2.0$ & OSMP & $0.0308$~m & $0.0066$ & $0.055$~m & $0.0195$~m/s & $0.0061$~m & $0.0020$~m \\
        & Square & $2.0$ & TT & $0.0166$~m & $0.0123$ & $0.039$~m & $0.0526$~m/s & $0.0388$~m & $0.0099$~m \\
        & Square & $2.0$ & OSMP & $0.0216$~m & $0.0073$ & $0.067$~m & $0.0302$~m/s & $0.0112$~m & $0.0026$~m \\
        & Doge & $2.0$ & TT & $0.0128$~m & $0.0141$ & $0.033$~m & $0.0514$~m/s & $0.0290$~m & $0.0101$~m \\
        & Doge & $2.0$ & OSMP & $0.0059$~m & $0.0039$ & $0.011$~m & $0.0182$~m/s & $0.0108$~m & $0.0033$~m \\
        \hline
    \end{tabular}
    \end{tiny}
\end{table}


\clearpage 


\subsubsection*{Captions for Movies}
\paragraph{\small Caption for Movie S1.}
\textbf{Deployment of OSMPs on diverse robot embodiments.}
This movie is an animated version of Fig.~\ref{fig:robot_embodiments_results}.
Specifically, we present real-world experimental results for deploying OSMPs achieving locomotion and rhythmic tool use on a diverse set of robot embodiments. This includes tracking challenging image contours with the UR5 manipulator and the Helix soft robot, and imitating biological turtle swimming oracles with the ``Crush'' turtle robot (joint-space forward and reverse swimming, task-space forward swimming).

\paragraph{\small Caption for Movie S2.}
\textbf{Kinesthetic teaching experiments.}
This movie animates the kinesthetic teaching experiments presented in Fig.~\ref{fig:robot_embodiments_results} and compares the demonstration motion obtained via kinesthetic teaching with the motion generated when executing the OSMP on the robot. We consider two types of demonstrations: whiteboard cleaning with the UR5 robot and brooming with the Kuka cobot.

\paragraph{\small Caption for Movie S3.}
\textbf{Verification of motion policy compliance.}
This movie is an animated version of Figure~\ref{fig:compliance_results} and evaluates the compliance of OSMPs against a classical time-parametrized trajectory tracking controller by analyzing the motion behavior upon interaction with the environment and perturbations with respect to the nominal oracle and the time reference.

\paragraph{\small Caption for Movie S4.}
\textbf{Phase synchronization of motion policies.}
This movie is an animated version of Figure~\ref{fig:phase_sync_results} and presents both simulation and experimental results for synchronizing multiple OSMPs in their phase. First, we demonstrate the scalability of the approach by synchronizing up to 6 OSMPs on the Dolphin contour. Subsequently, we showcase how the turtle robot's locomotion differs significantly with and without an activated phase synchronization.

\paragraph{\small Caption for Movie S5.}
\textbf{Task-conditioning of the motion policy and transfer to unseen tasks.}
This movie is an animated version of Figure~\ref{fig:conditioning_results} and exhibits how task conditioning of the encoder allows OSMPs to exhibit multiple, distinct motion behaviors. First, we show the results of an ablation study where we compare the behavior of models trained with and without the $\mathcal{L}_\mathrm{sci}$ loss when transferring to / prompted with tasks that were unseen during training. Finally, we include a video of the turtle robot swimming at various task conditionings, ranging from $z=-1$ (the reverse swimming joint space oracle), to $z=1$ (the forward swimming joint space oracle).





\end{document}